\documentclass{article}

\usepackage{microtype}
\usepackage{graphicx}
\usepackage{booktabs} %

\usepackage{amsfonts}       %
\usepackage{amsmath,amssymb}
\usepackage{nicefrac}       %
\usepackage{microtype}      %
\usepackage{xcolor}         %

\usepackage[colorlinks,citecolor=green!80!black]{hyperref}

\usepackage[accepted]{icml2023}

\usepackage[USenglish]{babel}
\usepackage[utf8]{inputenc} %
\usepackage[T1]{fontenc}    %
\usepackage{csquotes}
\usepackage{url}            %
\usepackage{booktabs}       %
\usepackage{amsfonts}       %
\usepackage{amsmath,amssymb}
\usepackage{nicefrac}       %
\usepackage{microtype}      %
\usepackage{xcolor}         %

\usepackage{amsmath}
\usepackage{amssymb}
\usepackage{amsthm}
\usepackage{mathtools, nccmath}
\usepackage{wrapfig}
\usepackage{comment}

\usepackage{caption}
\usepackage{graphicx}
\usepackage{subcaption}
\usepackage{placeins}

\usepackage[capitalize,noabbrev]{cleveref}

\theoremstyle{plain}
\newtheorem{theorem}{Theorem}[section]
\newtheorem{proposition}[theorem]{Proposition}
\newtheorem{Lemma}[theorem]{Lemma}
\crefname{Lemma}{Lemma}{Lemmas}
\newtheorem{corollary}[theorem]{Corollary}
\theoremstyle{definition}

\newtheorem{remark}[theorem]{Remark}
\newtheorem{setting}{Setting}

\usepackage[disable]{todonotes} %

\newcommand{\danica}[2][noinline]{\todo[color=violet!20,#1]{Danica: #2}}

\DeclareMathOperator{\bigO}{\mathcal{O}}
\def\polylog{\operatorname{polylog}}

\DeclareMathOperator*{\E}{\mathbb{E}}
\newcommand{\dset}{\mathcal{D}}

\DeclareMathOperator{\tr}{tr}
\newcommand{\tp}{^\mathsf{T}}

\newcommand{\cX}{\mathcal{X}}
\newcommand{\cY}{\mathcal{Y}}

\newcommand{\N}{\mathcal{N}}
\newcommand{\R}{\mathbb{R}}

\newcommand{\WF}{\mathrm{WF}}

\newcommand{\eNTK}{eNTK}
\newcommand{\meNTK}{\Theta_\theta}
\newcommand{\pNTK}{pNTK}
\newcommand{\mpNTK}{\hat{\Theta}_\theta}

\newcommand{\abs}[1]{\lvert #1 \rvert}
\newcommand{\Abs}[1]{\left\lvert #1 \right\rvert}
\newcommand{\norm}[1]{\lVert #1 \rVert}

\usepackage{xspace}
\makeatletter
\DeclareRobustCommand\onedot{\futurelet\@let@token\@onedot}
\def\@onedot{\ifx\@let@token.\else.\null\fi\xspace}
\makeatother
\def\eg{\emph{e.g}\onedot}

\icmltitlerunning{A Fast, Well-Founded Approximation to the Empirical Neural Tangent Kernel%
}

\begin{document}

\twocolumn[
\icmltitle{A Fast, Well-Founded Approximation to the Empirical Neural Tangent Kernel}

\icmlsetsymbol{equal}{*}

\begin{icmlauthorlist}
\icmlauthor{Mohamad Amin Mohamadi}{ubc}
\icmlauthor{Wonho Bae}{ubc}
\icmlauthor{Danica J.\ Sutherland}{ubc,amii}
\end{icmlauthorlist}

\icmlaffiliation{ubc}{Computer Science Department, University of British Columbia, Vancouver, Canada}
\icmlaffiliation{amii}{Alberta Machine Intelligence Institute, Edmonton, Canada}

\icmlcorrespondingauthor{Mohamad Amin Mohamadi}{lemohama@cs.ubc.ca}
\icmlcorrespondingauthor{Wonho Bae}{whbae@cs.ubc.ca}
\icmlcorrespondingauthor{Danica J.\ Sutherland}{dsuth@cs.ubc.ca}

\icmlkeywords{Machine Learning, ICML}

\vskip 0.3in
]

\printAffiliationsAndNotice{}  %

\begin{abstract}
Empirical neural tangent kernels (eNTKs) can provide a good understanding of a given network's representation: they are often far less expensive to compute and applicable more broadly than infinite-width NTKs. For networks with $O$ output units (\eg\ an $O$-class classifier), however, the eNTK on $N$ inputs is of size $NO \times NO$, taking $\bigO\big( (N O)^2\big)$ memory and up to $\bigO\big( (N O)^3 \big)$ computation to use. Most existing applications have therefore used one of a handful of approximations yielding $N \times N$ kernel matrices, saving orders of magnitude of computation, but with limited to no justification. We prove that one such approximation, which we call ``sum of logits,'' converges to the true eNTK at initialization. Our experiments demonstrate the quality of this approximation for various uses across a range of settings.
\end{abstract}

\section{Introduction}

The pursuit of a theoretical foundation for deep learning has lead researches to uncover interesting connections between neural networks (NNs) and kernel methods.
It has long been known that randomly initialized NNs in the infinite width limit are Gaussian processes with the Neural Network Gaussian Process (NNGP) kernel,
and training the last layer with gradient flow under squared loss corresponds to the posterior mean
\citep{neal1996priors,williams1996computing,hazan2015steps,lee2017deep,matthews2018gaussian,novak2018bayesian,yang2019wide}.
More recently, \citet{ntk2018jacot} built off a line of closely related prior work to show that the same is true with a different kernel, the Neural Tangent Kernel (NTK), if we train all the parameters of the network.
\citet{yang2020tensor,yang2021tensor} showed this holds not just for fully-connected NNs but universally across architectures, including ResNets and Transformers.
\citet{linntk2019lee} also showed that the dynamics of training wide but finite-width NNs with gradient descent can be approximated by a linear model obtained from the first-order Taylor expansion of that network around its initialization.
Furthermore, they experimentally showed that this approximation approximation excellently holds even for networks that are not so wide.

In addition to theoretical insights,
NTKs have had significant impact in diverse practical settings.
\citet{arora2019harnessing} show very strong performance of NTK-based models on a variety of low-data classification and regression tasks.
The condition number of an NN's NTK has been shown correlation directly with the trainability and generalization capabilities of the NN \citep{xiao2018dynamical, xiao2020disentangling};
thus, \citet{park2020towards,chen2021vision} have used this to develop practical algorithms for neural architecture search.
\citet{wei2022more,bachmann2022generalization} estimate the generalization ability of a specific network, randomly initialized or pre-trained on a different dataset, with efficient cross-validation.
\citet{zhou2021meta} use NTK regression for efficient meta-learning, and \citet{wang2021deep,holzmuller2022framework,mohamadi:active-ntk} use NTKs for active learning.

\begin{figure*}
    \centering
        \includegraphics[width=0.32\textwidth]{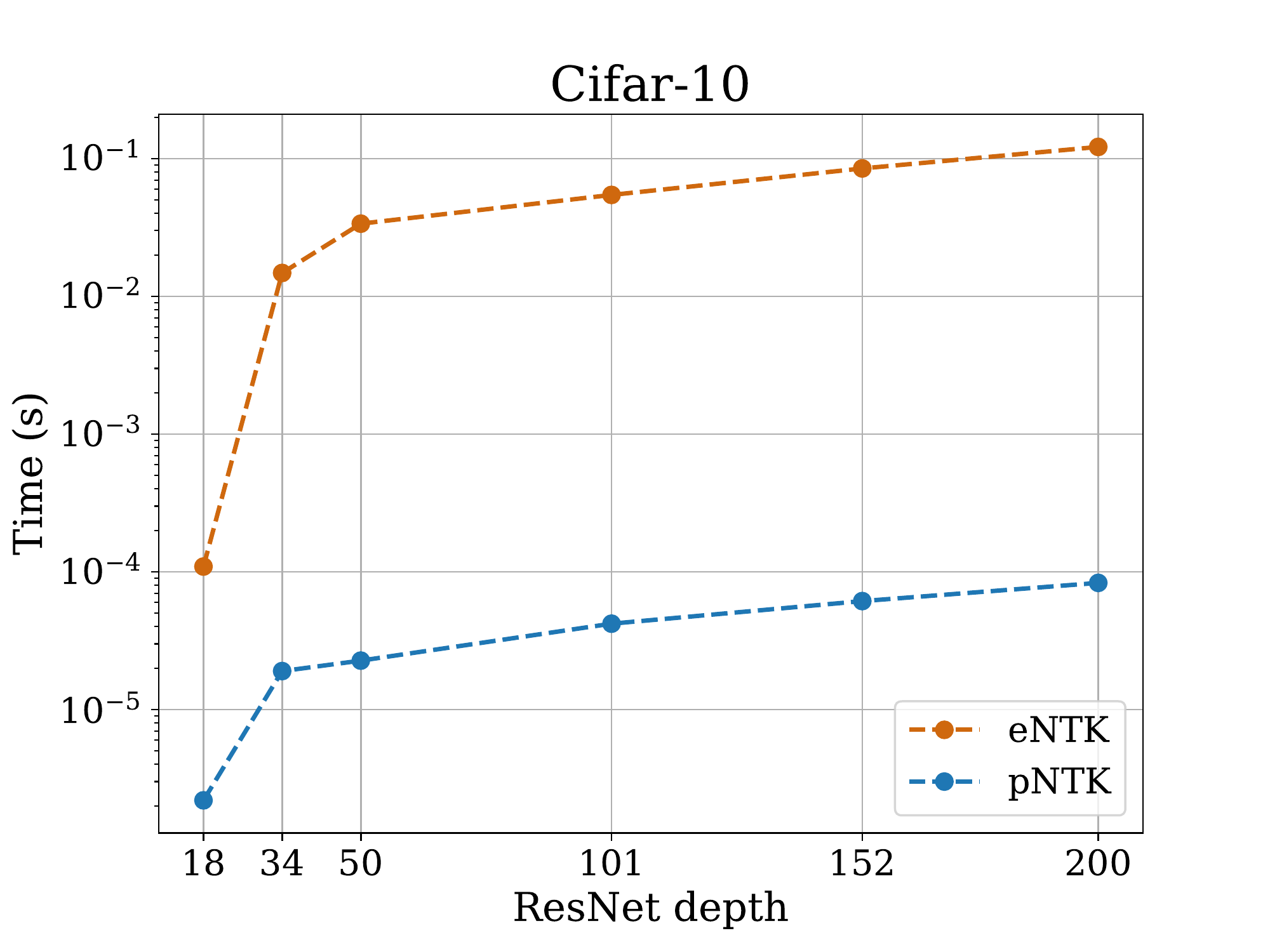}
    \hspace*{-1em}
        \includegraphics[width=0.32\textwidth]{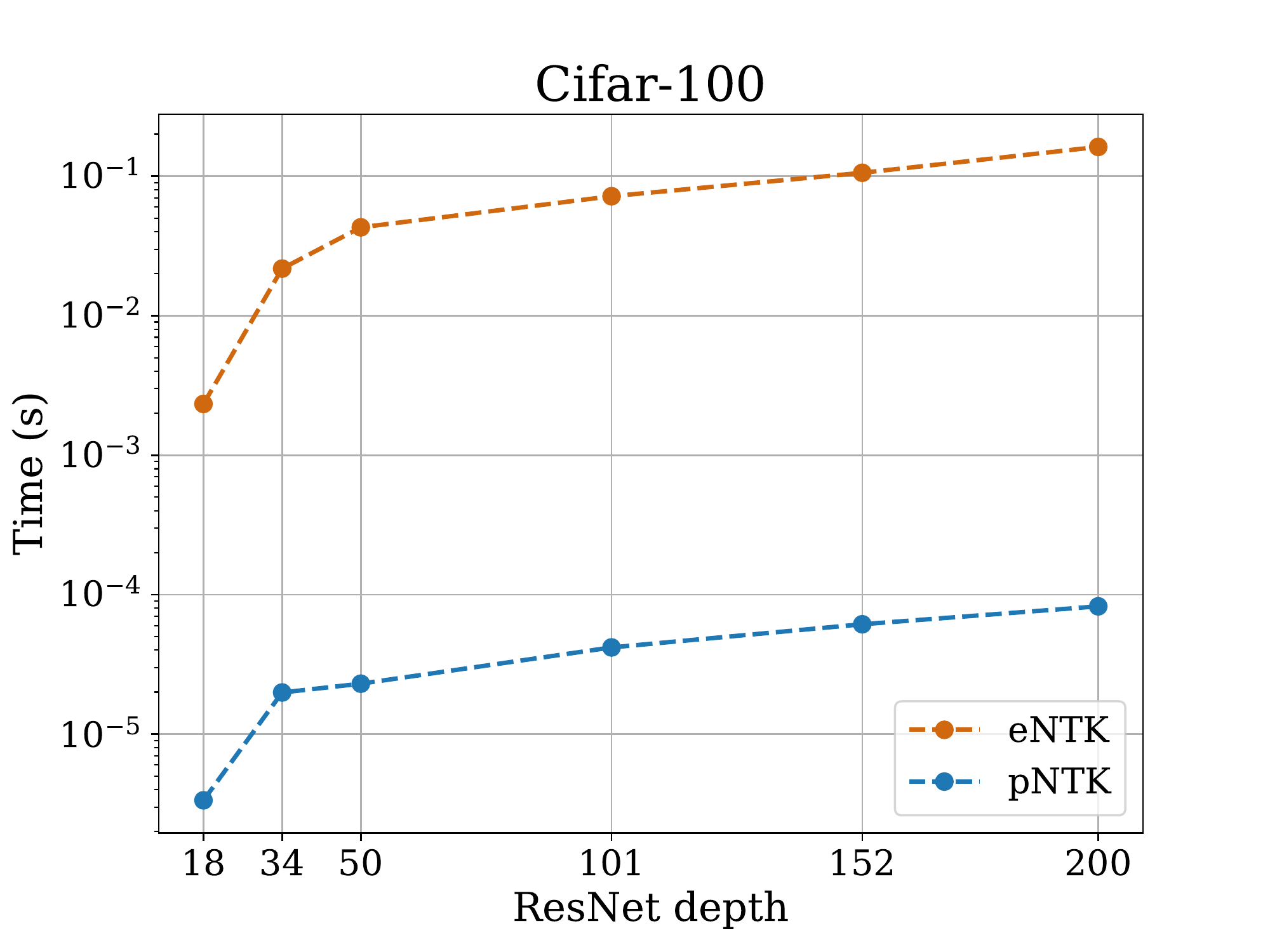}
    \hspace*{-1em}
        \includegraphics[width=0.32\textwidth]{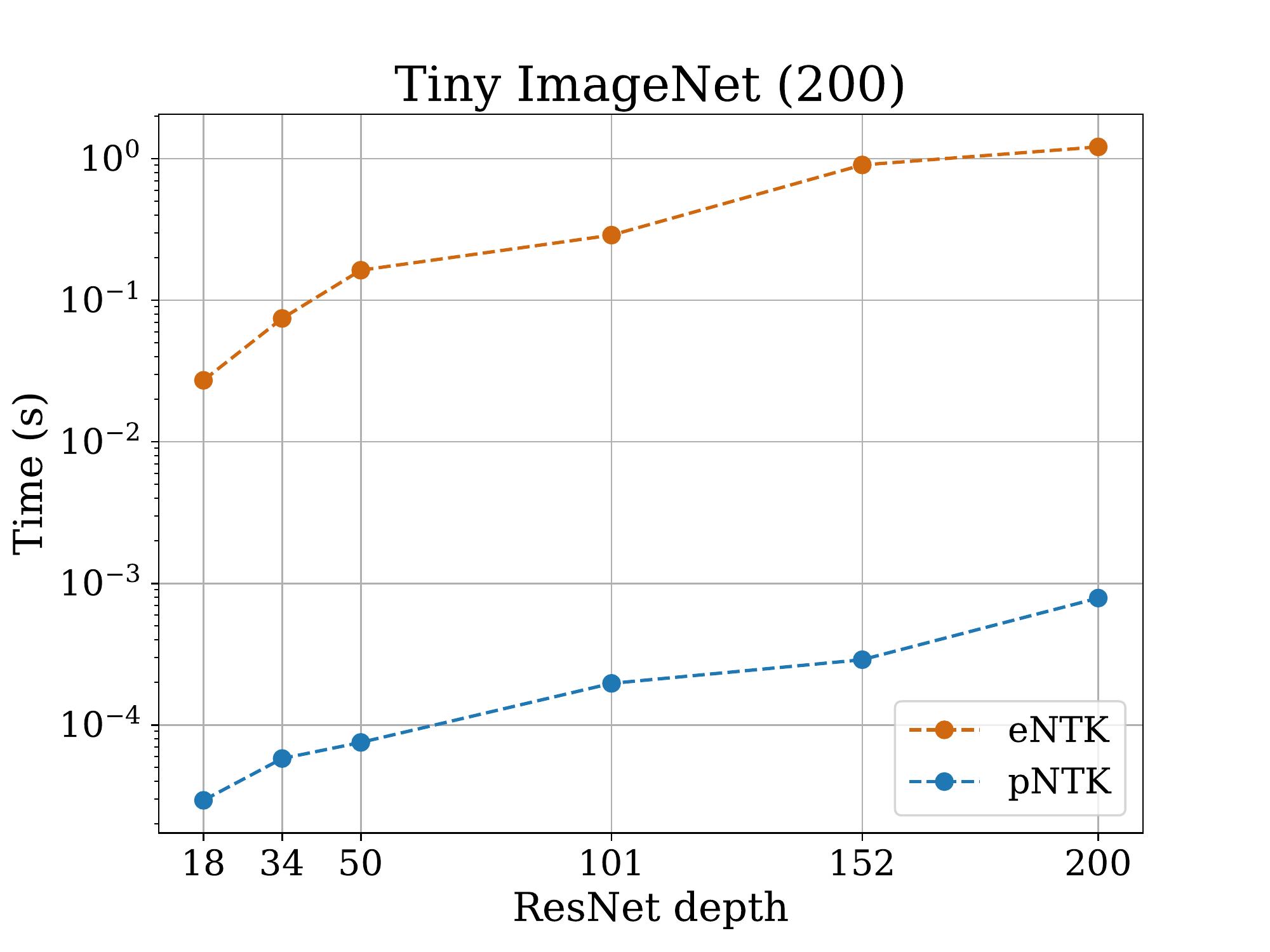}
    \caption{\textbf{Wall-clock time to evaluate the \eNTK~and \pNTK}~for one pair of inputs, across datasets and ResNet depths.}
    \label{fig:timing_experiments}
\end{figure*}

There has also been significant theoretical insight gained from empirical studies of networks' NTKs.
For instance,
\citet{fort2020deep} use NTKs to study how the loss geometry the NN evolves under gradient descent.
\citet{franceschi2021neural} employ NTKs to analyze the behaviour of Generative Adverserial Networks (GANs).
\citet{nguyen2020dataset,nguyen2021dataset} use NTKs for dataset distillation.
\citet{he2020bayesian,adlam2020exploring} use NTKs to predict and analyze the uncertainty of a NN's predictions.
\citet{tancik2020fourier} use NTKs to analyze the behaviour of MLPs in learning high frequency functions, leading to new insights into our understanding of neural radiance fields.
We thus believe NTKs will continue to be used in both theoretical and empirical deep learning.

Unfortunately, however, computing the NTK for practical networks is extremely challenging, and usually not computationally feasible.
The ``empirical'' NTK (eNTK; we discuss the difference from what others term ``the NTK'' shortly) is
\begin{equation} \label{eq:jac_ntk}
    \Theta_\theta(x_1, x_2) = \left[J_\theta \left(f_\theta(x_1) \right) \right] \left [ J_\theta \left (f_\theta(x_2) \right) \right ]^\top
,\end{equation}
where $J_\theta \left(f_\theta(x) \right)$ denotes the Jacobian of the function $f$ at a point $x$ with respect to the flattened vector of all its parameters, $\theta \in \R^P$.
If $D$ is the input dimension of $f$ and $O$ the number of outputs,
we have $J_\theta \left(f_\theta(x) \right) \in \R^{O \times P}$ and $\Theta_\theta(x_1, x_2) \in \R^{O \times O}$.
Thus, computing the eNTK between $N_1$ and $N_2$ data points
yields $N_1 N_2$ matrices, each of shape $O \times O$;
we usually arrange this as an $N_1 O \times N_2 O$ matrix.

When computing an eNTK on tasks involving large datasets and with multiple output neurons, \eg\ in a classification model with $O$ classes,
the eNTK quickly becomes impractical regardless of how fast each entry is computed due to its $NO \times NO$ size.
The full eNTK of a classification model even on the relatively small CIFAR-10 dataset~\citep{cifar102009krizhevsky}, stored in double precision, takes over 1.8 terabytes in memory. %
For practical usage, we need something better.

This work presents a simple trick for a strong approximation of the eNTK that removes the $O^2$ from the size of the kernel matrix,
resulting in a factor of $O^2$ improvement in the memory and up to $O^3$ in computation.
Since for typical classification datasets $O$ is at least $10$ (\eg\ CIFAR-10) and potentially $1{,}000$ or more \citep[\eg\ ImageNet,][]{deng2009imagenet},
this provides multiple orders of magnitude savings over the full \eNTK{}.
We prove this approximation converges to the original \eNTK{} at a rate of $\bigO(n^{-1/2})$ for a standard-initialization NN of depth $L$ and width $n$ in each layer,
and the predictions of kernel regression with the approximate kernel do the same. 
We also conduct diverse experimental investigations to support our theoretical results, across a range of architectures and settings.
We hope this approximation further enables researches to employ NTKs towards theoretical and empirical advances in wide networks.

\textbf{Infinite NTKs.}
In the infinite-width limit of appropriately initialized NNs, $\meNTK$ converges almost surely at initialization to a particular kernel, and remains constant through training.
Algorithms are available to compute this expectation exactly,
but they tend to be substantially more expensive than computing \eqref{eq:jac_ntk} directly for all but extremely wide networks.
The convergence to this infinite-width regime is slow in practice,
and moreover it eliminates some of the interest of the framework:
neural architecture search, predicting generalization of a pre-trained representation, and meta-learning are all considerably less interesting when we only consider infinite-width networks that do essentially no feature learning.
Thus we focus here only on the ``empirical'' eNTK as in \eqref{eq:jac_ntk}.

\section{Related Work}
Among the numerous recent works that have used eNTKs either to gain insights about various phenomenons in deep learning or to propose new algorithms, not many have publicized the computational costs and implementation details of computing eNTKs. Nevertheless, all are in agreement about the expense of such computations \citep{park2020towards,holzmuller2022framework,fort2020deep}.

Several recent works have,
mostly ``quietly,''
employed various techniques to avoid dealing with the full eNTK matrix; however,
to the best of our knowledge, none provide any rigorous justifications.
\citet[Section 2.3]{wei2022more} point out that if the final layer of a NN is randomly initialized, the \textit{expected} \eNTK~can be written as $K_0 \otimes I_O$ for some kernel $K_0$, where $I_O$ is the $O \times O$ identity matrix and $\otimes$ is the Kronecker product. Thus, they use the approximation in which they only compute the \eNTK~with respect to one of the logits of the NN. Although their approach to approximating the \eNTK~is similar to ours, they don't provide any rigorous bounds or empirical study of how closely the actual \eNTK~is approximated by its expectation in this regard.
\citet{wang2021deep} employs the same ``single-logit'' strategy, though they only mention the infinite-width limit as a motivation supporting their trick.
Despite these claims, we will see in our experiments that the \eNTK~is generally \emph{not} diagonal.
We will, however, prove upper bounds on distance of our approximation to the \eNTK,
and provide experimental support that this approximation captures the behaviour of the \eNTK~even when the NN's weights are not at initialization. 
\citet{park2020towards} and \citet{chen2021vision} also seem to use a form of ``single-logit'' approximation to \eNTK, without explicitly mentioning it. \citet{linntk2019lee}, by contrast, do use the full \eNTK, and hence never compute the kernel on more than 256 datapoints.

\citet{novak2021fast} recently performed an in-depth analysis of computational and memory complexity required for computing the \eNTK, and proposed two new approaches (depending on the NN architecture) to reduce the time complexity, but not the memory burden, of computing the \eNTK~over explicitly implementing \eqref{eq:jac_ntk}.
Our approaches are complementary;
in fact, we use their ``structured derivatives'' method to help compute our approximation.

\section{Pseudo-NTK} \label{sec:def-pseudo}
We define the pseudo-NTK (\pNTK), which we denote as $\mpNTK(x_1, x_2)$, as
\begin{equation}\label{eq:pntk_def}
    \underbrace{\left[\nabla_\theta\;
    \frac{1}{\sqrt{O}}\sum_{i=1}^O f^{(i)}_\theta(x_1)
    \right]}_{1 \times P}
    \underbrace{\left[ \nabla_\theta\;
    \frac{1}{\sqrt{O}}\sum_{i=1}^O f^{(i)}_\theta(x_2)
    \right]^\top}_{P \times 1}
,
\end{equation}
where $f^{(i)}_\theta(x)$ denotes the $i$-th output of $f_\theta$ on the input $x$. While the \eNTK~is a matrix-valued kernel for each pair of inputs, the \pNTK~ is a traditional scalar-valued kernel.

Some recent work \citep{cntk2019arora,yang2020tensor,wei2022more,wang2021deep} has pointed out that in the infinite width limit $\lim_{n \to \infty} \Theta(x_1, x_2)$, the NTK becomes a constant-diagonal matrix, where the class-class component becomes identity.
Thus, one can avoid computing the off-diagonal entries of the infinite-width NTK of each pair through using $\meNTK(x_1, x_2) = \mpNTK(x_1, x_2) \otimes I_O$, giving a drastic $\bigO(O^2)$ time and memory complexity decrease.

Practitioners have accordingly used the same approach in computing the \eNTK~of a finite width network, but with little to no further justification.
We see in our experiments that for finite width networks, the NTK is \textbf{not} diagonal.
In fact, we show that for most practical networks, it is very far from being diagonal,
casting doubts on the validity of arguments justifying the approximation with asymptotic diagonality.
We justify this category of approximation with theoretical bounds on
the difference of the true NTK from the approximation \eqref{eq:pntk_def}, which we also call ``sum of logits.''

Before our formal results and experimental evaluation,
we give some intuition.
First, suppose $f_\theta^{(i)}(x) = \phi(x) \cdot v_i$,
so that $v_i {{}\in \R^{n_{L-1}}}$ is the $i$th row of a linear read-out layer;
then
$\frac{1}{\sqrt{O}} \sum_{i=1}^O f_\theta^{(i)}(x) = \phi(x) \cdot \left[ \frac{1}{\sqrt O} \sum_{i=1}^O v_i \right]$.
If the $v_i \sim \N(0, \sigma^2 I_{n_{L-1}})$ are independent,
$\frac{1}{\sqrt O} \sum_{i=1}^O v_i$ %
has the same normal distribution as, say, $v_1$.
Thus, at initialization, our sum of logits approximation agrees in distribution with the first-logit approximation.
Our proof uses the sum-of-logits form, though,
and we believe it may be more sensible for
networks that are not at random initialization.

Calling this vector (whether the first logit or sum of logits) $v$,
we can think of \eqref{eq:pntk_def} as the NTK of a model with a single scalar output as a function of $\phi$, whose last layer has weights $v$.
When we linearize a network with that kernel for an $O$-class classification problem,
getting the formula \eqref{eq:pseudo_lin} discussed in \cref{sec:kernel-regression},
we end up effectively using a one-vs-rest classifier scheme.
Thus, we can think of the pseudo-NTK as approximating the process of training $O$ one-vs-rest classifiers,
rather than a single $O$-way classifier.

\section{Approximation Quality of Pseudo-NTK}

We will now study various aspects of the approximation of \eqref{eq:pntk_def} to \eqref{eq:jac_ntk},
both in theory and empirically.
Our experiments compare different characteristics of \pNTK{} and \eNTK{}, both at initialization and throughout training.
We evaluate four widely-used architectures:
FCN \citep[a fully-connected network of depth 3, as in][]{linntk2019lee, lee2020finite},
ConvNet \citep[a fully-convolutional network of depth 8, as in][]{cntk2019arora, arora2019harnessing, lee2020finite}, ResNet18 \citep{he2016deep},
and WideResNet-16-$k$ \citep{wide_resnet2016zagoruyko}.
We evaluate each architecture at different widths, as mentioned in the plot legends: we show exact widths for FCN, while for others we show a widening factor.
For consistency with most other recent papers studying NTKs and properties of NNs in general, we focus on data from CIFAR-10 \citep{cifar102009krizhevsky}. Each experiment is repeated using three seeds\danica{Still true?}; means and corresponding error bars are also shown, except when they interfered with clear interpretation of the plots. 
All models are trained for 200 epochs,
using stochastic gradient descent (SGD),
on 32GB NVIDIA V100 GPUs. %
More details on models and optimization are provided in \cref{app:expt-details}. The measured statistic for each experiment are reported after 0, 50, 100, 150, and 200 epochs.

\textbf{A Note On Parameterization} In order to be maximally applicable to practical implementations, both our experiments and our theoretical results are based on standard parameterization (``fan-in'' variance). Although most related work uses the so-called NTK parameterization (``fan-out'' variance), this is rarely used in practice while training NNs, mostly due to the poor generalization results achieved in comparison to training with standard parameterization \citep[Section I]{park19paramterization}. We encountered similarly poor behaviour when training NNs with NTK parameterization, but note that our theorems could also be adapted to the fan-out case.

\begin{figure*}[!htb]
    \centering
    \begin{subfigure}[b]{0.24\textwidth}
        \includegraphics[width=\textwidth]{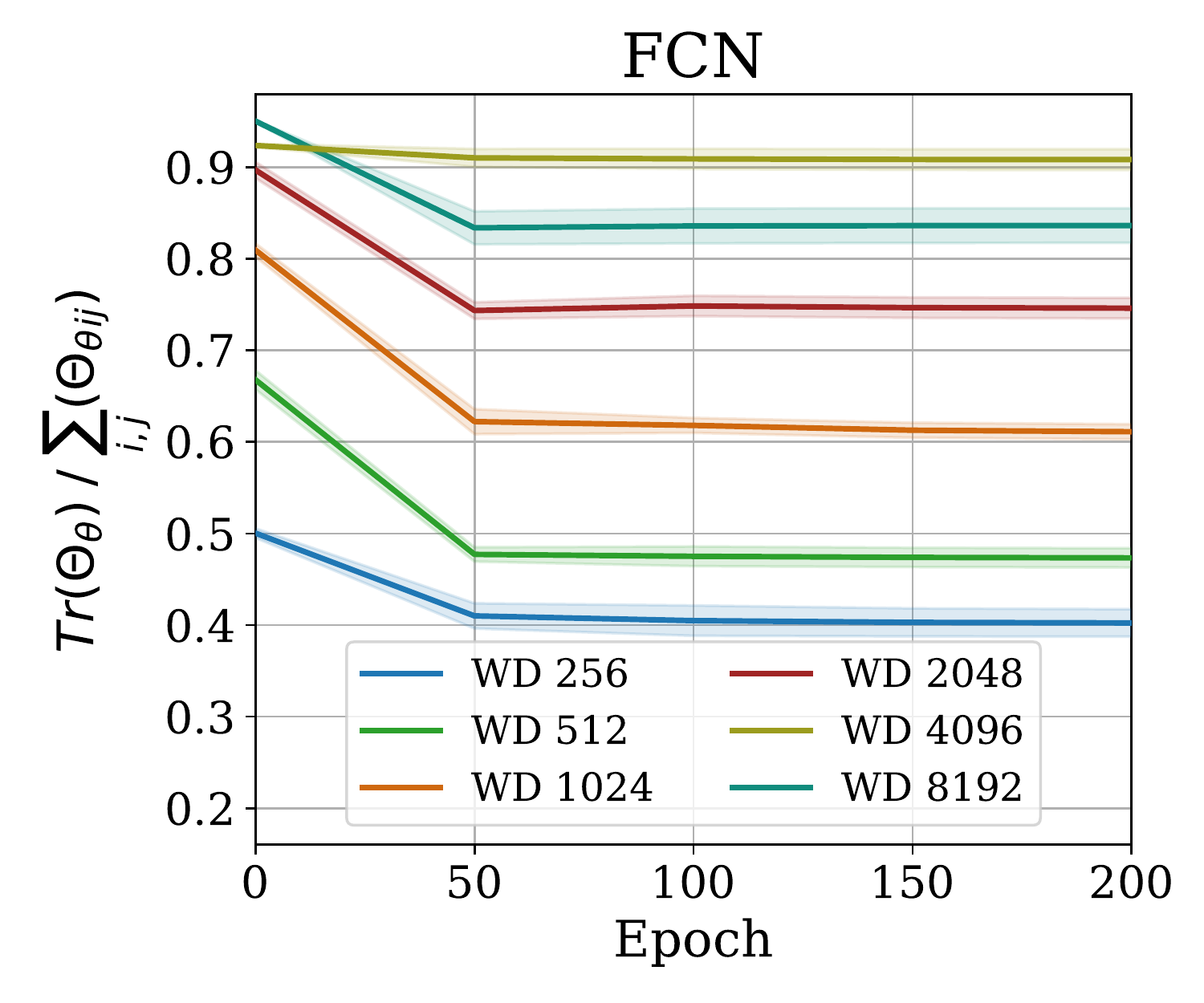}
    \end{subfigure}
    \hfill
    \begin{subfigure}[b]{0.24\textwidth}
        \includegraphics[width=\textwidth]{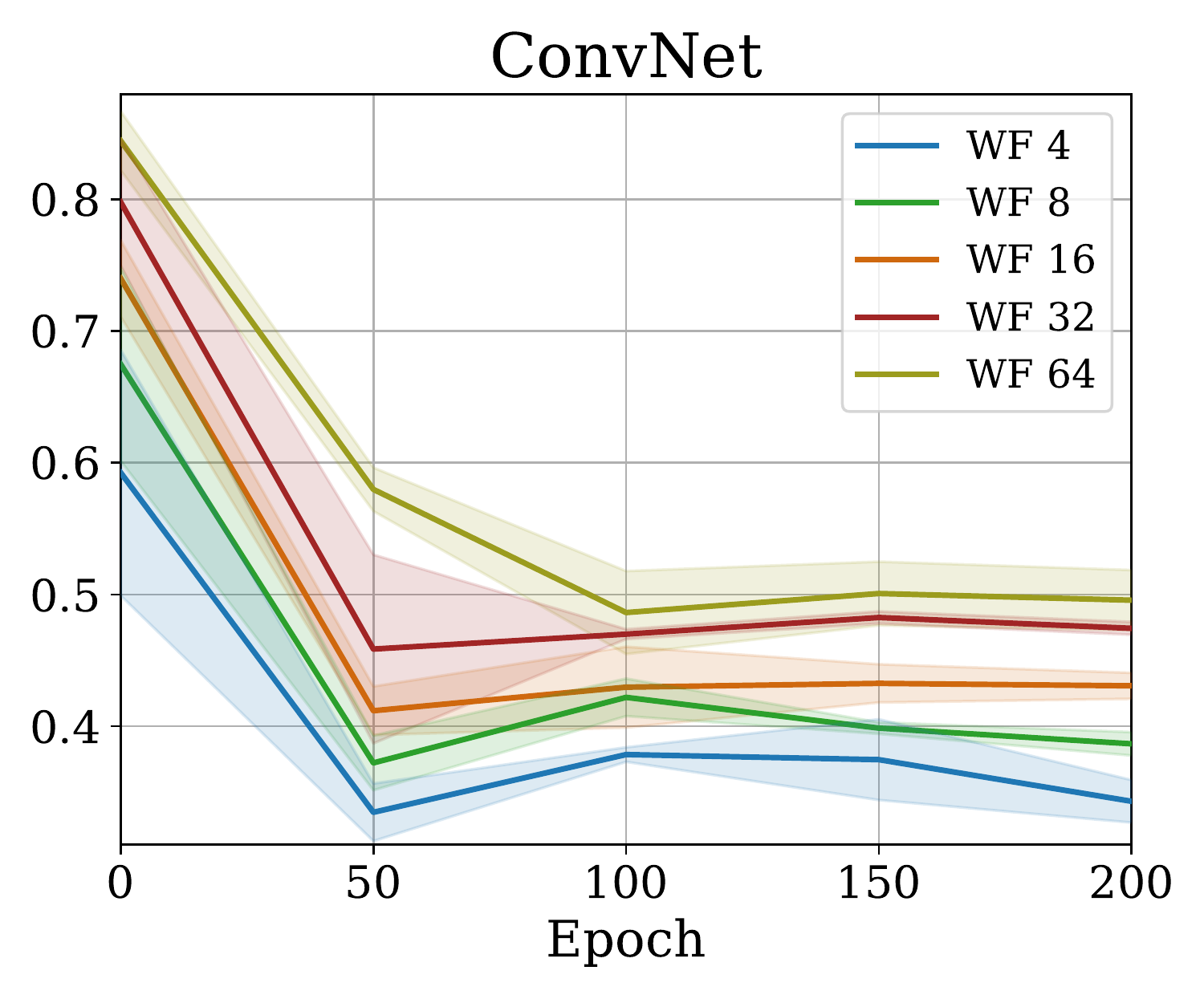}
    \end{subfigure}
    \hfill
    \begin{subfigure}[b]{0.24\textwidth}
        \includegraphics[width=\textwidth]{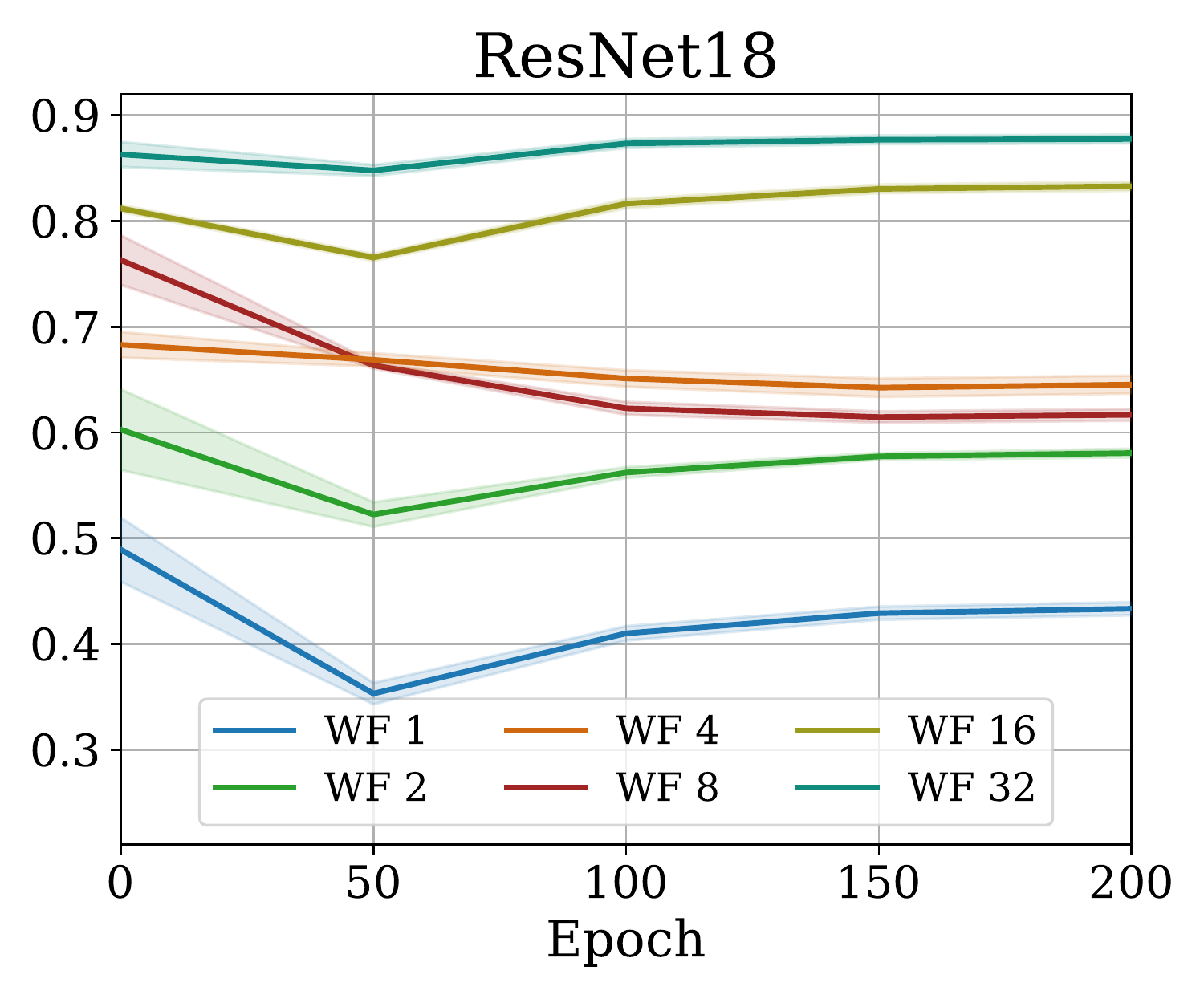}
    \end{subfigure}
    \hfill
    \begin{subfigure}[b]{0.24\textwidth}
        \includegraphics[width=\textwidth]{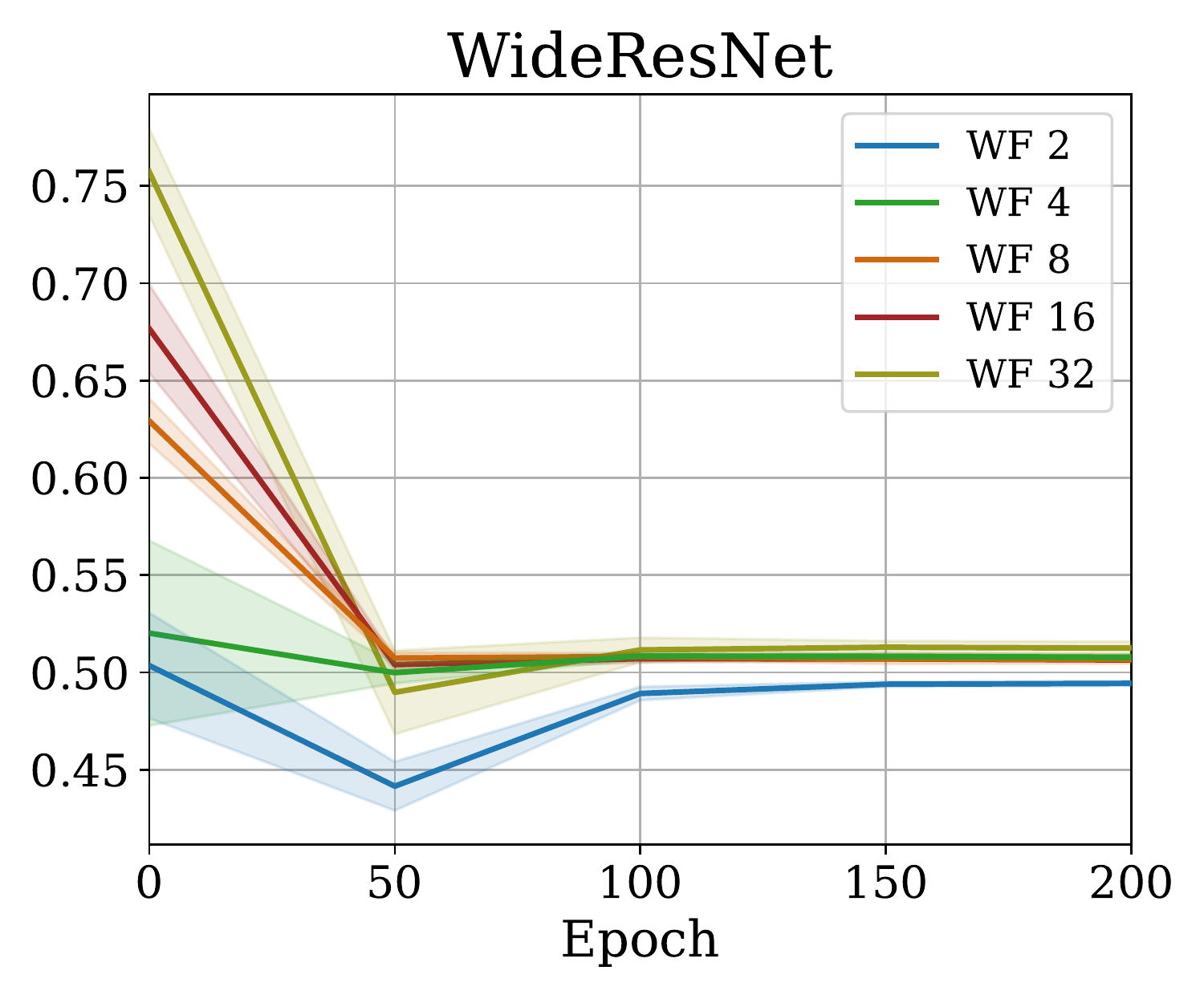}
    \end{subfigure}
    \vspace{-1mm}
    \caption{%
        \textcolor{black}{Comparing the \textbf{magnitude of sum of on-diagonal and off-diagonal elements of $\meNTK(\dset, \dset)$}} at initialization and throughout training,
        based on $\dset$ being $1000$ random points from CIFAR-10.
        The reported numbers are the average of $1000 \times 1000$ matrices, each having a shape of $10 \times 10$.
        The same subset has then been used to train the NN using SGD.
        As the NN's width grows, the \eNTK~converges to being diagonal at initialization among all different architectures.
    }
    \label{fig:diagonality}
    \vspace{-2mm}
\end{figure*}

\begin{figure*}[!htb]
    \centering
    \begin{subfigure}[b]{0.24\textwidth}
        \includegraphics[width=\textwidth]{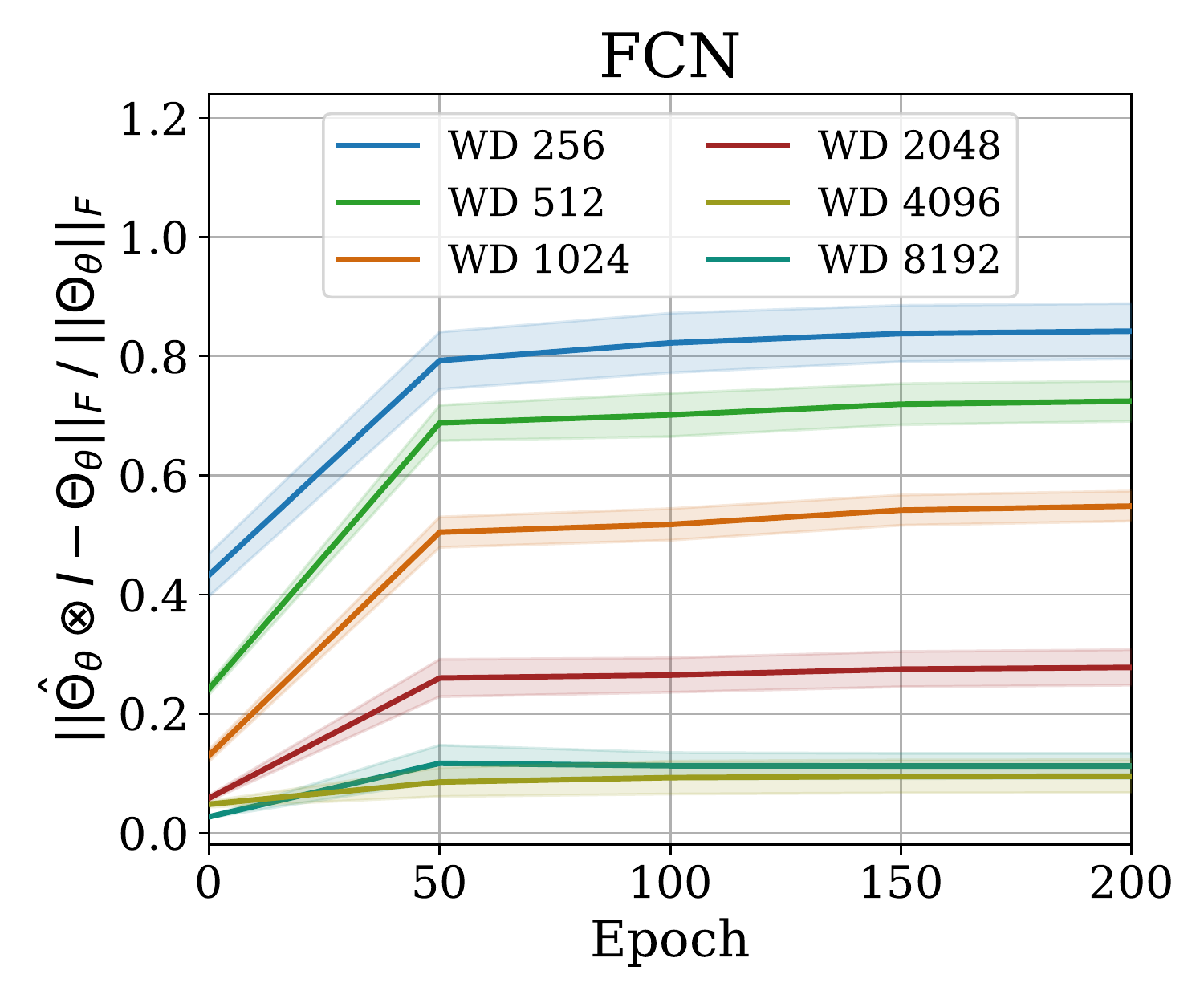}
    \end{subfigure}
    \hfill
    \begin{subfigure}[b]{0.24\textwidth}
        \includegraphics[width=\textwidth]{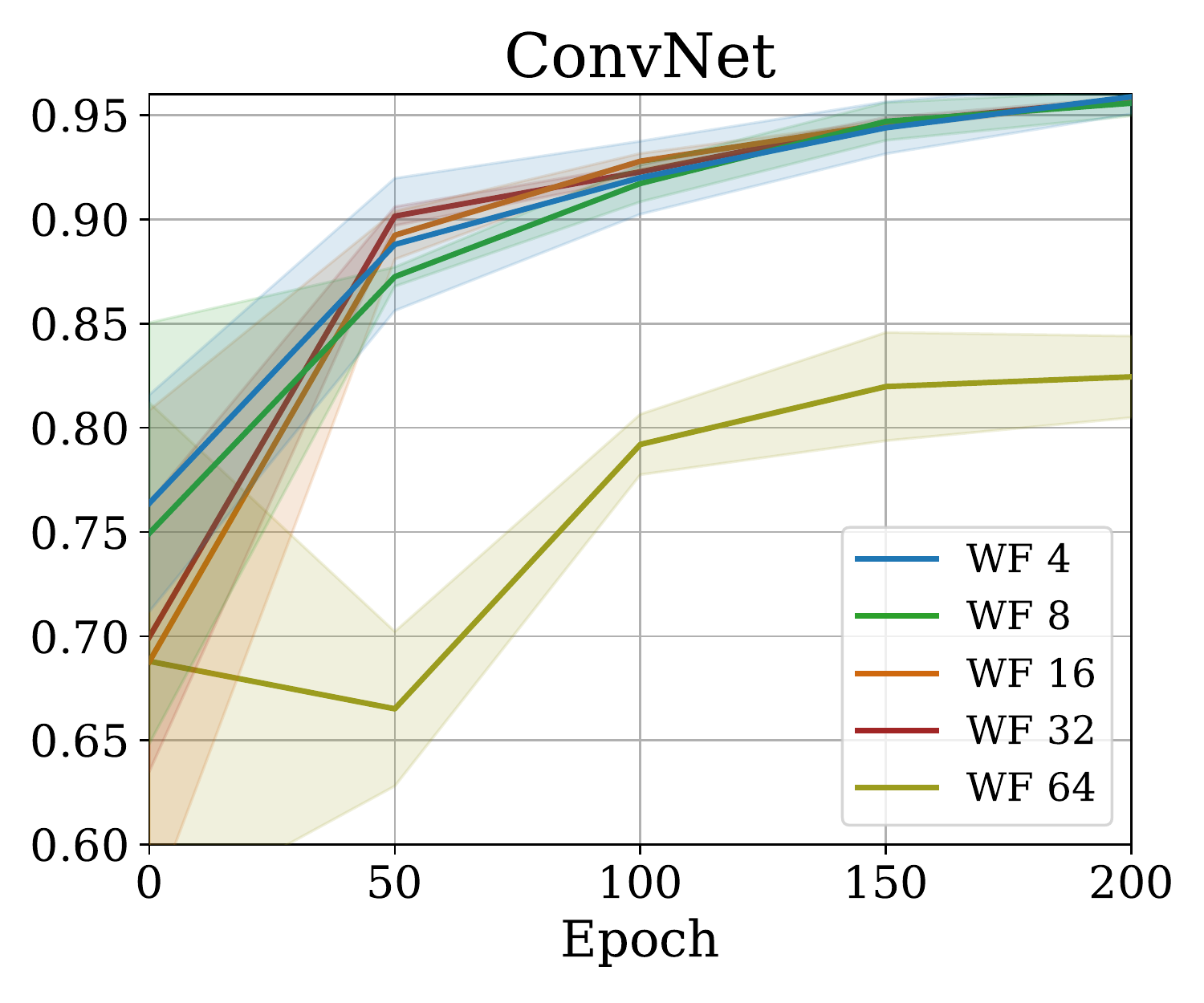}
    \end{subfigure}
    \hfill
    \begin{subfigure}[b]{0.24\textwidth}
        \includegraphics[width=\textwidth]{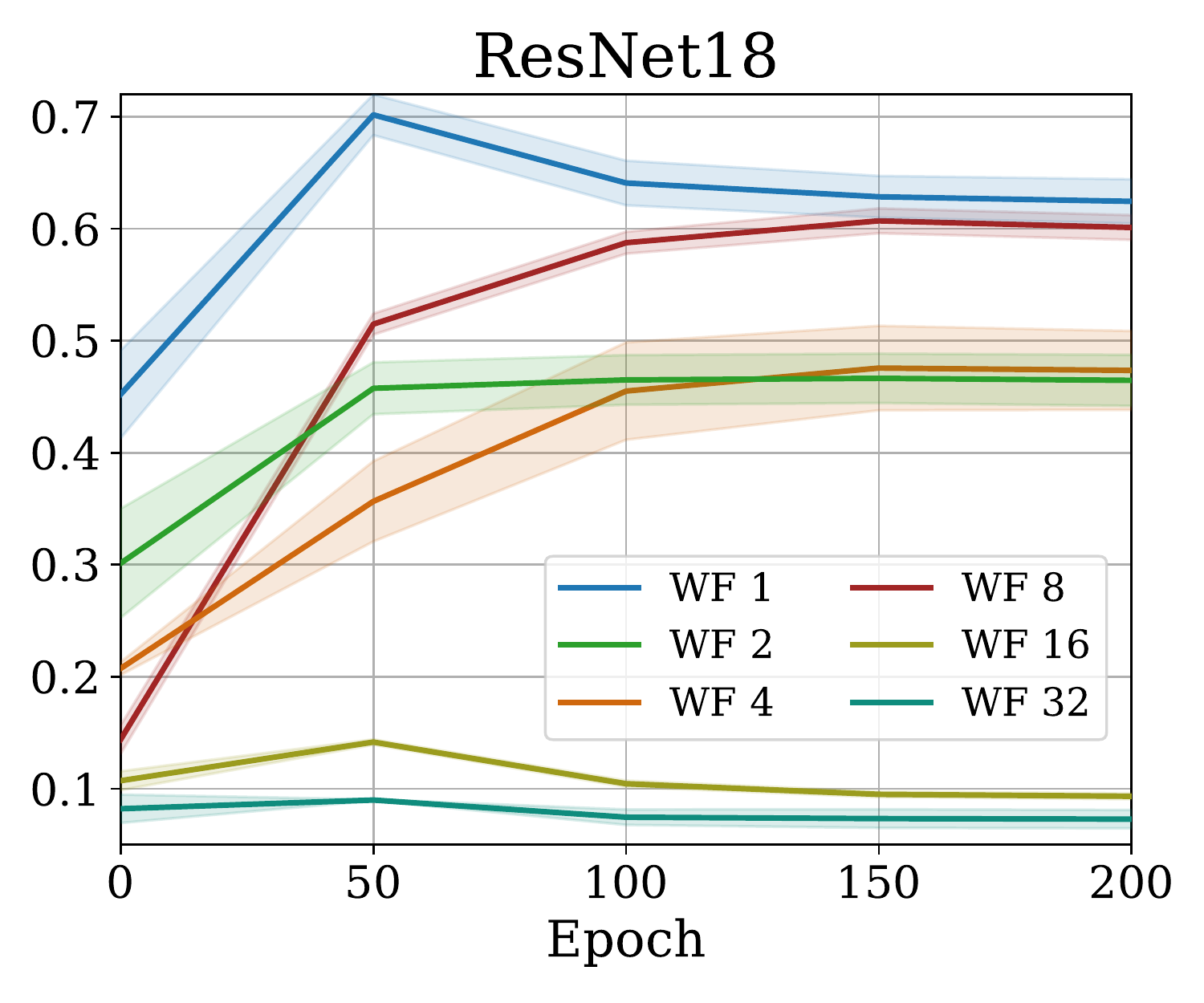}
    \end{subfigure}
    \hfill
    \begin{subfigure}[b]{0.24\textwidth}
        \includegraphics[width=\textwidth]{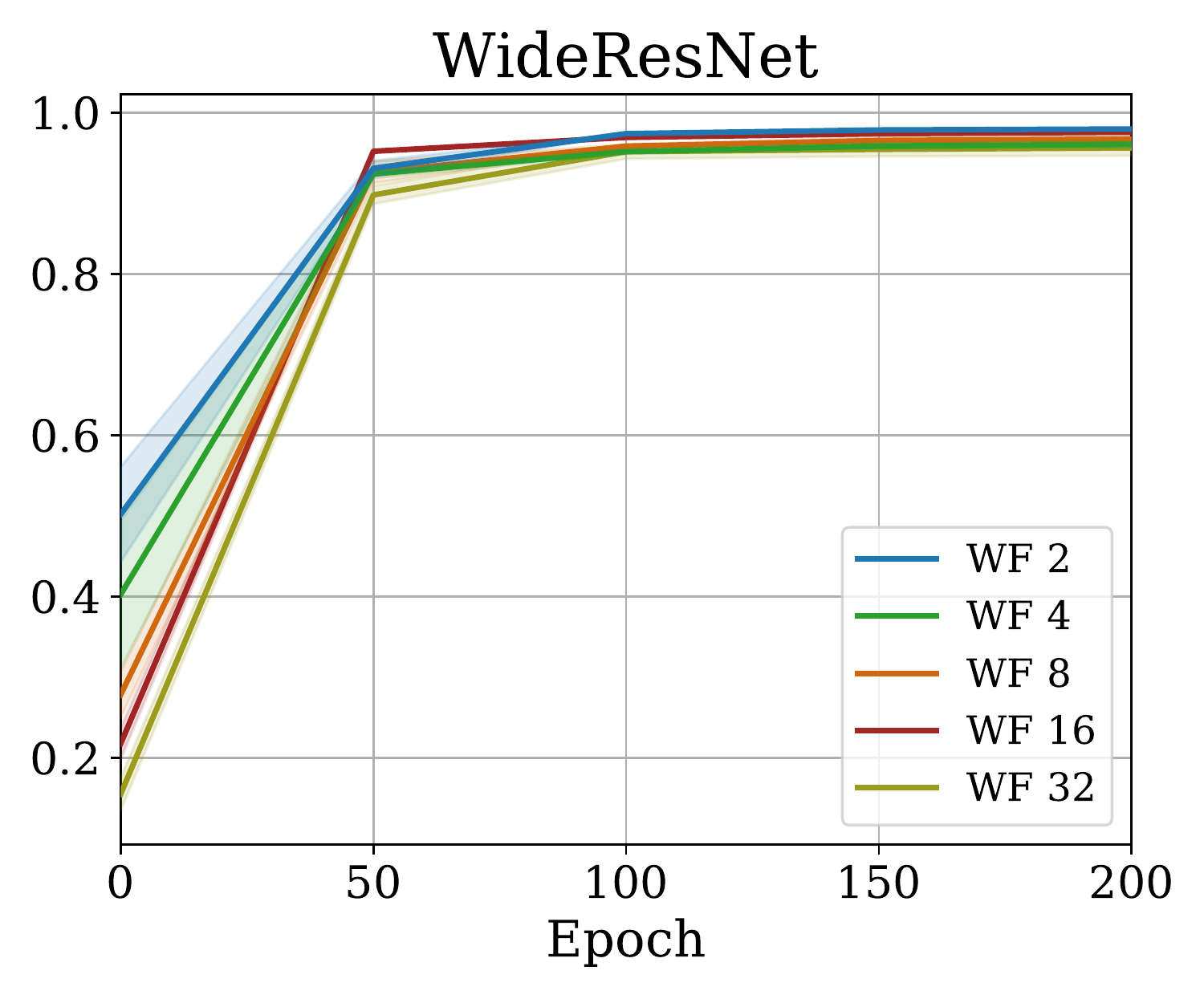}
    \end{subfigure}
    \vspace{-1mm}
    \caption{Evaluating the \textbf{relative difference of Frobenius norm of $\meNTK(\dset, \dset)$ and $\mpNTK(\dset, \dset) \otimes I_O$}
    at initialization and throughout training,
    based on $\dset$ being $1000$ random points from CIFAR-10.
    Wider nets have more similar 
    $\norm{\meNTK}_F$ and $\norm{\mpNTK \otimes I_O}_F$ 
    at initialization.}
    \label{fig:fro_norm_diff}
    \vspace{-2mm}
\end{figure*}

\subsection{pNTK Converges to eNTK as Width Grows}
\label{sec:frob}

The first crucial thing to verify is whether the \pNTK{} kernel matrix approximates the true \eNTK{} as a whole.
We study this first in terms of Frobenius norm.

\begin{theorem}[Informal] \label{theorem:pntk_fro_norm}
Let $f_\theta: \R^{D} \to \R^{O}$ be
a fully-connected network with layers of width $n$ whose parameters are initialized as in \citet{he2016kaiming},
with ReLU-type activations.
Let $\mpNTK(x_1, x_2)$ be the pNTK of $f_\theta$ as in \eqref{eq:pntk_def} and $\meNTK(x_1, x_2)$ the eNTK as in \eqref{eq:jac_ntk}
for a fixed pair of inputs $x_1$, $x_2$.
With high probability over the initialization,
\begin{equation}
    \frac{\norm{\mpNTK(x_1, x_2) \otimes I_O - \meNTK(x_1, x_2)}_F}{\norm{\meNTK(x_1, x_2)}_F} \in \bigO(n^{-\frac{1}{2}}).
\end{equation}
\end{theorem}

\begin{remark}
All of the results in the paper can be straightforwardly extended to networks with different widths, as long as the consecutive layers' widths satisfy $n_{l+1} = \Theta(n_l)$.
Moreover, the results can be made architecturally universal with the techniques of \citet{yang2020tensor,yang2021tensor}. 
\end{remark}

\cref{theorem:pntk_fro_norm} provides the first upper bound on the convergence rate of \pNTK~towards \eNTK. A formal statement for \cref{theorem:pntk_fro_norm} and its proof are in \cref{supp:fro_norm_proof}. 

\begin{remark}
Based on the provided proof, it is straightforward that the ratio of information between off-diagonal and on-diagonal elements of the \eNTK{} matrix converges to zero with a rate of $\bigO(n^{-\frac{1}{2}})$ with high probability over random initialization, as depicted in \cref{fig:diagonality}.
\end{remark}

\begin{figure*}[!ht]
    \centering
    \begin{subfigure}[b]{0.24\textwidth}
        \includegraphics[width=\textwidth]{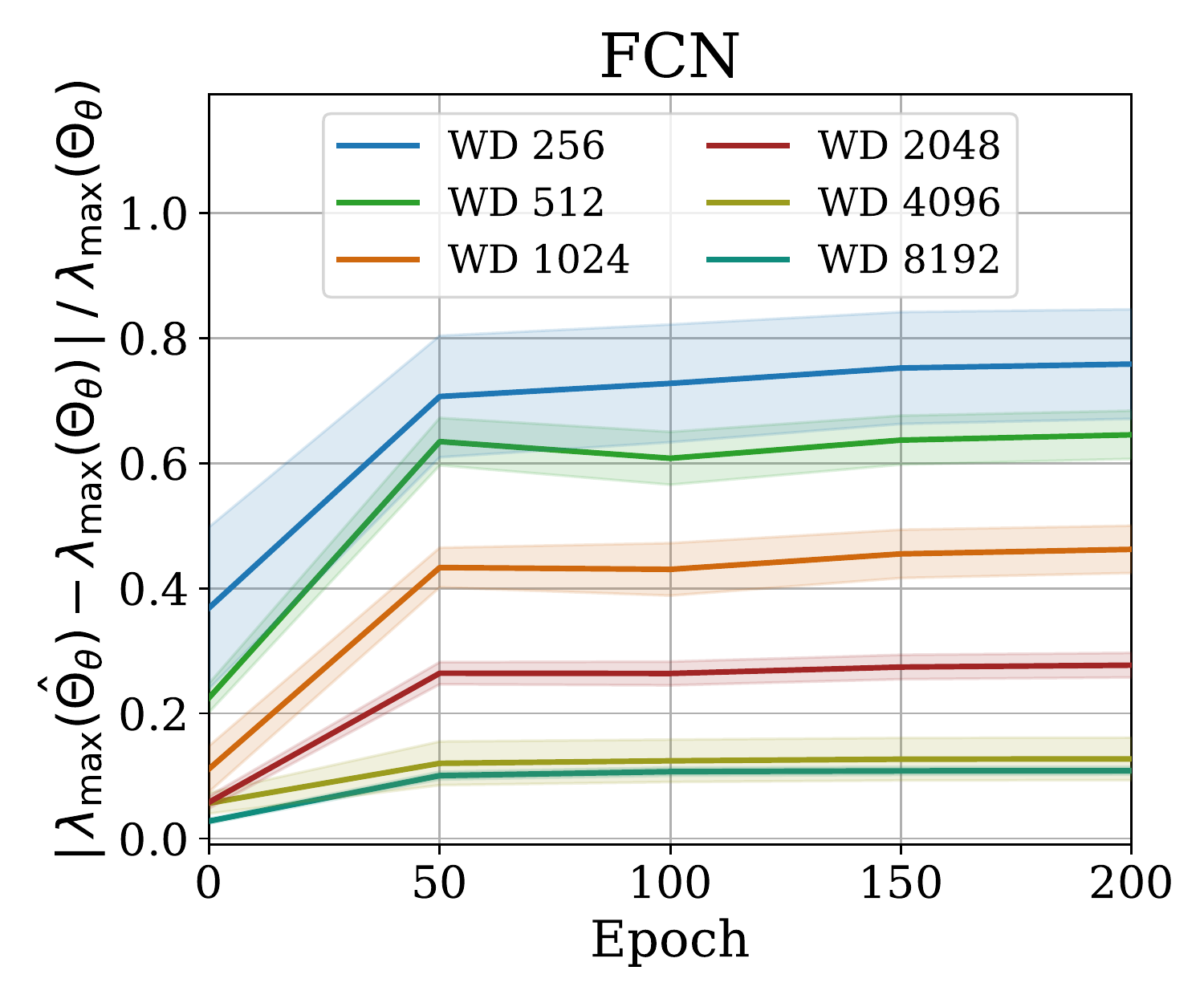}
    \end{subfigure}
    \hfill
    \begin{subfigure}[b]{0.24\textwidth}
        \includegraphics[width=\textwidth]{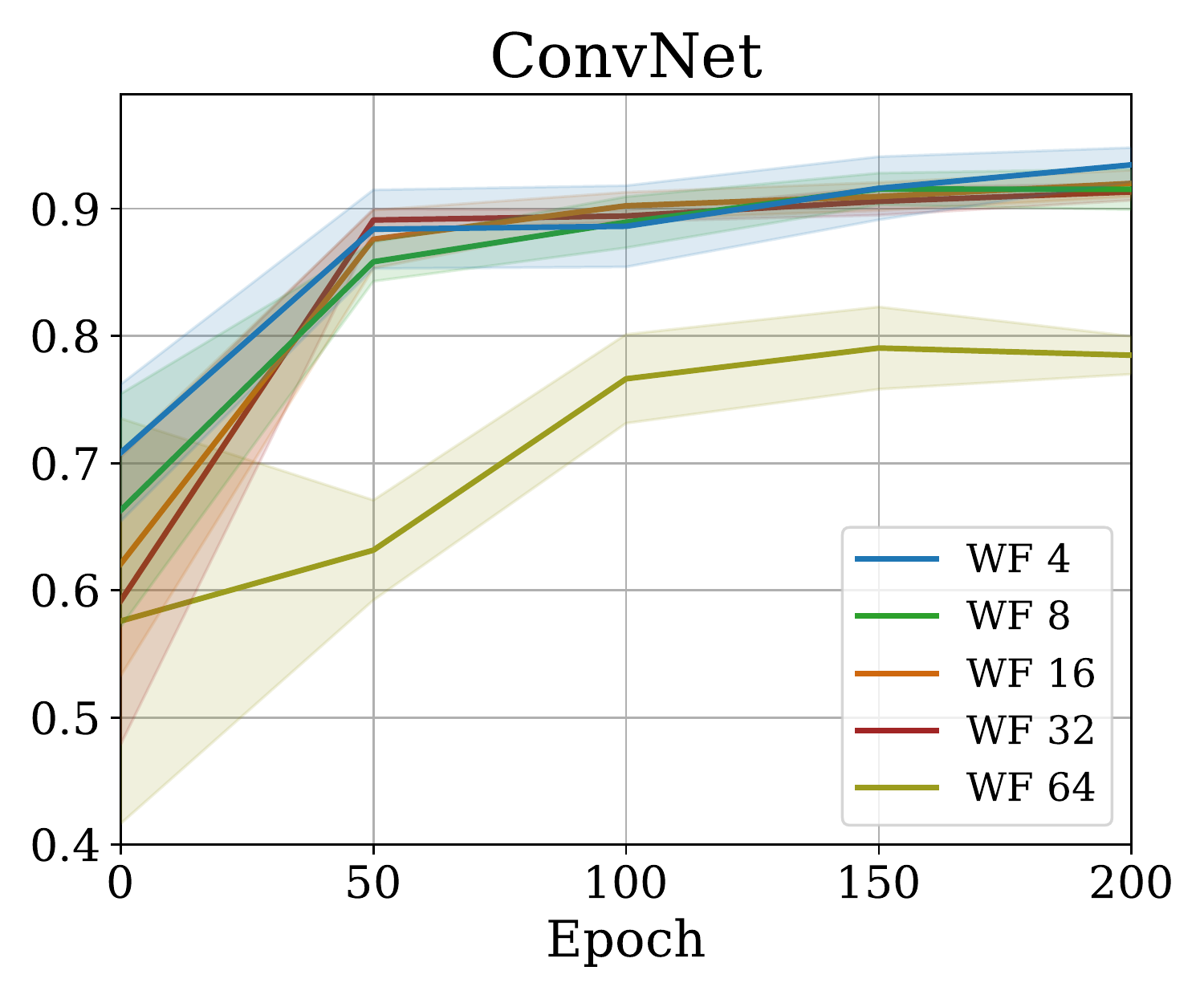}
    \end{subfigure}
    \hfill
    \begin{subfigure}[b]{0.24\textwidth}
        \includegraphics[width=\textwidth]{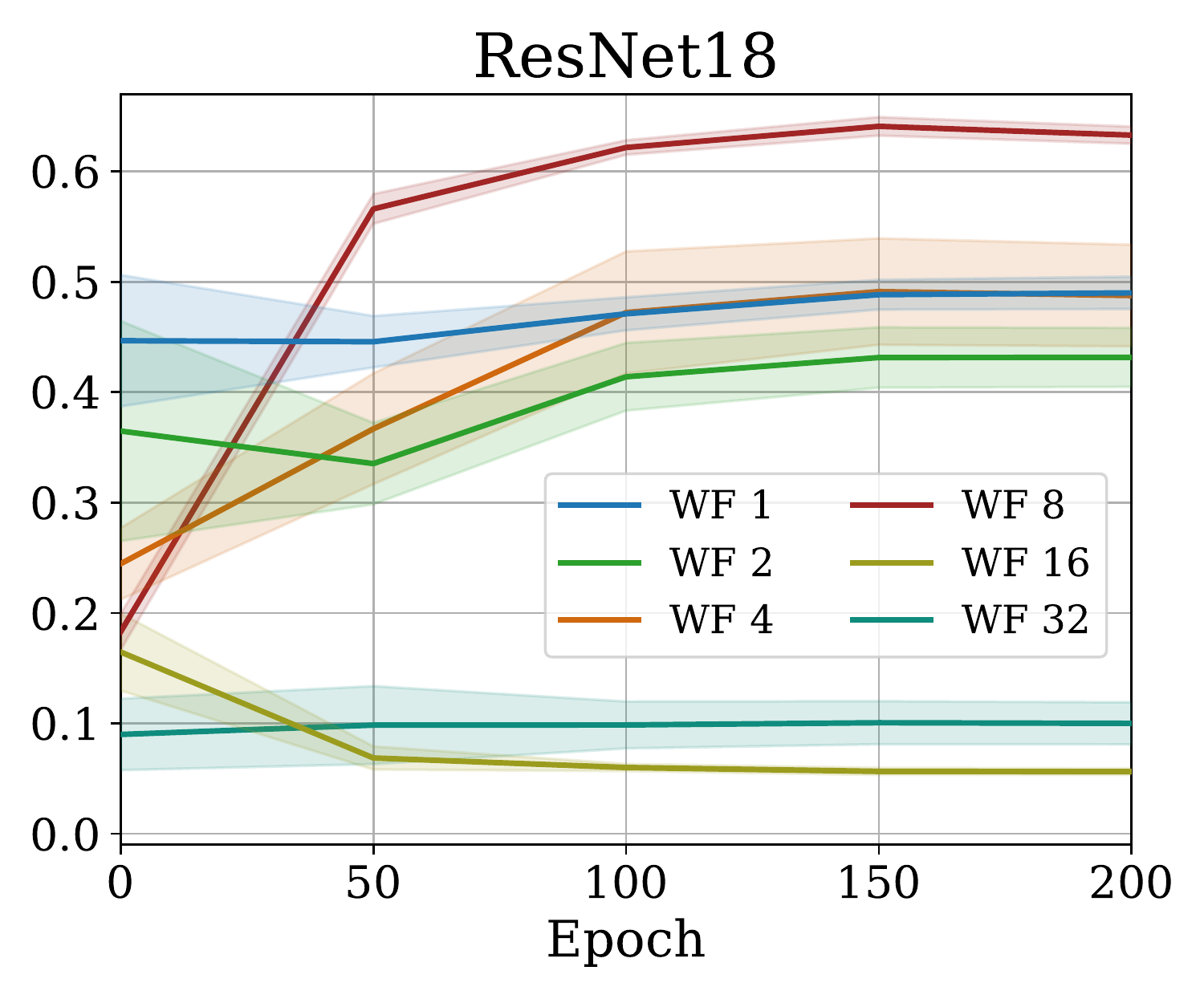}
    \end{subfigure}
    \hfill
    \begin{subfigure}[b]{0.24\textwidth}
        \includegraphics[width=\textwidth]{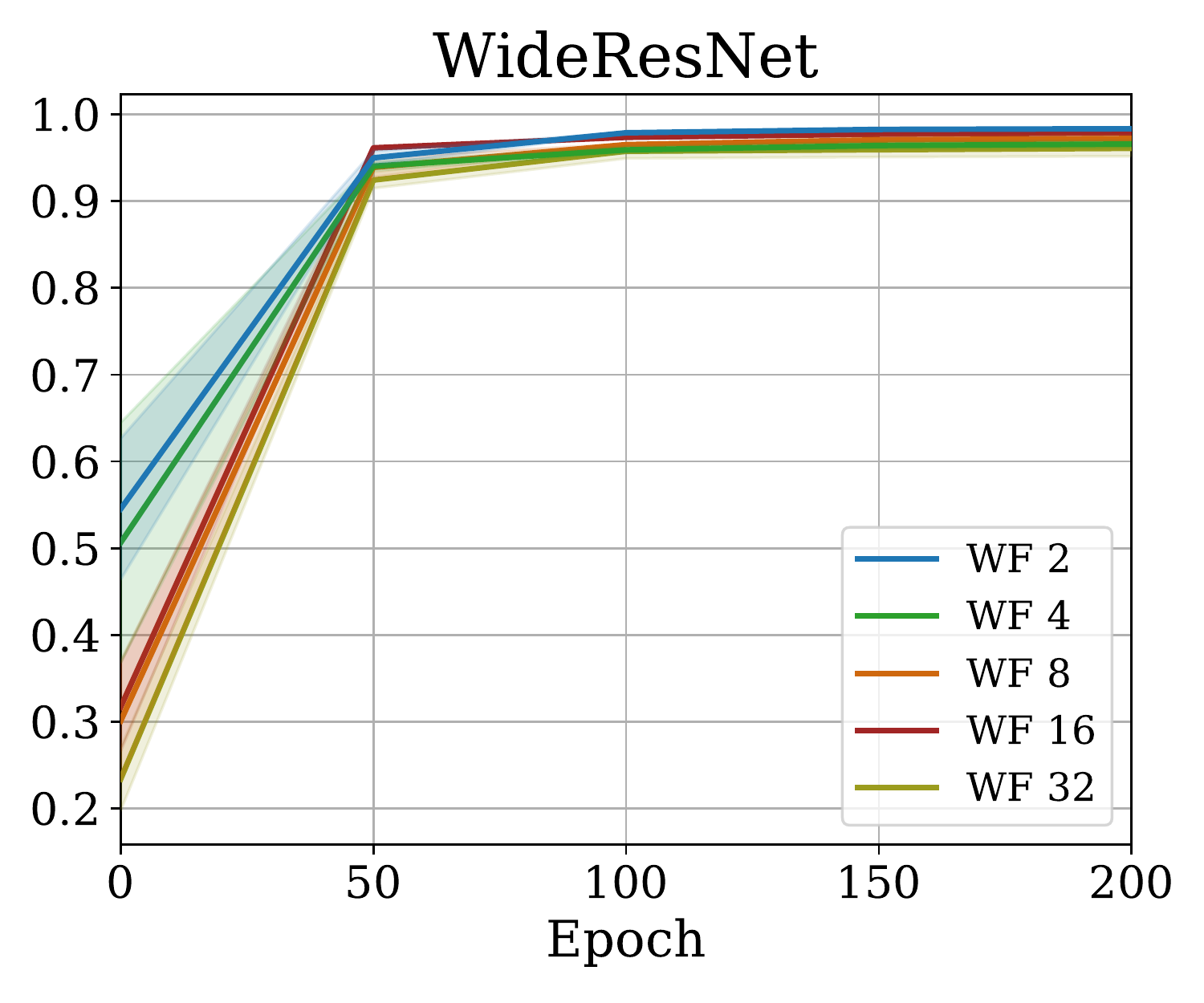}
    \end{subfigure}
    \vspace{-1mm}
    \caption{Evaluating the \textbf{relative difference of $\lambda_{\max}$ of $\meNTK(\dset, \dset)$ and $\mpNTK(\dset, \dset)$}
    at initialization and throughout training,
    based on $\dset$ being 1000 random points from CIFAR-10.
    Wider nets have more similar
    $\lambda_{\max} (\meNTK(\dset, \dset))$ and $\lambda_{\max} (\mpNTK(\dset, \dset))$.}
    \label{fig:max_eigval_diff}
    \vspace{-1mm}
\end{figure*}

\begin{figure*}[!ht]
    \centering
    \begin{subfigure}[b]{0.24\textwidth}
        \includegraphics[width=\textwidth]{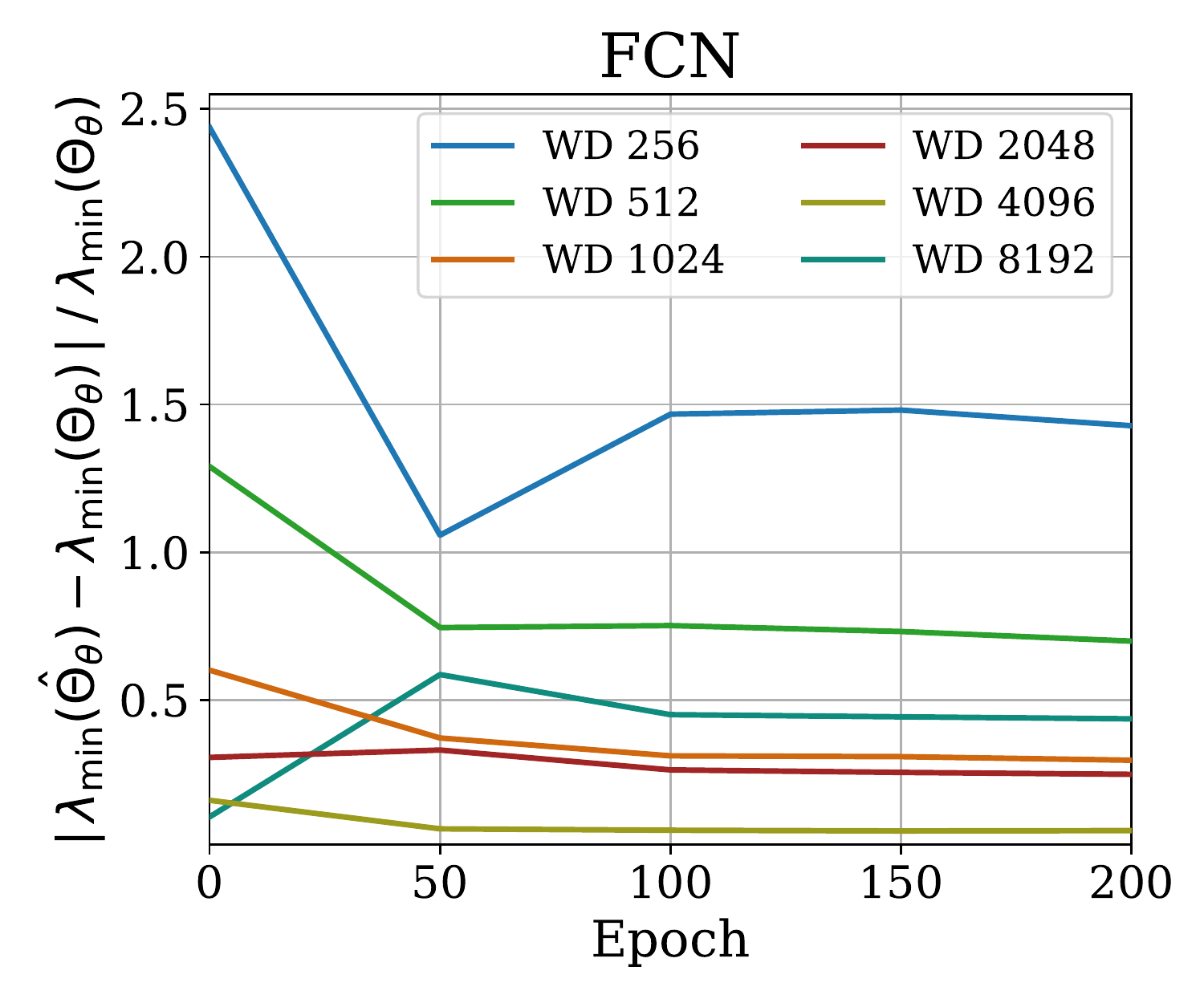}
    \end{subfigure}
    \hfill
    \begin{subfigure}[b]{0.24\textwidth}
        \includegraphics[width=\textwidth]{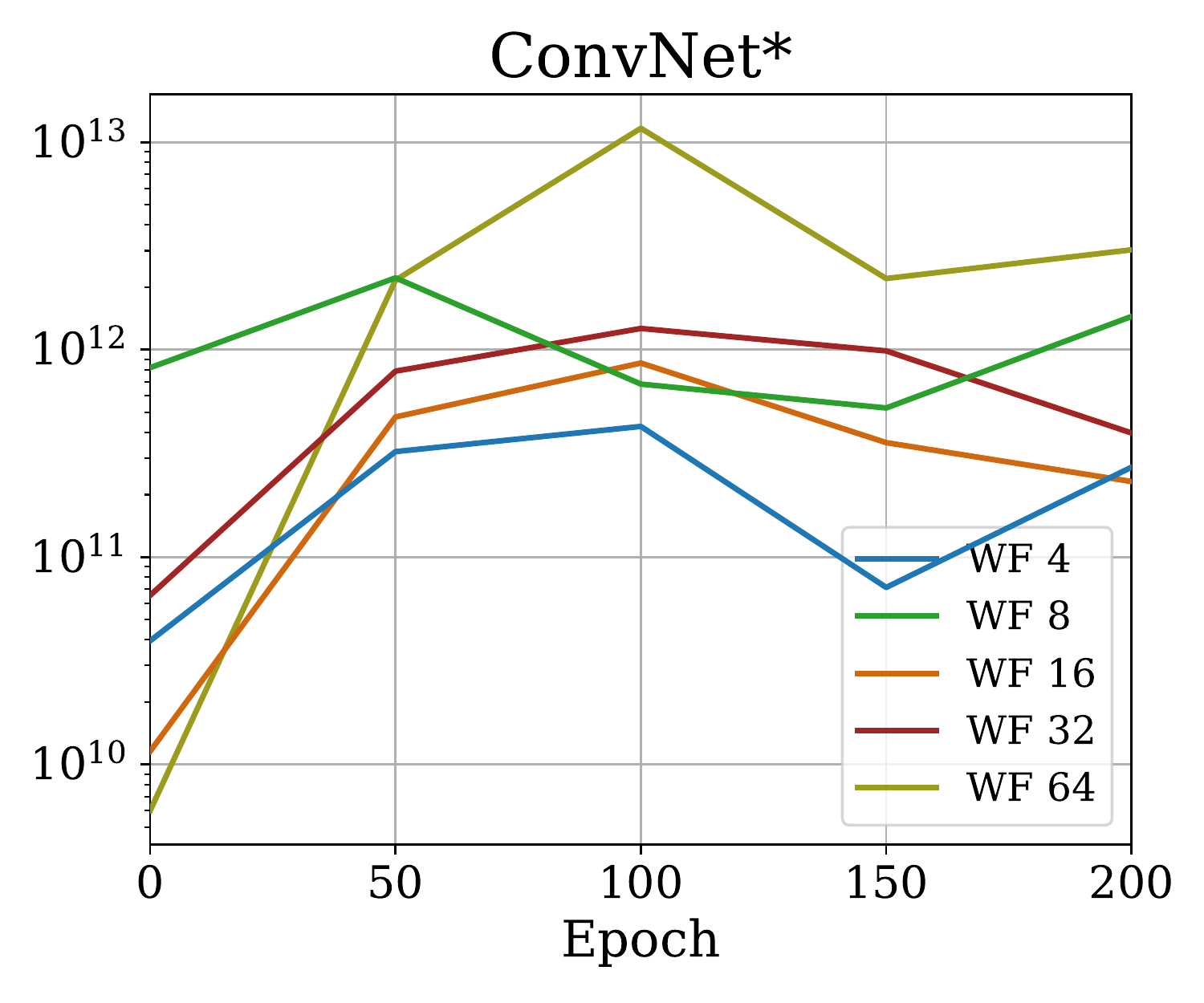}
    \end{subfigure}
    \hfill
    \begin{subfigure}[b]{0.24\textwidth}
        \includegraphics[width=\textwidth]{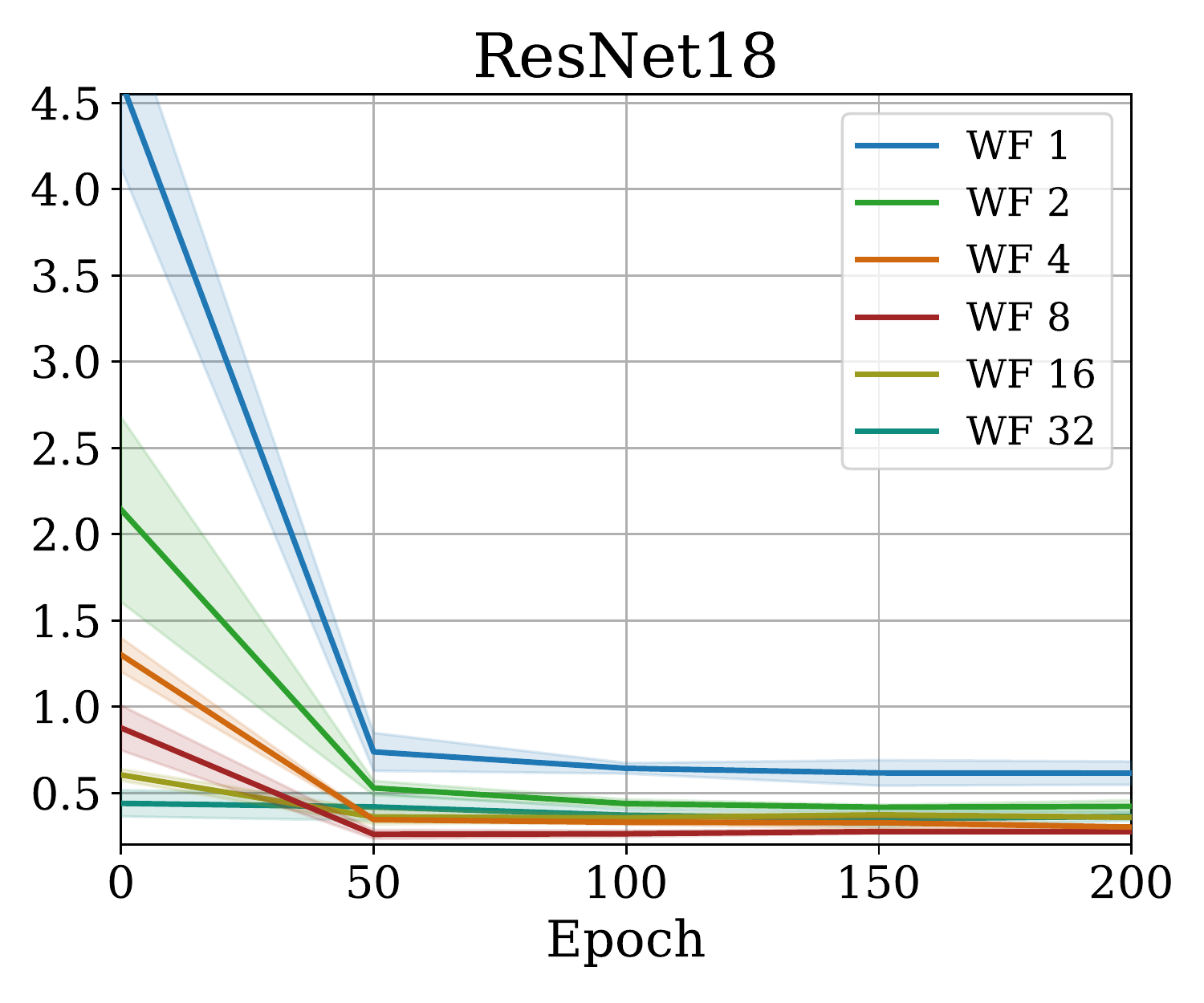}
    \end{subfigure}
    \hfill
    \begin{subfigure}[b]{0.24\textwidth}
        \includegraphics[width=\textwidth]{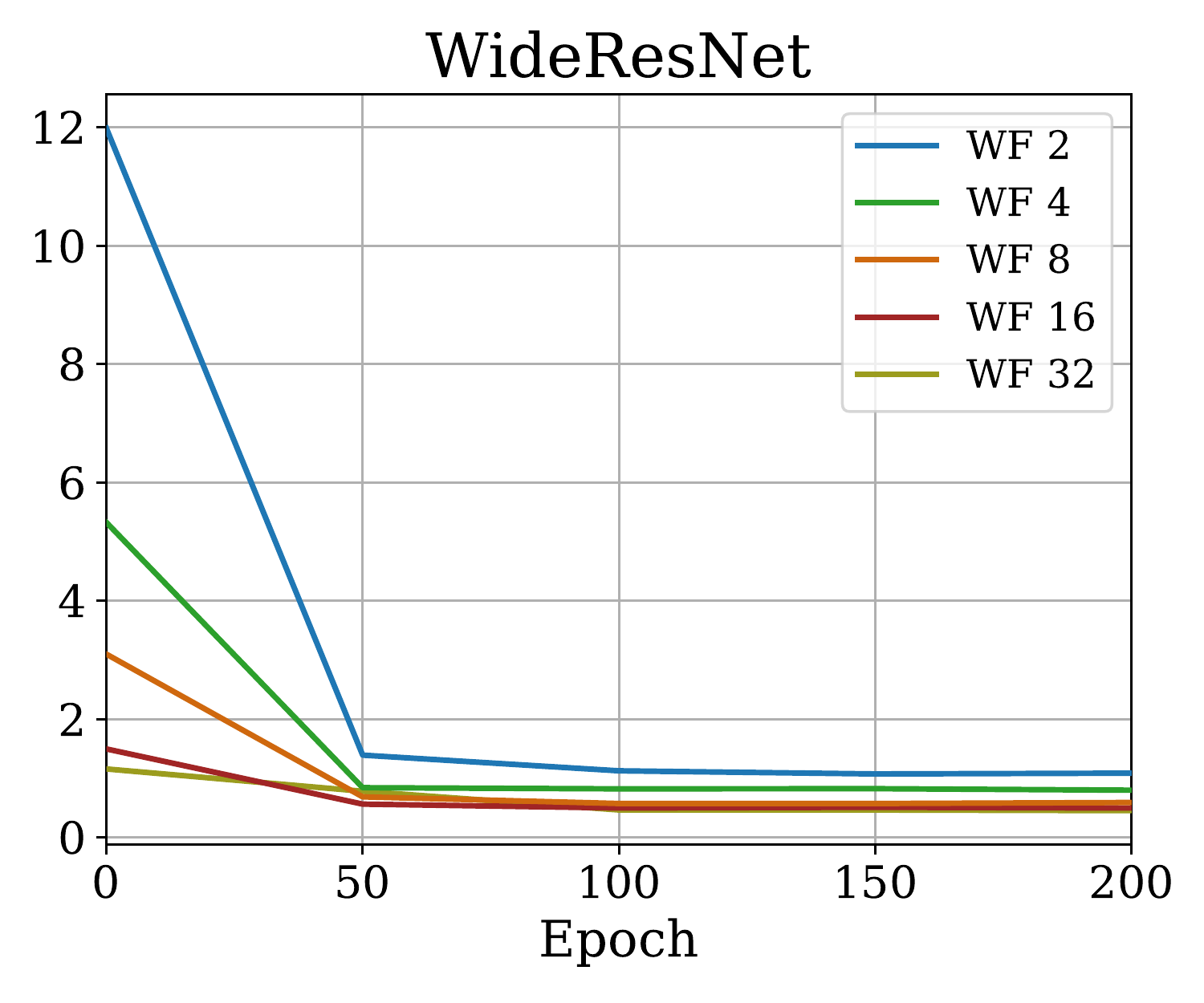}
    \end{subfigure}
    \vspace{-1mm}
    \caption{Evaluating the \textbf{relative difference of $\lambda_{\min}$ of $\meNTK(\dset, \dset)$ and $\mpNTK(\dset, \dset)$} based on $\dset$ being 1000 random points from CIFAR-10.
    Wider nets have more similar $\lambda_{\min}(\meNTK(\dset, \dset))$ and $\lambda_{\min}(\mpNTK(\dset, \dset))$.
    Note, though, the extremely large values reported for ConvNet; as observed by \citet{lee2020finite} and \citet{xiao2020disentangling}, it is ill-conditioned and $\lambda_{\min} (\meNTK(\dset, \dset)) \to 0$, while $\lambda_{\min} (\mpNTK(\dset, \dset)) > 0.001$, causing the huge discrepancy. More details in \cref{app:eigs}.}
    \label{fig:min_eigval_diff}
    \vspace{-1mm}
\end{figure*}

\Cref{fig:diagonality} shows that as the width grows, the sum of off-diagonal elements of $\meNTK(x_1, x_2)$ becomes small compared to the diagonal.
Furthermore, \cref{fig:fro_norm_diff} provides experimental support that as width grows, $\mpNTK \otimes I_O$ converges to $\meNTK$ in terms of relative Frobenius norm.

\Cref{theorem:pntk_fro_norm} only applies to epoch zero of these figures, as it assumes networks with weights at initialization.
As can be seen in the figures, these results don't necessarily apply to the NNs not at initialization (i.e., after a few epochs of training). This naturally gives rise to the question: \textit{Can the \pNTK~be used to analyze and represent NNs whose parameters are far from initialization?} We will now take various experimental approaches towards studying this question. %

\subsection{Largest Eigenvalue Converges as Width Grows}
\label{sec:eigs}

As discussed before, the conditioning of a network's eNTK has been shown to be closely related to generalization properties of the network, such as trainability and generalization risk \citep{xiao2018dynamical,xiao2020disentangling,wei2022more}.
Thus, we would like to know how well the pNTK's eigenspectrum approximates that of the eNTK. The following theorem gives a bound on the rate of convergence between the maximum eigenvalues of the two kernels.

\begin{theorem}[Informal] \label{theorem:pntk_max_eigval}
Let $f_\theta: \R^{D} \to \R^{O}$ be
a fully-connected network with layers of width $n$ whose parameters are initialized according to \citet{he2016kaiming} initialization,
with ReLU-type activations.
Let $\mpNTK(x_1, x_2)$ be the corresponding pNTK of $f_\theta$ as in \eqref{eq:pntk_def} and $\meNTK(x_1, x_2)$ the corresponding eNTK as in \eqref{eq:jac_ntk}
for a fixed pair of inputs $x_1$, $x_2$.
With high probability over the initialization,
\begin{equation*}
    \frac{\lambda_{\max} ( \mpNTK(x_1, x_2) \otimes I_O ) - \lambda_{\max} \left( \meNTK \left(x_1, x_2 \right) \right) }{ \lambda_{\max} \left(\meNTK \left(x_1, x_2 \right) \right) } \in \bigO(n^{-\frac{1}{2}}).
\end{equation*}
\end{theorem}

\begin{figure*}[!ht]
    \centering
    \begin{subfigure}[b]{0.24\textwidth}
        \includegraphics[width=\textwidth]{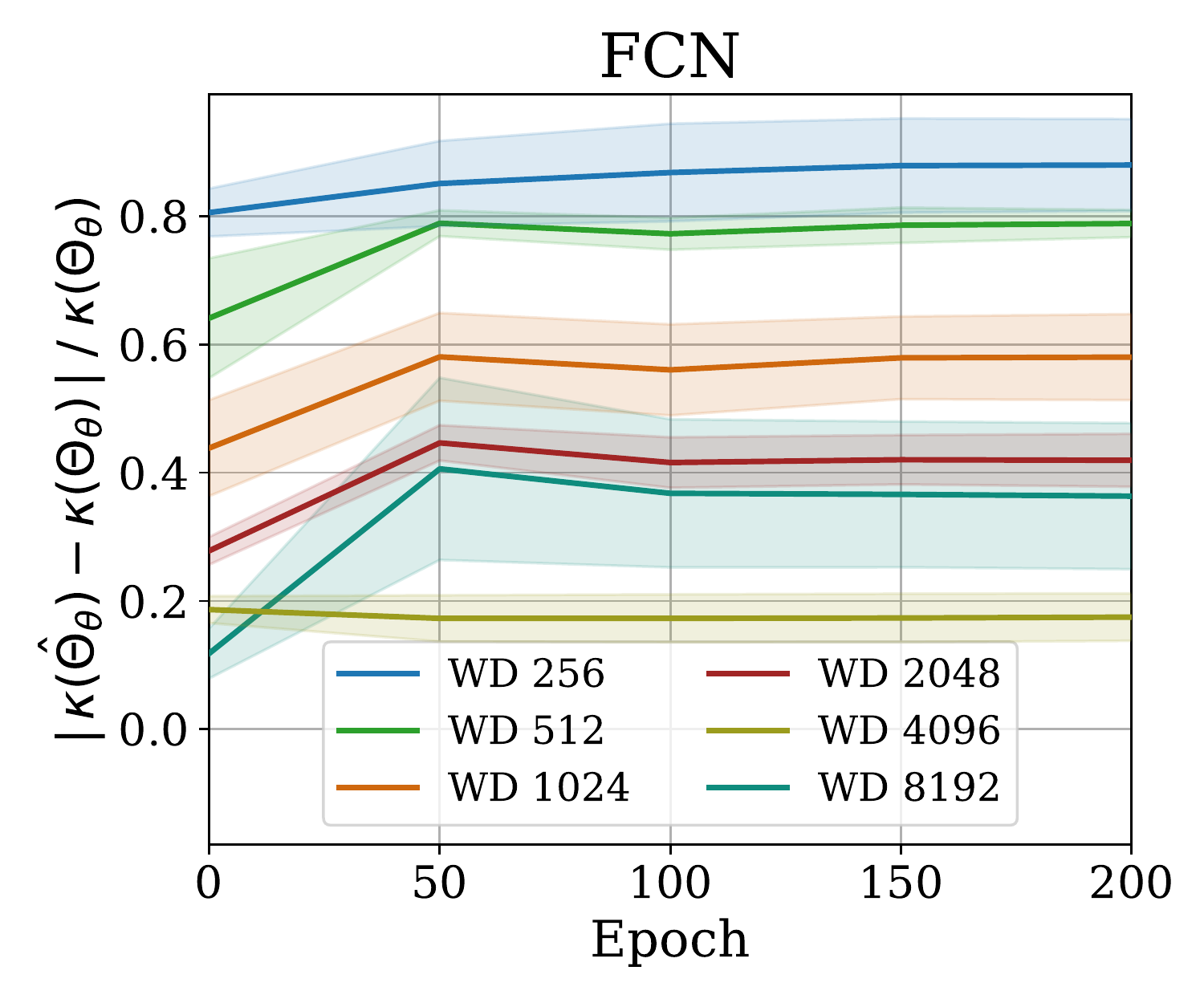}
    \end{subfigure}
    \hfill
    \begin{subfigure}[b]{0.24\textwidth}
        \includegraphics[width=\textwidth]{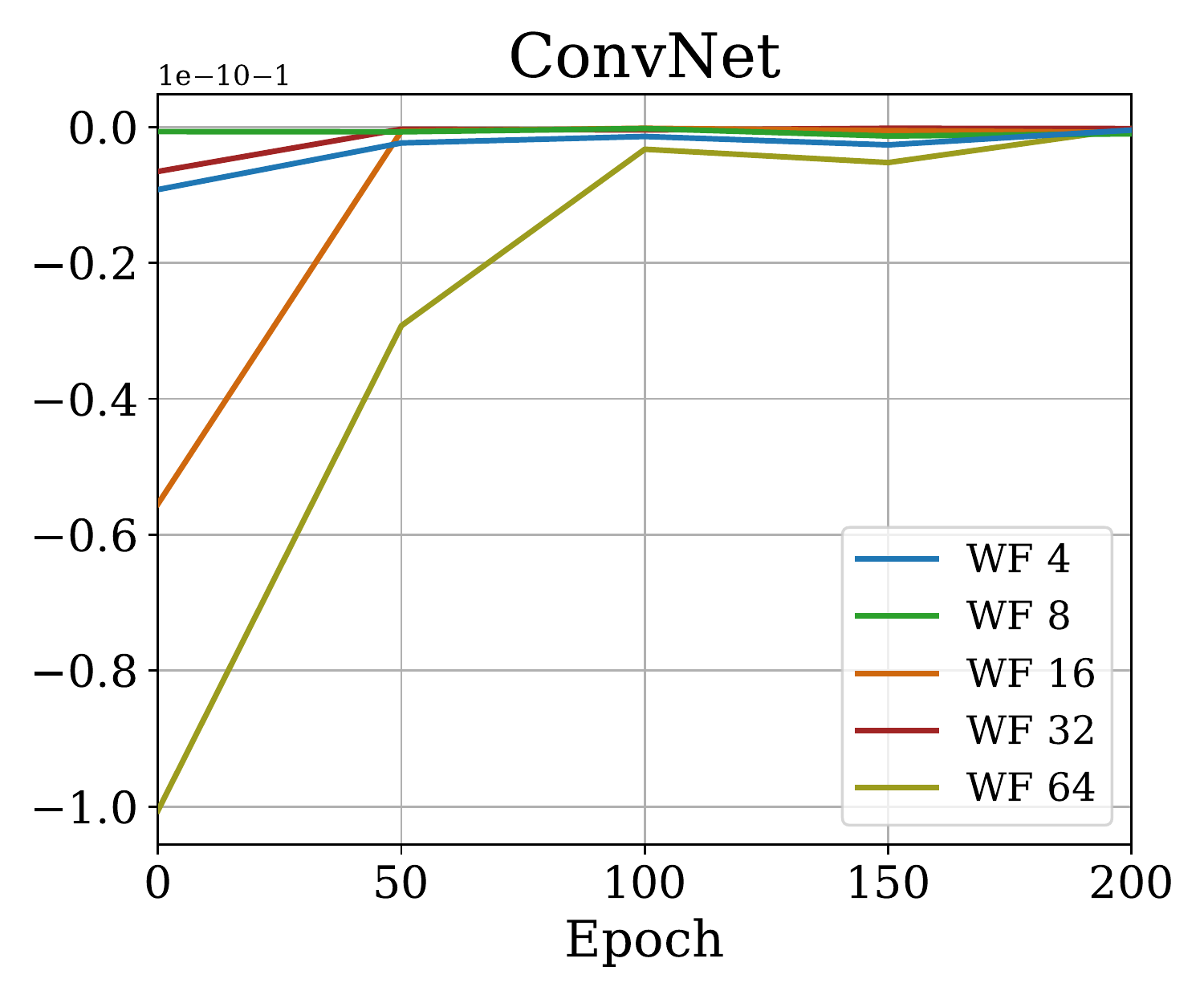}
    \end{subfigure}
    \hfill
    \begin{subfigure}[b]{0.24\textwidth}
        \includegraphics[width=\textwidth]{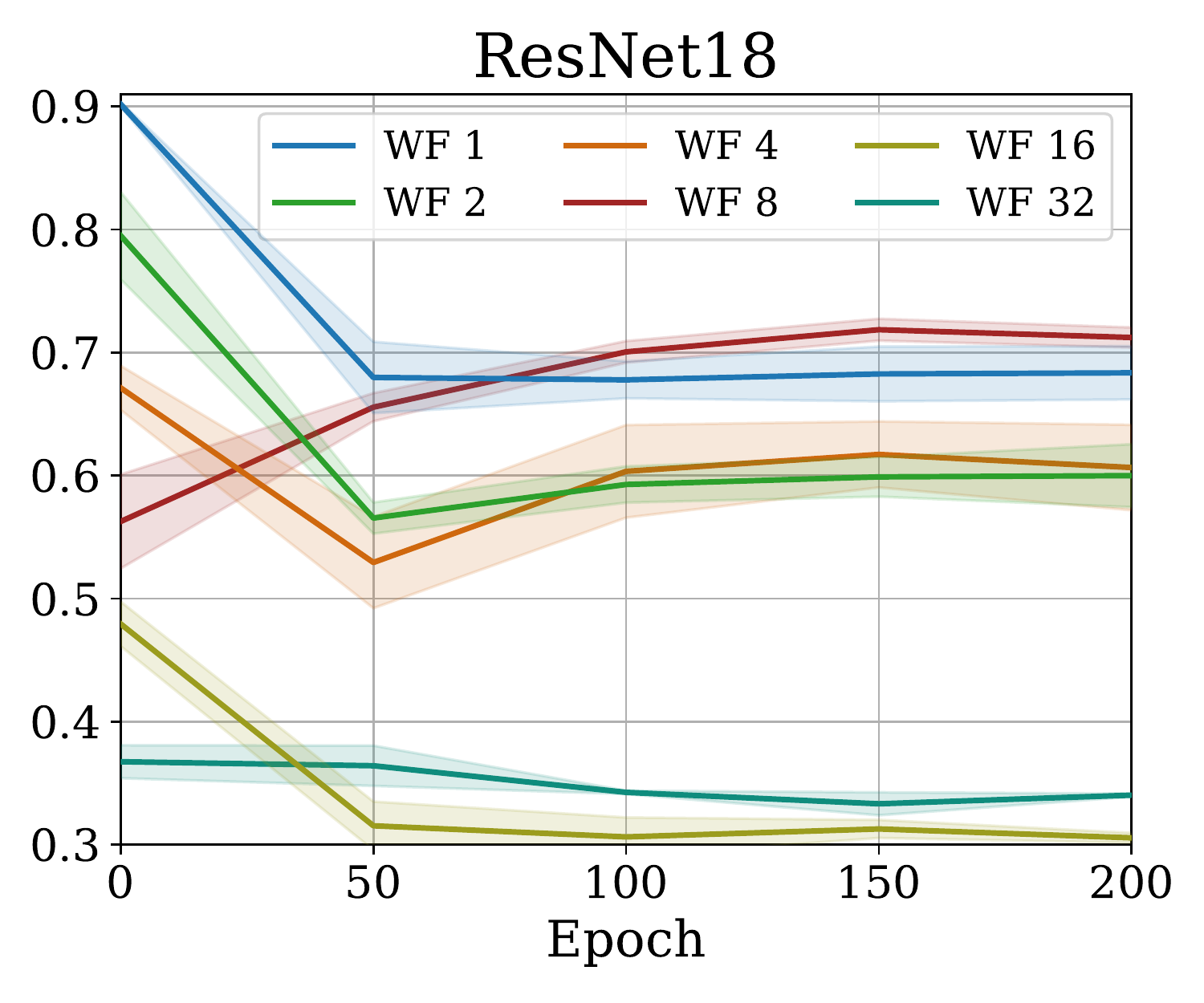}
    \end{subfigure}
    \hfill
    \begin{subfigure}[b]{0.24\textwidth}
        \includegraphics[width=\textwidth]{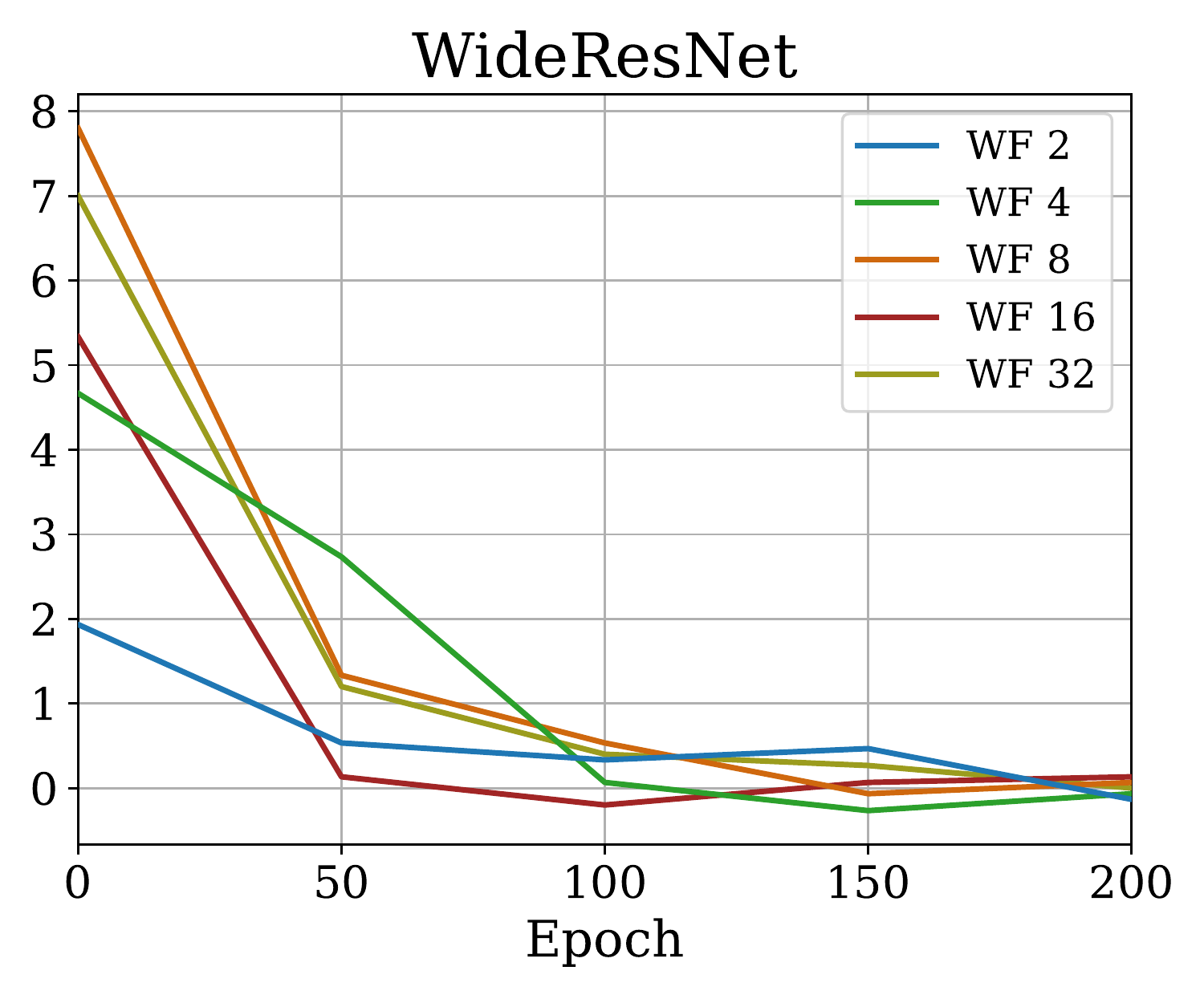}
    \end{subfigure}
    \vspace{-1mm}
    \caption{The \textbf{relative difference in condition number},
    $\kappa(K) = \lambda_{\max}(K) / \lambda_{\min}(K)$,
    decreases for wider nets.
    The strange ConvNet plot is due to the issue in \cref{fig:min_eigval_diff}; more details in \cref{app:eigs}.}
    \vspace{-1mm}
    \label{fig:cond_number_diff}
\end{figure*}

\cref{theorem:pntk_max_eigval} bounds the difference between the maximum eigenvalue of \pNTK~and the maximum eigenvalue of \eNTK~based on the NN's width, for networks at initialization. A formal statement for \cref{theorem:pntk_max_eigval} and its proof are given in \cref{app:eigs}.
\Cref{fig:max_eigval_diff} also supports this trend experimentally. %

\Cref{fig:min_eigval_diff} shows a similar trend for the minimum eigenvalues, although we have not found a proof of this convergence.
This suggests that the condition number $\kappa = \lambda_{\max} / \lambda_{\min}$ should become similar as width grows; this is also supported by results in \cref{fig:cond_number_diff}.

Interestingly, the rate of increase/decrease in the difference between maximum and minimum eigenvalues and the condition numbers between \pNTK~and \eNTK~do not \textcolor{blue}{necessarily have} a monotonic behaviour as the training goes on. Observing the exact values of $\lambda_{\min}$, $\lambda_{\max}$, and $\kappa$ for different architectures, widths at initialization and throughout training reveals that in ConvNet, WideResNet and ResNet18 architectures, $\lambda_{\min}$ is close to zero at initialization, but grows during training; the inverse phenomenon is observed with FCNs. Further investigations of these statistics might reveal interesting insights about the behaviour of NNs trained with SGD and the connections between \eNTK{} and trainibility of the architecture.

\subsection{Kernel Regression Using pNTK vs. eNTK} 
\label{sec:kernel-regression}

\begin{figure*}[!ht]
    \centering
    \begin{subfigure}[b]{0.24\textwidth}
        \includegraphics[width=\textwidth]{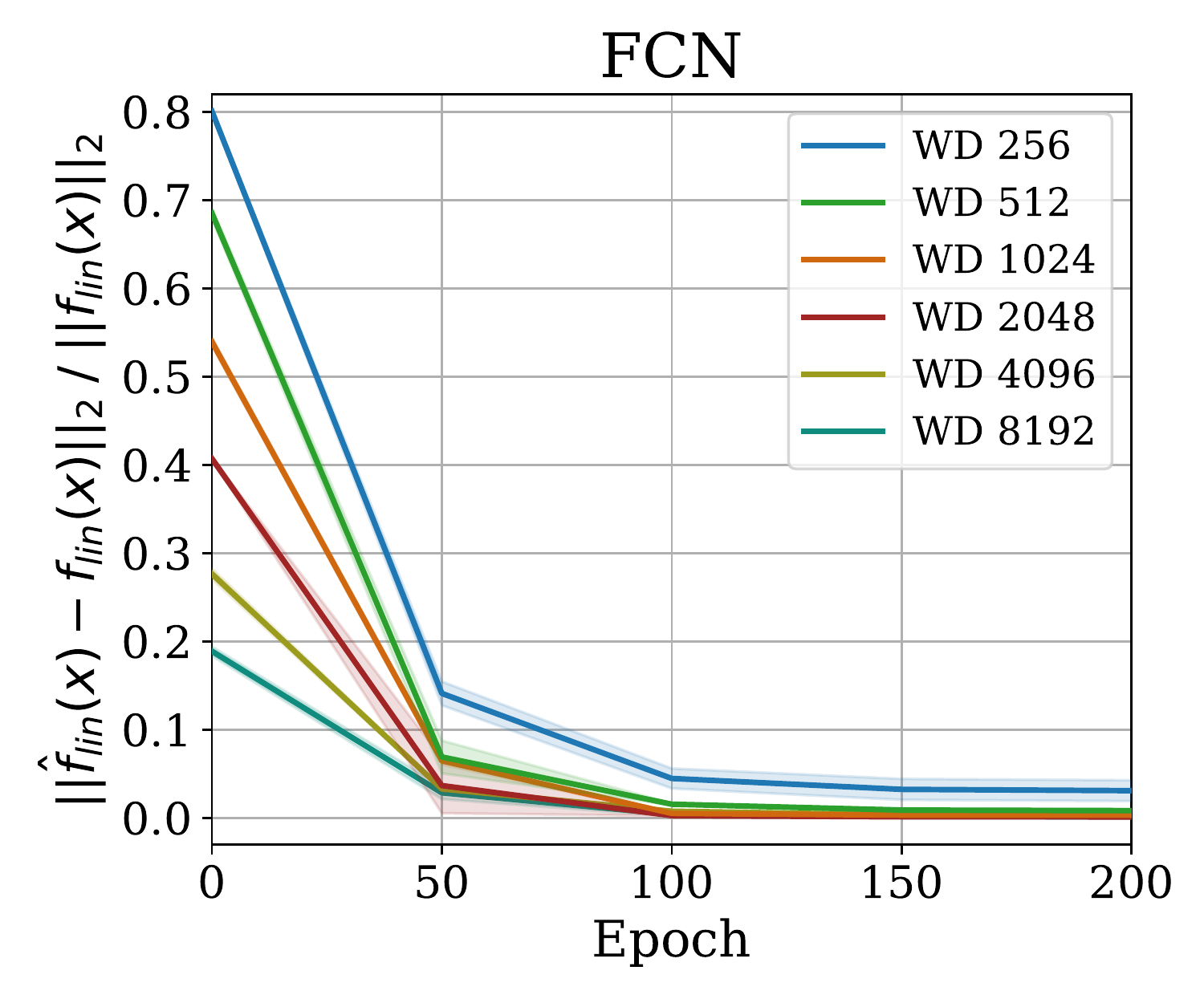}
    \end{subfigure}
    \hfill
    \begin{subfigure}[b]{0.24\textwidth}
        \includegraphics[width=\textwidth]{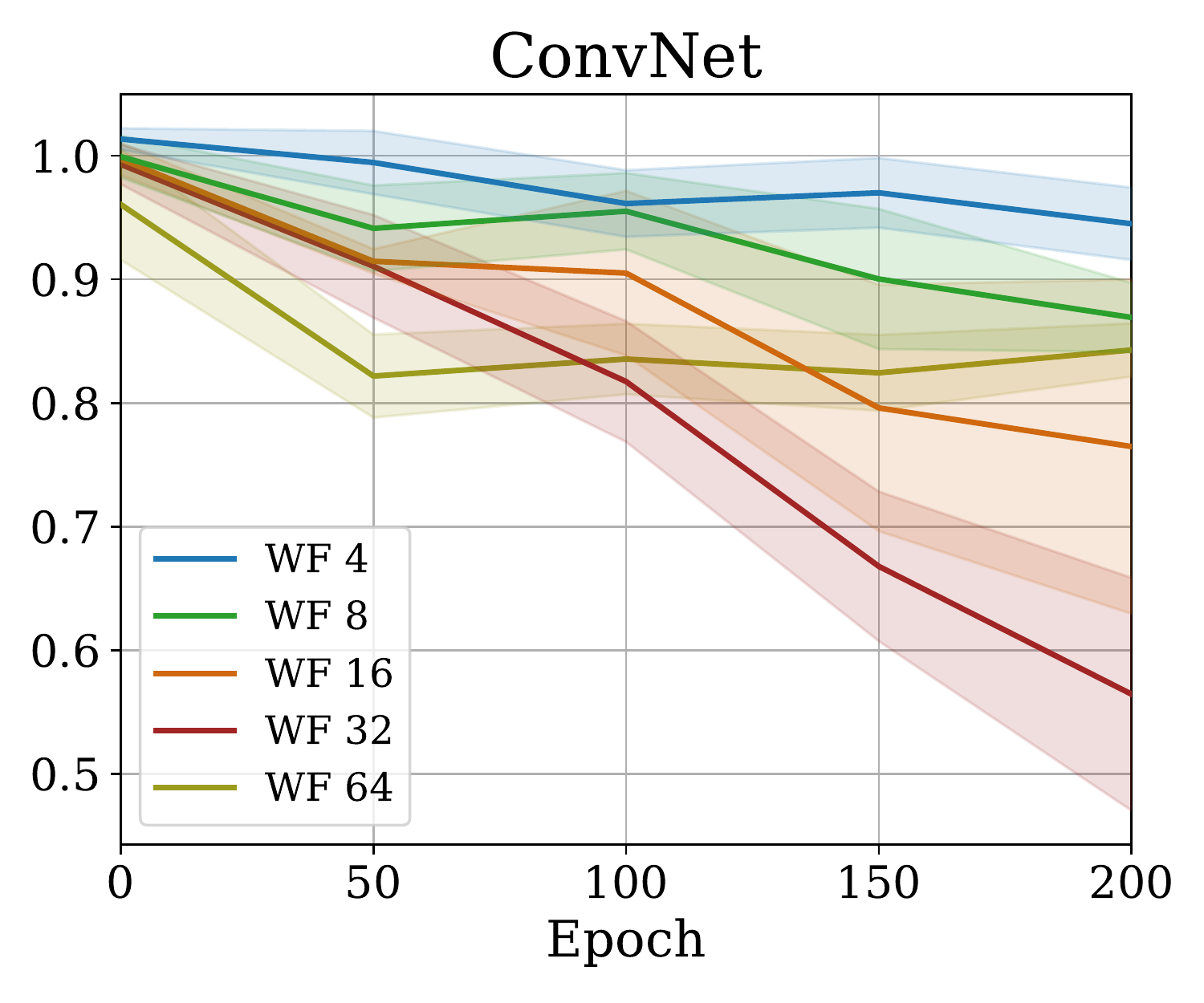}
    \end{subfigure}
    \hfill
    \begin{subfigure}[b]{0.24\textwidth}
        \includegraphics[width=\textwidth]{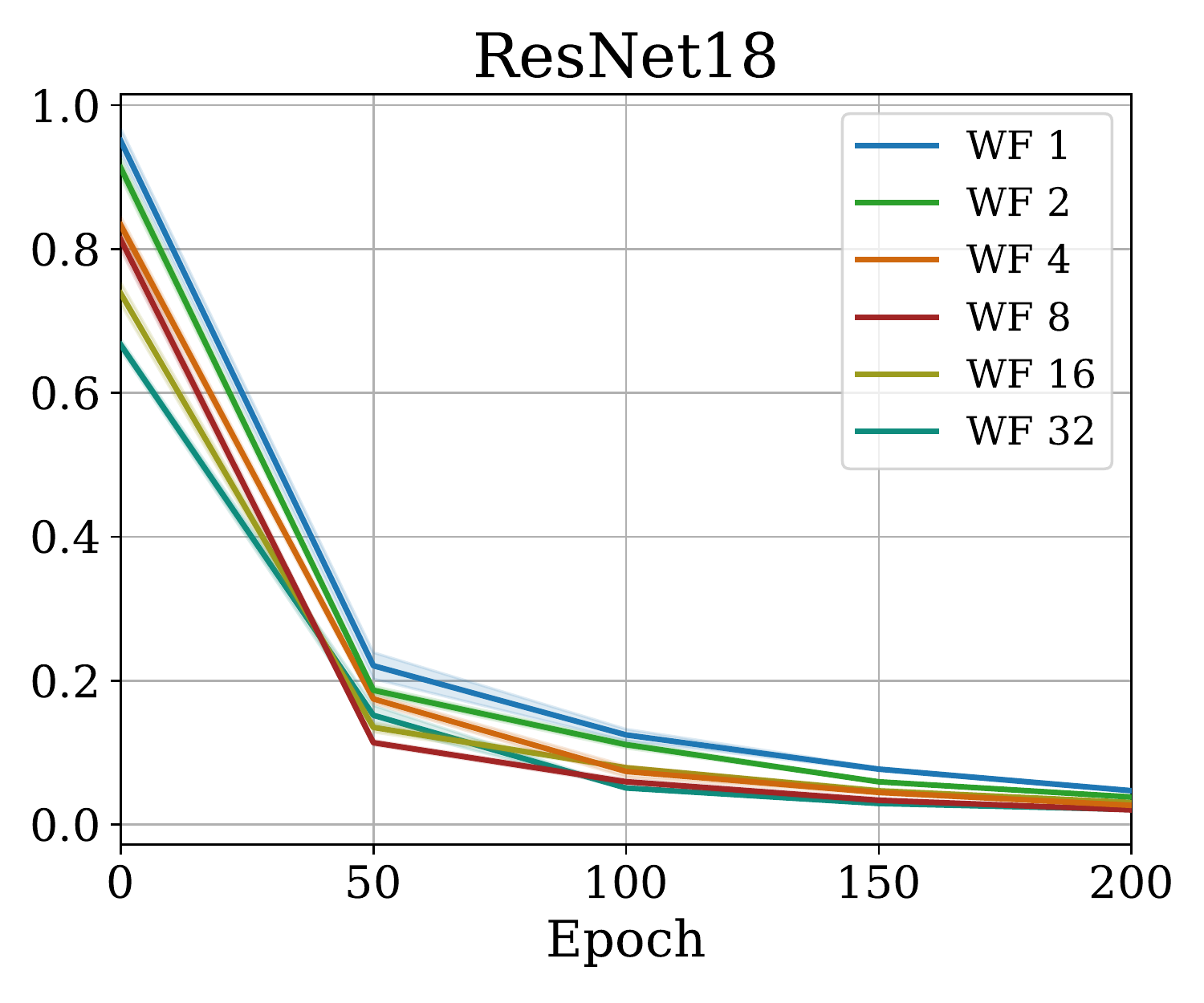}
    \end{subfigure}
    \hfill
    \begin{subfigure}[b]{0.24\textwidth}
        \includegraphics[width=\textwidth]{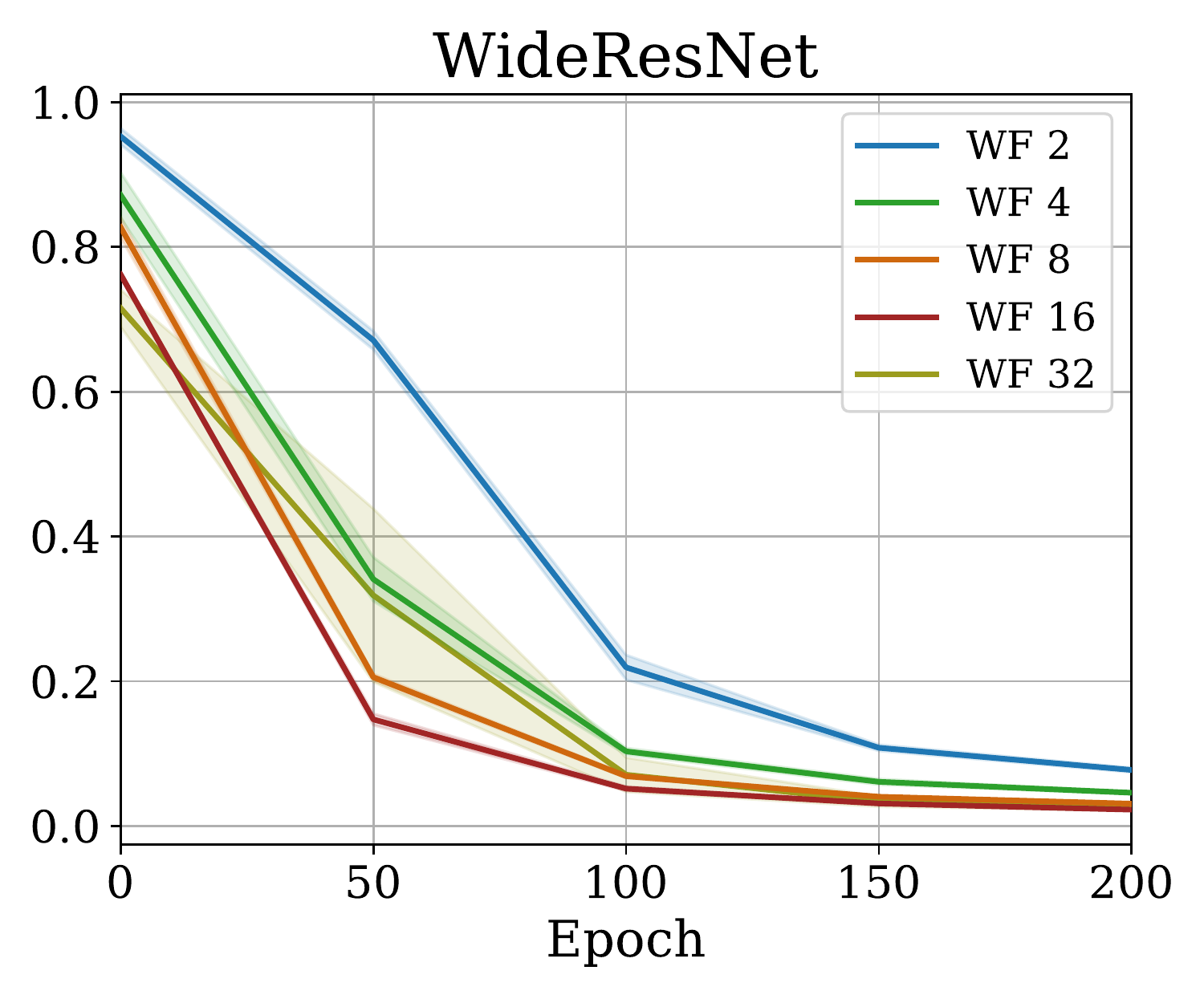}
    \end{subfigure}
    \vspace{-2mm}
    \caption{The \textbf{relative difference of kernel regression outputs}, \eqref{eq:f_lin} and \eqref{eq:pseudo_lin}, when training on $\abs{\dset} = 1000$ random CIFAR-10 points and testing on $\abs{\cX} = 500$. For wider NNs, the relative difference in $\hat{f}^\mathit{lin}(\cX)$ and $f^\mathit{lin}(\cX)$ decreases at initialization. Surprisingly, the difference between these two continues to quickly vanish while training the network.}
    \label{fig:kr_norm_diff_pseudo_full}
    \vspace{-2mm}
\end{figure*}

\begin{figure*}[!ht]
    \centering
    \begin{subfigure}[b]{0.24\textwidth}
        \includegraphics[width=\textwidth]{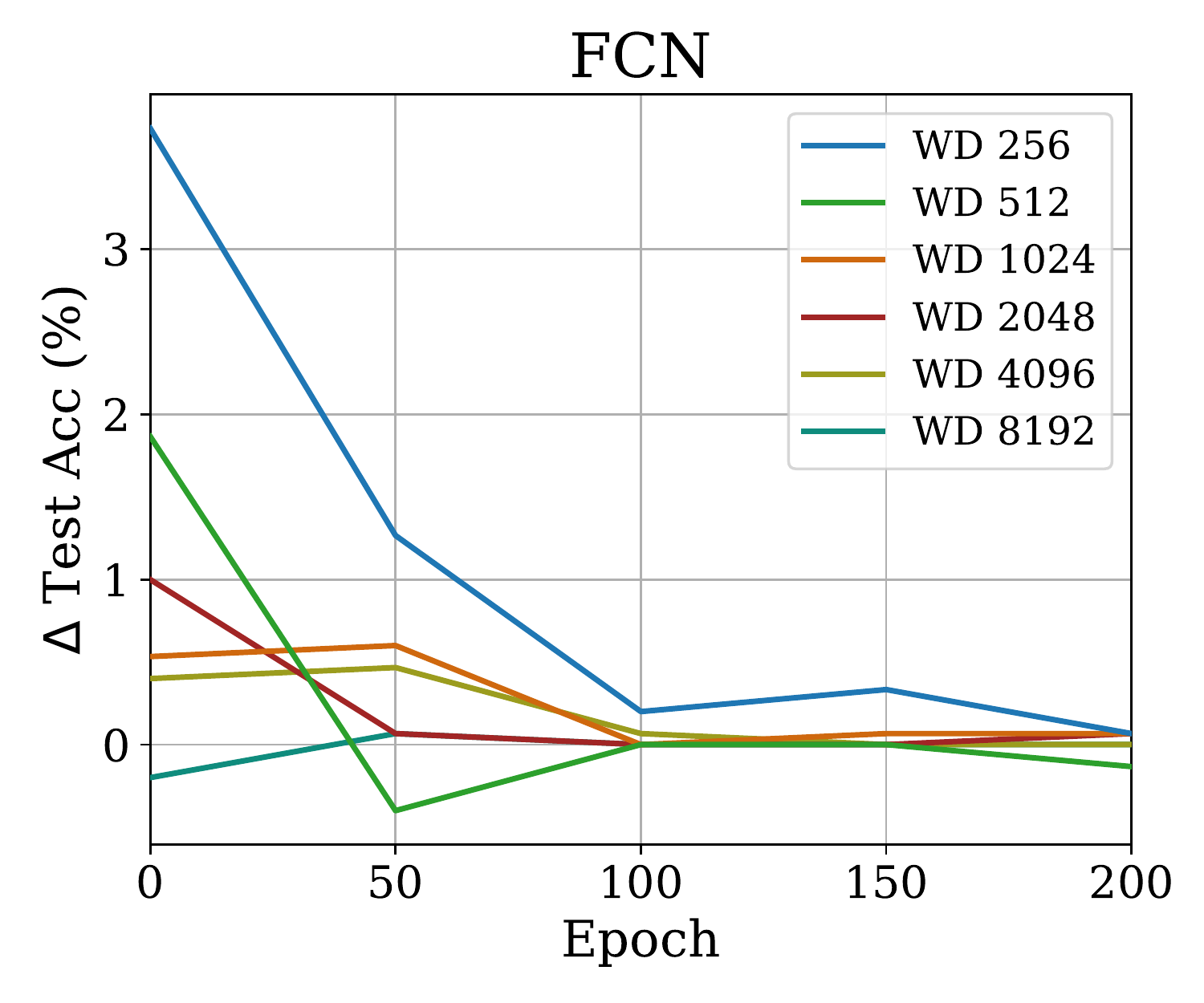}
    \end{subfigure}
    \hfill
    \begin{subfigure}[b]{0.24\textwidth}
        \includegraphics[width=\textwidth]{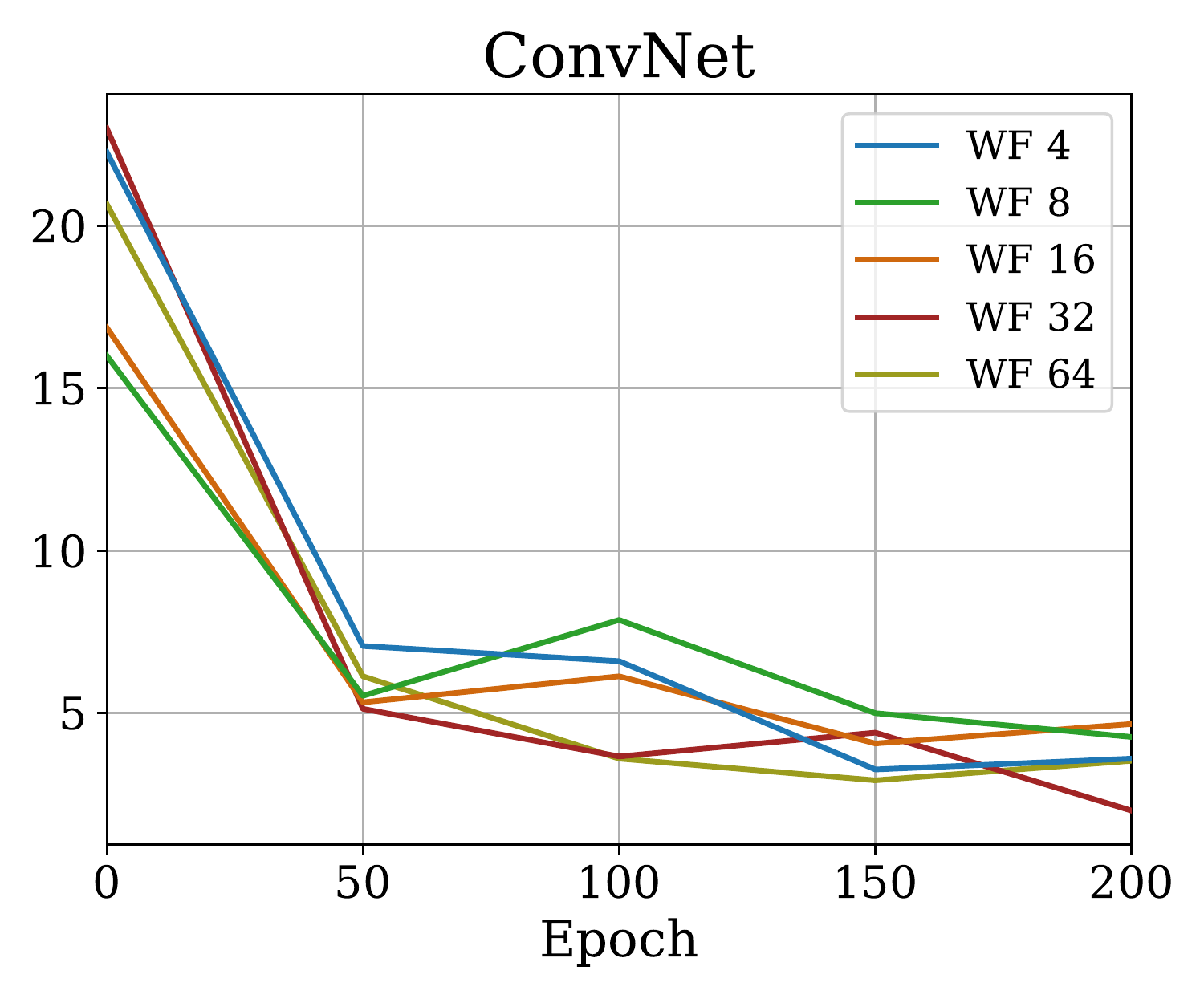}
    \end{subfigure}
    \hfill
    \begin{subfigure}[b]{0.24\textwidth}
        \includegraphics[width=\textwidth]{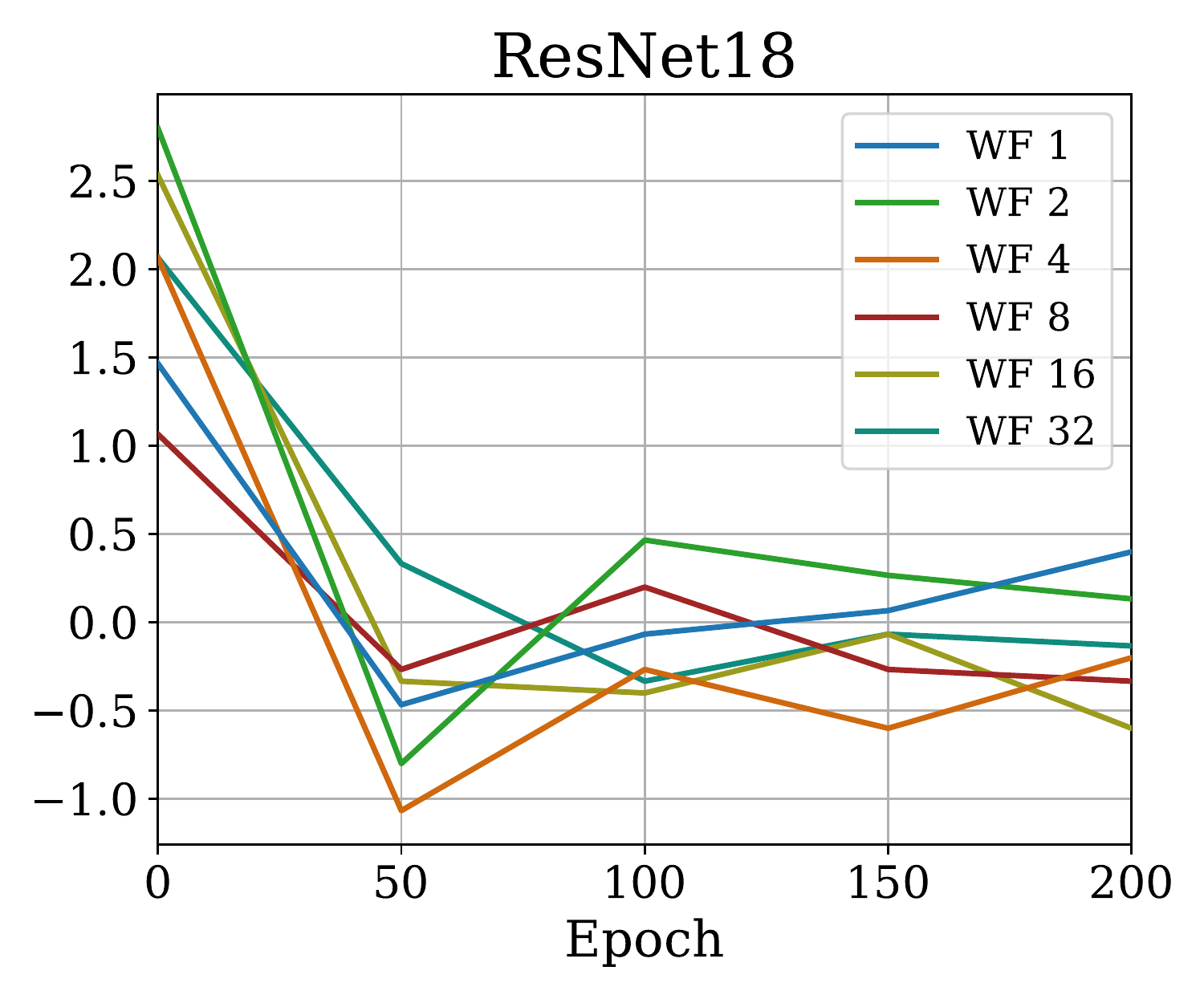}
    \end{subfigure}
    \hfill
    \begin{subfigure}[b]{0.24\textwidth}
        \includegraphics[width=\textwidth]{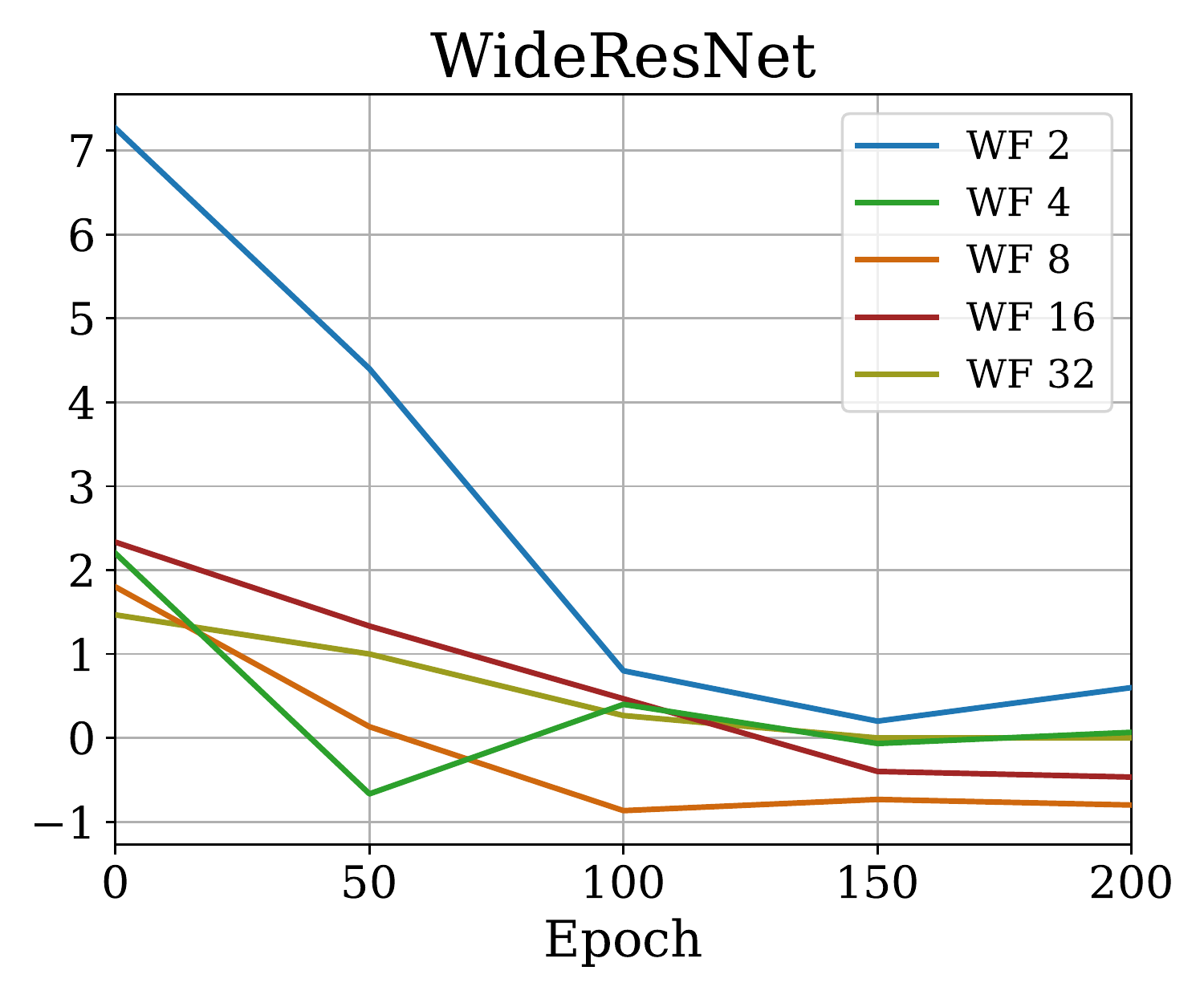}
    \end{subfigure}
    \vspace{-2mm}
    \caption{
    \textbf{Using \pNTK~in kernel regression} (as in \cref{fig:kr_norm_diff_pseudo_full}) \textbf{almost always achieves a higher test accuracy than using \eNTK.}
    Wider NNs and trained nets have more similar prediction accuracies of $\hat{f}^\mathit{lin}$ and $f^\mathit{lin}$ at initialization.
    Again, the difference between these two continues to vanish throughout the training process using SGD.
    }
    \label{fig:kr_acc_diff_pseudo_full}
    \vspace{-3mm}
\end{figure*}

\citet{linntk2019lee} proved that as a finite NN's width grows, its training dynamics can be well approximated using the first-order Taylor expansion of that NN around its initialization (a \textit{linearized neural network}). Informally, they showed that when $f$ is wide enough and trained on $\dset$ with a suitably small learning rate,
its predictions on $x$ can be approximated by those of the linearized network $f^\mathit{lin}$ given by
\begin{align} \label{eq:f_lin}
    \underbrace{f_0(x)}_{O \times 1} + \underbrace{\Theta_0(x, \dset)}_{O \times NO}
    \; \underbrace{{\Theta_0(\dset, \dset)}^{-1}}_{NO \times NO} 
    \, \underbrace{(\mathcal{Y}_\dset - f_0(\dset))}_{NO \times 1}
,\end{align}
where $\cY_{\dset}$ is the matrix of one-hot labels for the training points $\dset$, and $\Theta_0$ is the \eNTK~of $f$ at initialization $f_0$.
This is simply kernel regression on the training data $\dset$ using the kernel $\Theta_0$ and prior mean $f_0$.
\citet{wei2022more} use the same kernel in a generalized cross-validation estimator \citep{craven1978smoothing} to predict the generalization risk of the NN.
As discussed before, using the \eNTK{} in these applications is practically infeasible, due to huge time and memory complexity of the kernel,
but we show the \pNTK~approximates $f^\mathit{lin}(x)$ with much improved time and memory complexity.

\begin{theorem}[Informal] \label{theorem:kr_norm_diff}
Let $f_\theta: \R^{D} \to \R^{O}$ be
a fully-connected network with layers of width $n$ whose parameters are initialized as in \citet{he2016kaiming},
with ReLU-type activations.
Let $\mpNTK(x_1, x_2)$ be the corresponding pNTK of $f_\theta$ as in \eqref{eq:pntk_def} and $\meNTK(x_1, x_2)$ the corresponding eNTK as in \eqref{eq:jac_ntk}
for a fixed pair of inputs $x_1$, $x_2$.
Define $\hat f^\mathit{lin}(x)$ as
\begin{equation} \label{eq:pseudo_lin}
    \underbrace{f_0(x)}_{1 \times O} + \underbrace{\mpNTK(x, \dset)}_{1 \times N}
    \, \underbrace{{\mpNTK(\dset, \dset)}^{-1}}_{N \times N} 
    \underbrace{(\cY_\dset - f_0(\dset))}_{N \times O}
.\end{equation}
After proper reshaping, with high probability over random initialization,
\begin{equation}
    \norm{\hat{f}^\mathit{lin}(x) - f^\mathit{lin}(x)}_F \in \bigO(n^{-\frac{1}{2}}).
\end{equation}
\end{theorem}

\begin{figure*}[!ht]
    \centering
    \begin{subfigure}[b]{0.24\textwidth}
        \includegraphics[width=\textwidth]{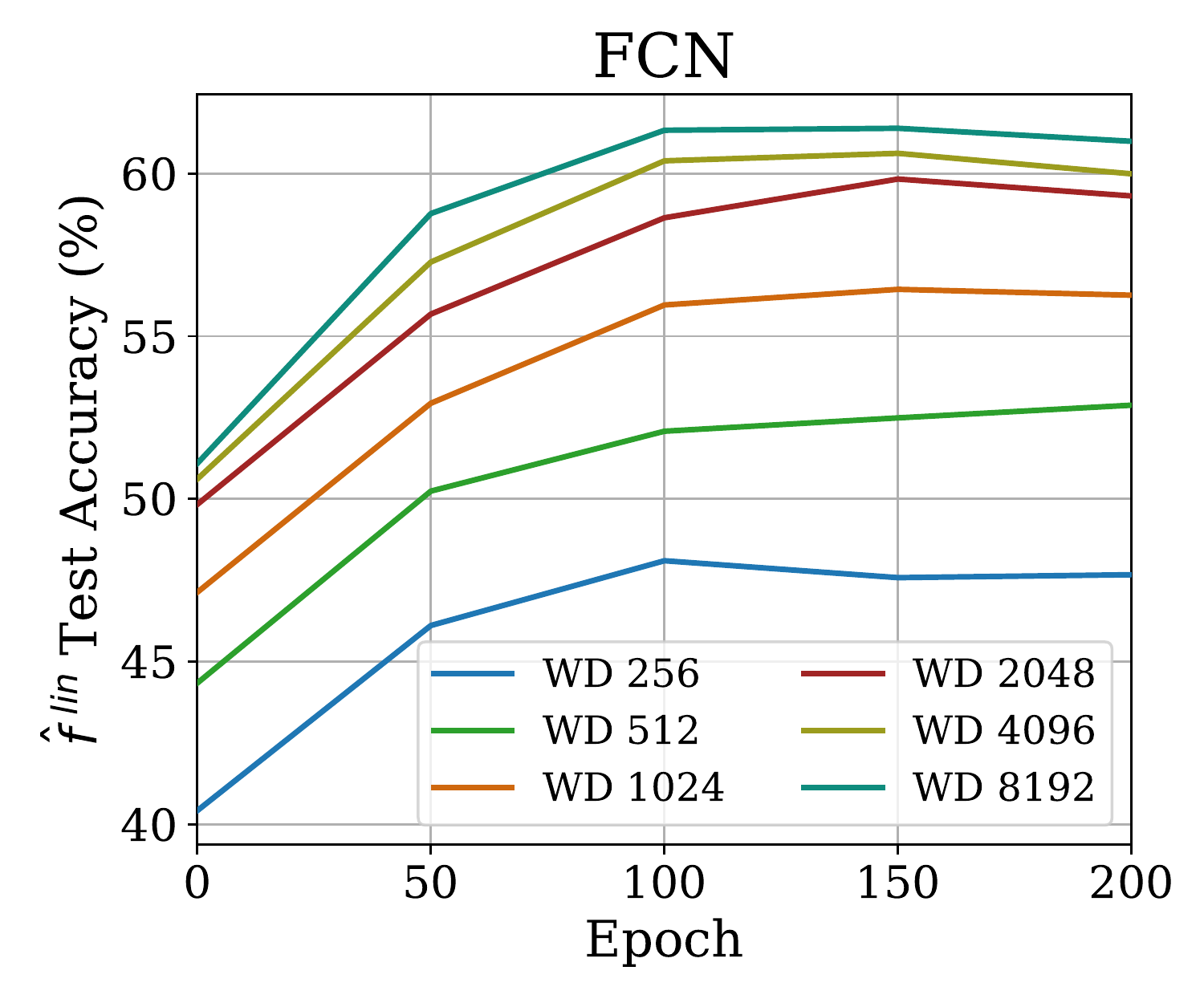}
    \end{subfigure}
    \hfill
    \begin{subfigure}[b]{0.24\textwidth}
        \includegraphics[width=\textwidth]{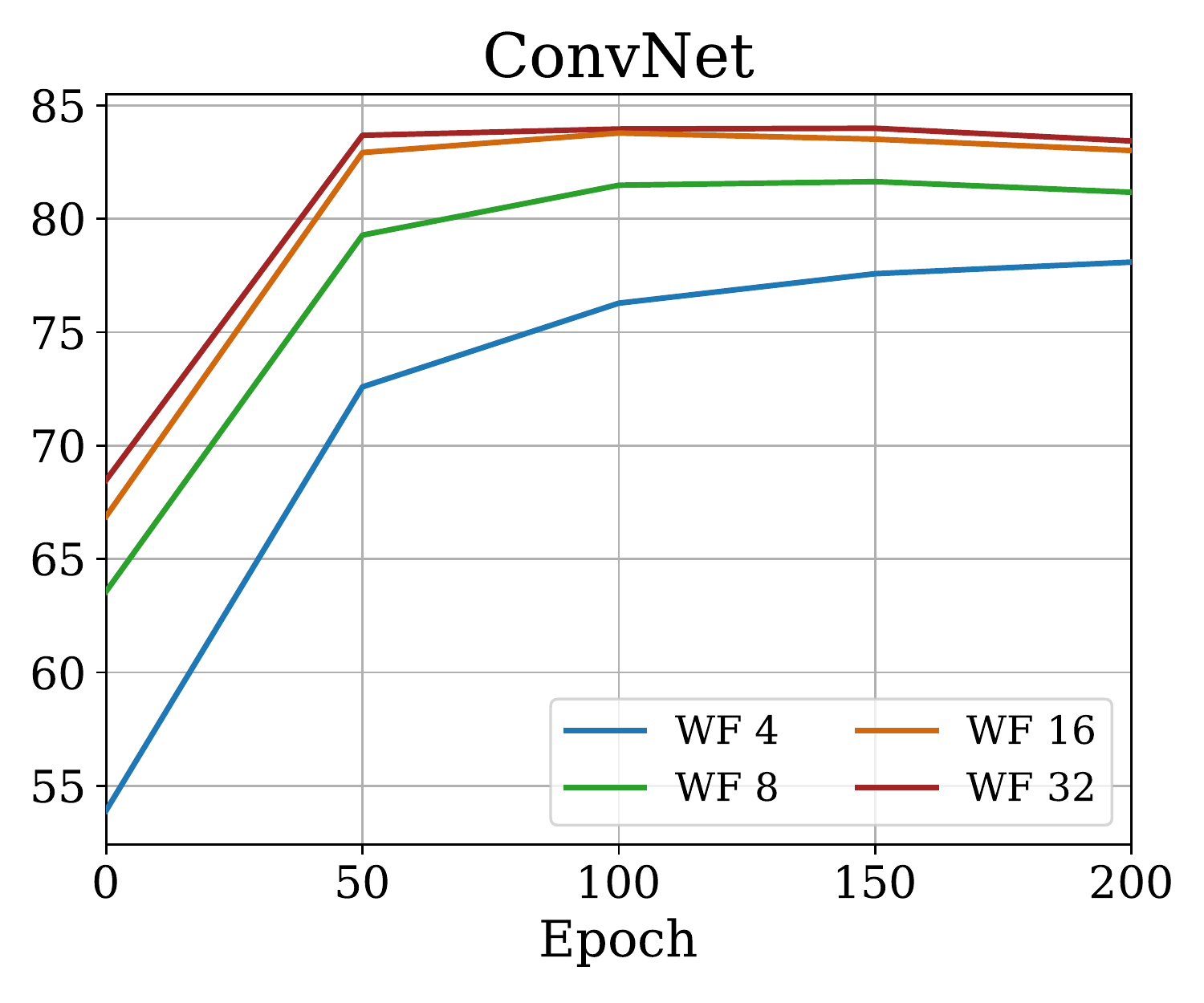}
    \end{subfigure}
    \hfill
    \begin{subfigure}[b]{0.24\textwidth}
    \includegraphics[width=\textwidth]{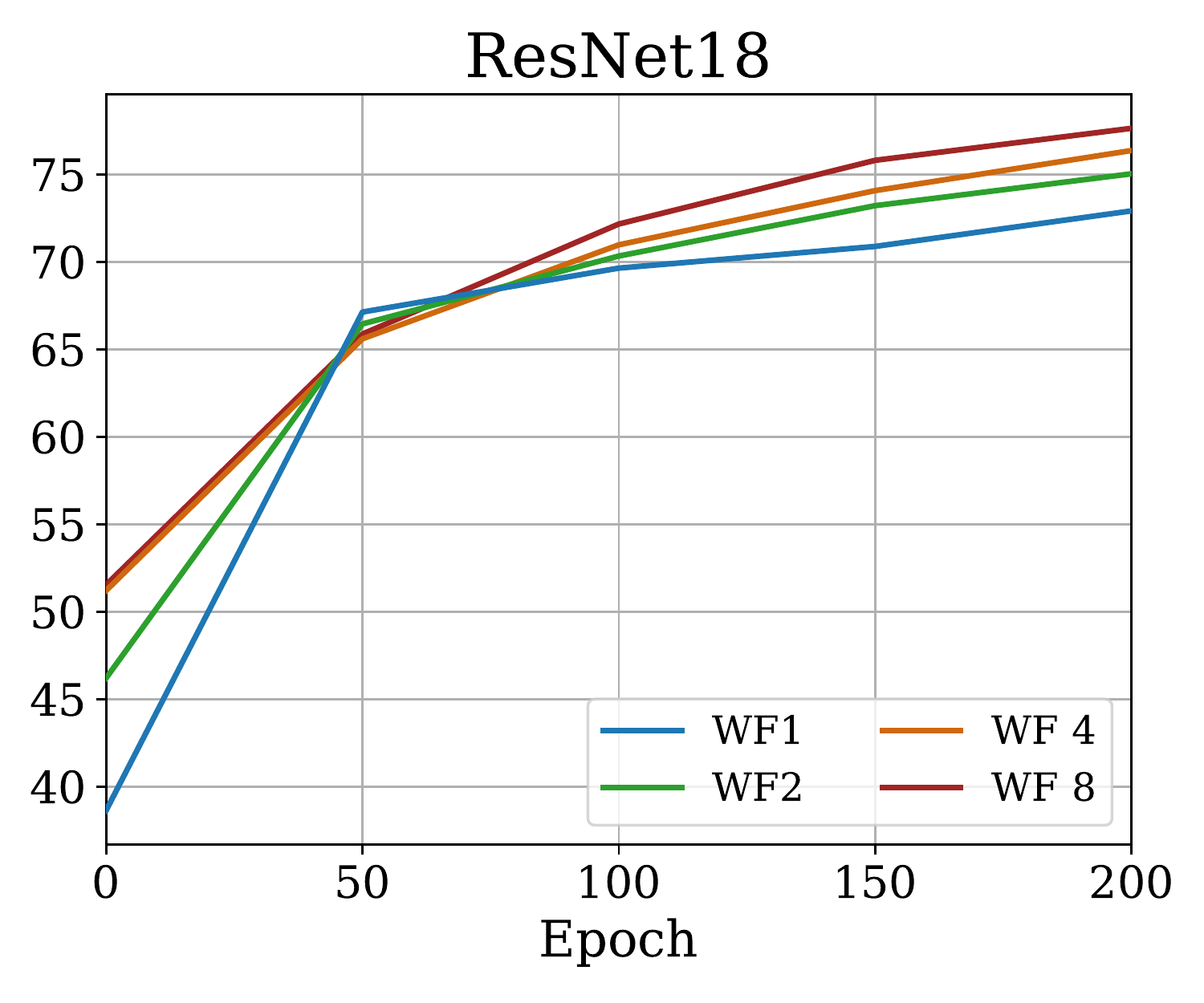}
    \end{subfigure}
    \hfill
    \begin{subfigure}[b]{0.24\textwidth}
        \includegraphics[width=\textwidth]{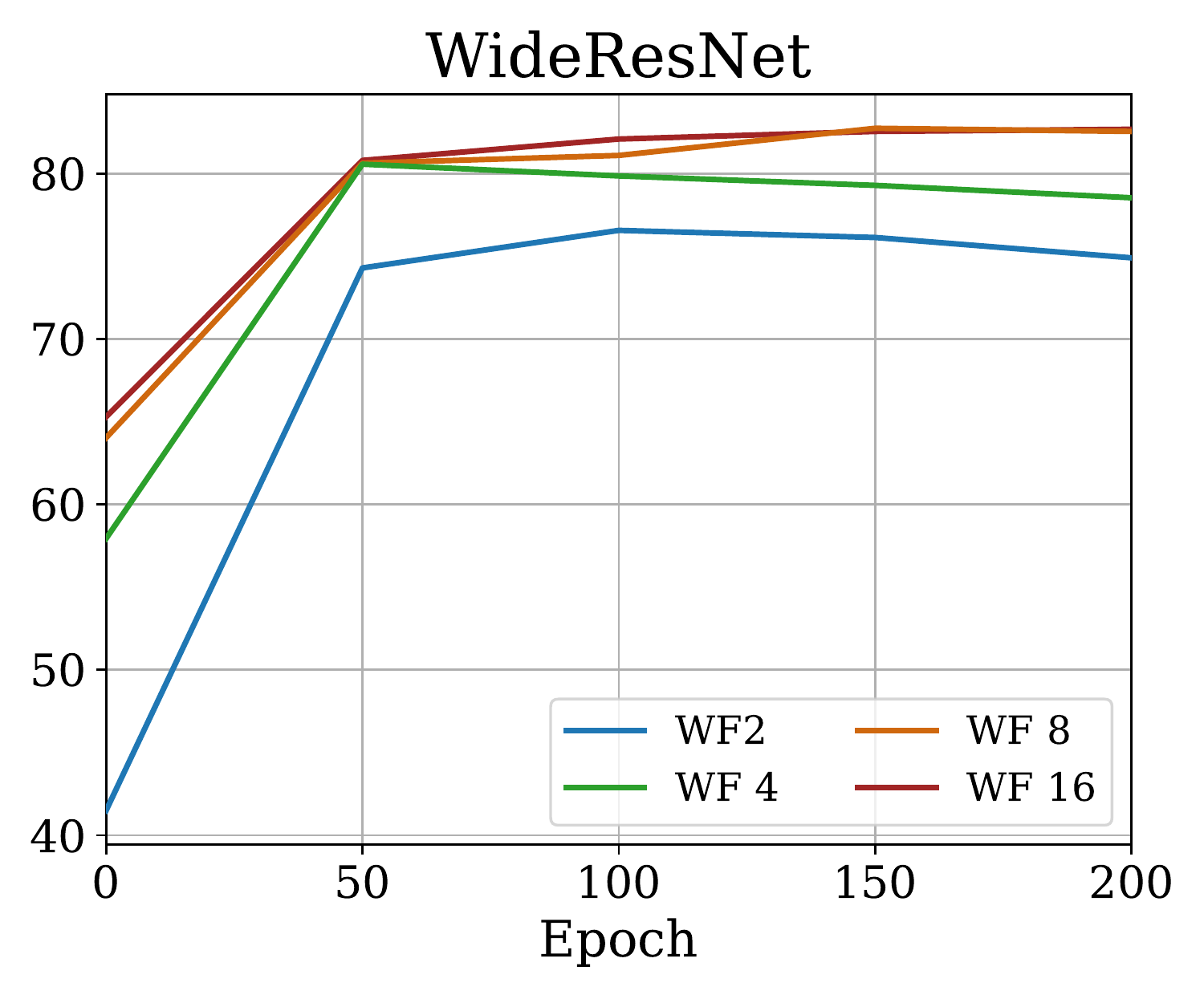}
    \end{subfigure}
    \vspace{-2mm}
    \caption{Evaluating the \textbf{test accuracy of kernel regression predictions using \pNTK~as in \eqref{eq:pseudo_lin} on the full CIFAR-10 dataset}.
    As the NN's width grows, the test accuracy of $\hat{f}^\mathit{lin}$ also improves, but eventually saturates with the growing width. Using trained weights in computation of \pNTK~results in improved test accuracy of $\hat{f}^\mathit{lin}$.}
    \label{fig:fhat_full_CIFAR-10_acc}
    \vspace{-2mm}
\end{figure*}

\begin{figure*}[!ht]
    \centering
    \begin{subfigure}[b]{0.24\textwidth}
        \includegraphics[width=\textwidth]{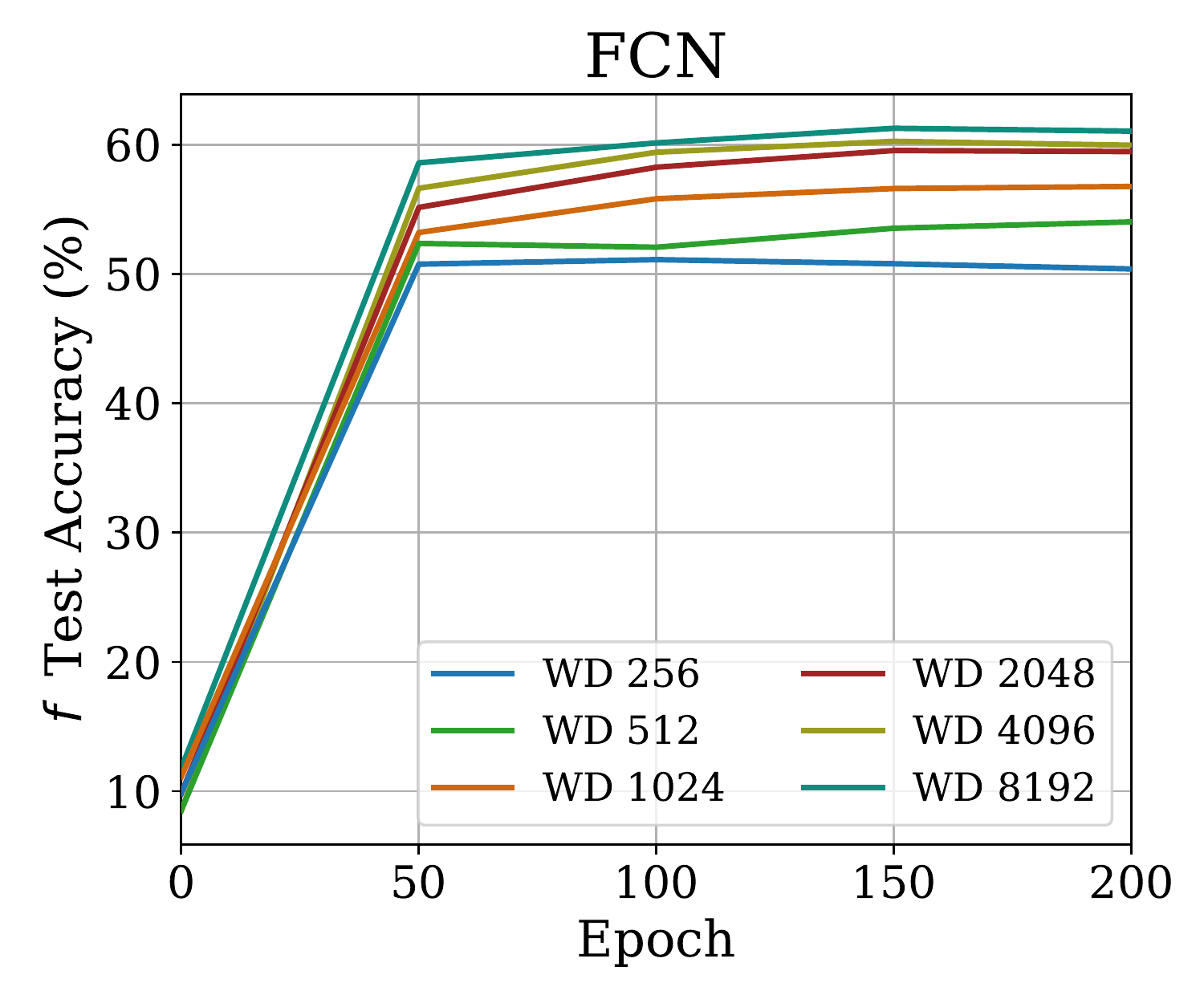}
    \end{subfigure}
    \hfill
    \begin{subfigure}[b]{0.24\textwidth}
        \includegraphics[width=\textwidth]{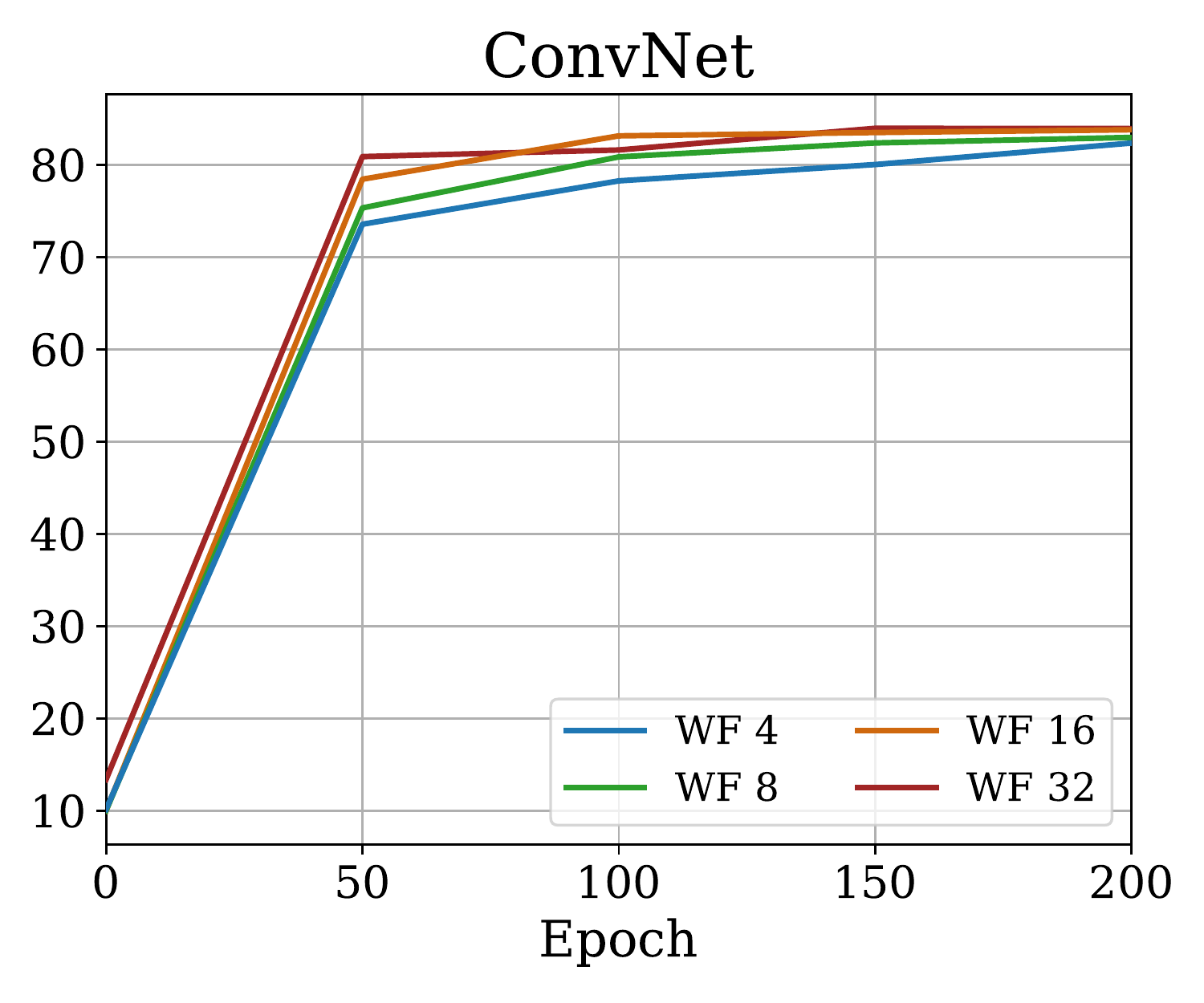}
    \end{subfigure}
    \hfill
    \begin{subfigure}[b]{0.24\textwidth}
        \includegraphics[width=\textwidth]{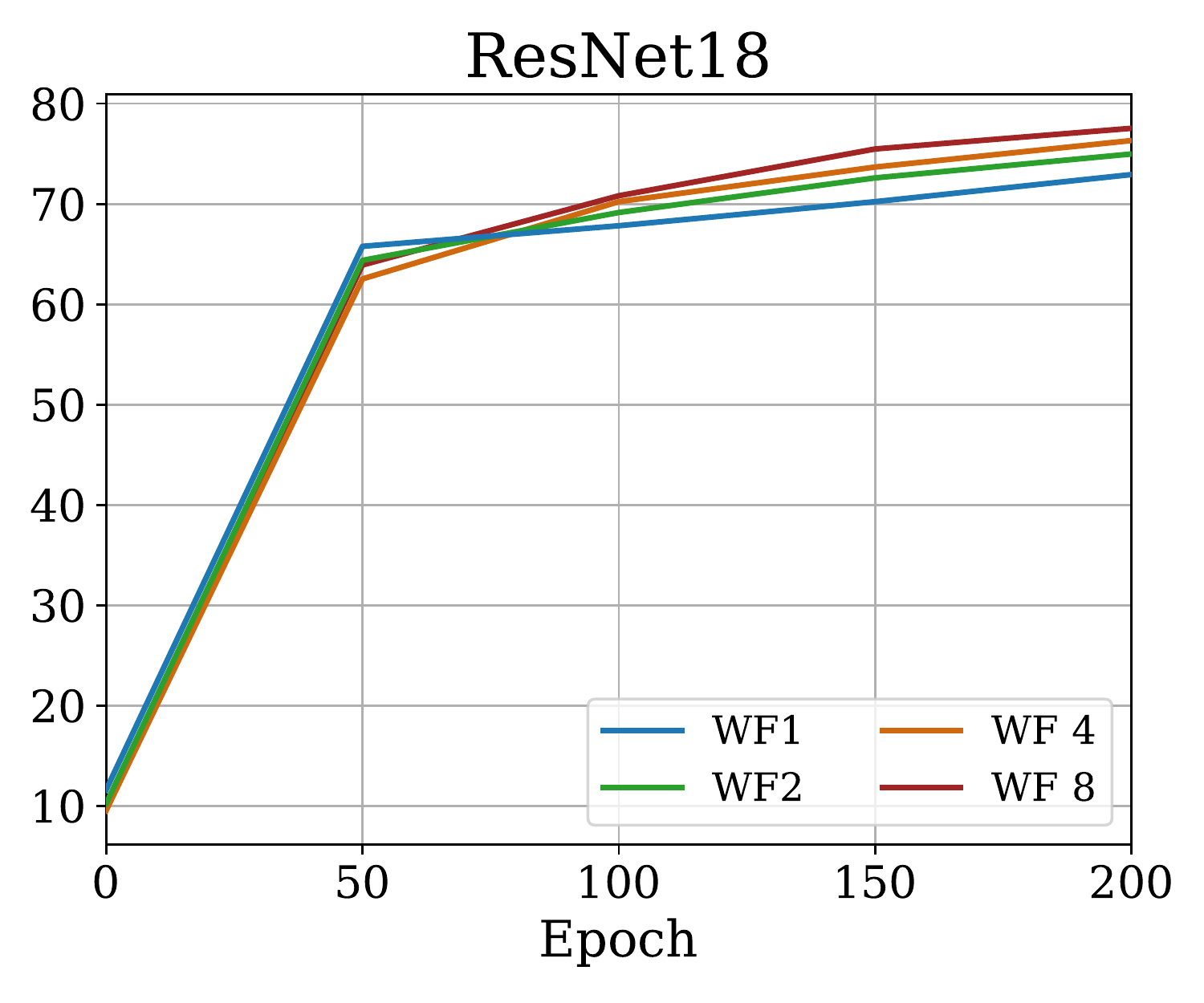}
    \end{subfigure}
    \hfill
    \begin{subfigure}[b]{0.24\textwidth}
        \includegraphics[width=\textwidth]{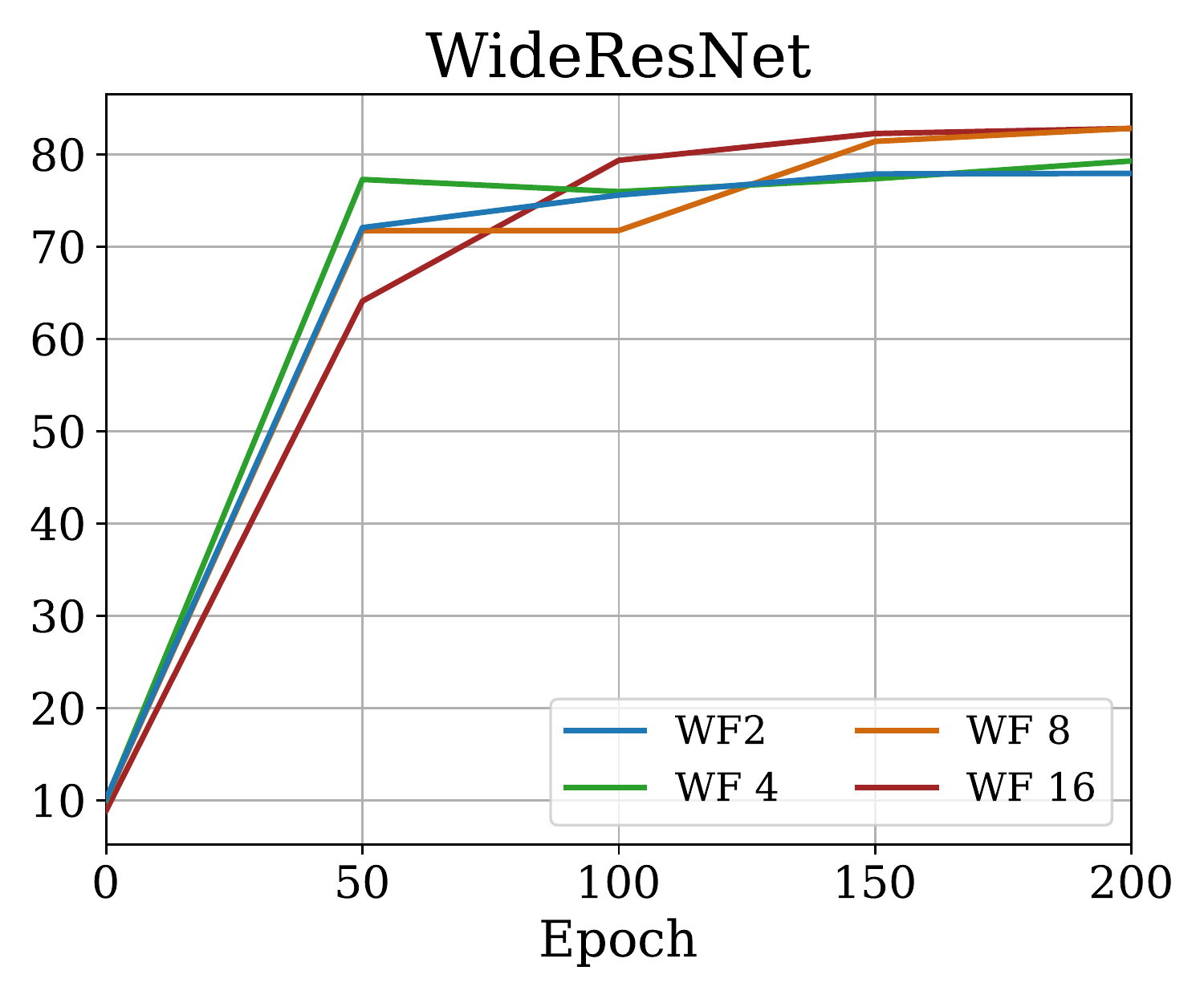}
    \end{subfigure}
    \vspace{-2mm}
    \caption{Evaluating the \textbf{test accuracy of model $f$ throughout SGD training on the full CIFAR-10 dataset}. In contrast to $\hat{f}^\mathit{lin}$, the test accuracy of $f$ does not significantly improve with growing width.}
    \label{fig:f_full_CIFAR-10_acc}
    \vspace{-2mm}
\end{figure*}

\begin{figure*}[!ht]
    \centering
    \begin{subfigure}[b]{0.24\textwidth}
        \includegraphics[width=\textwidth]{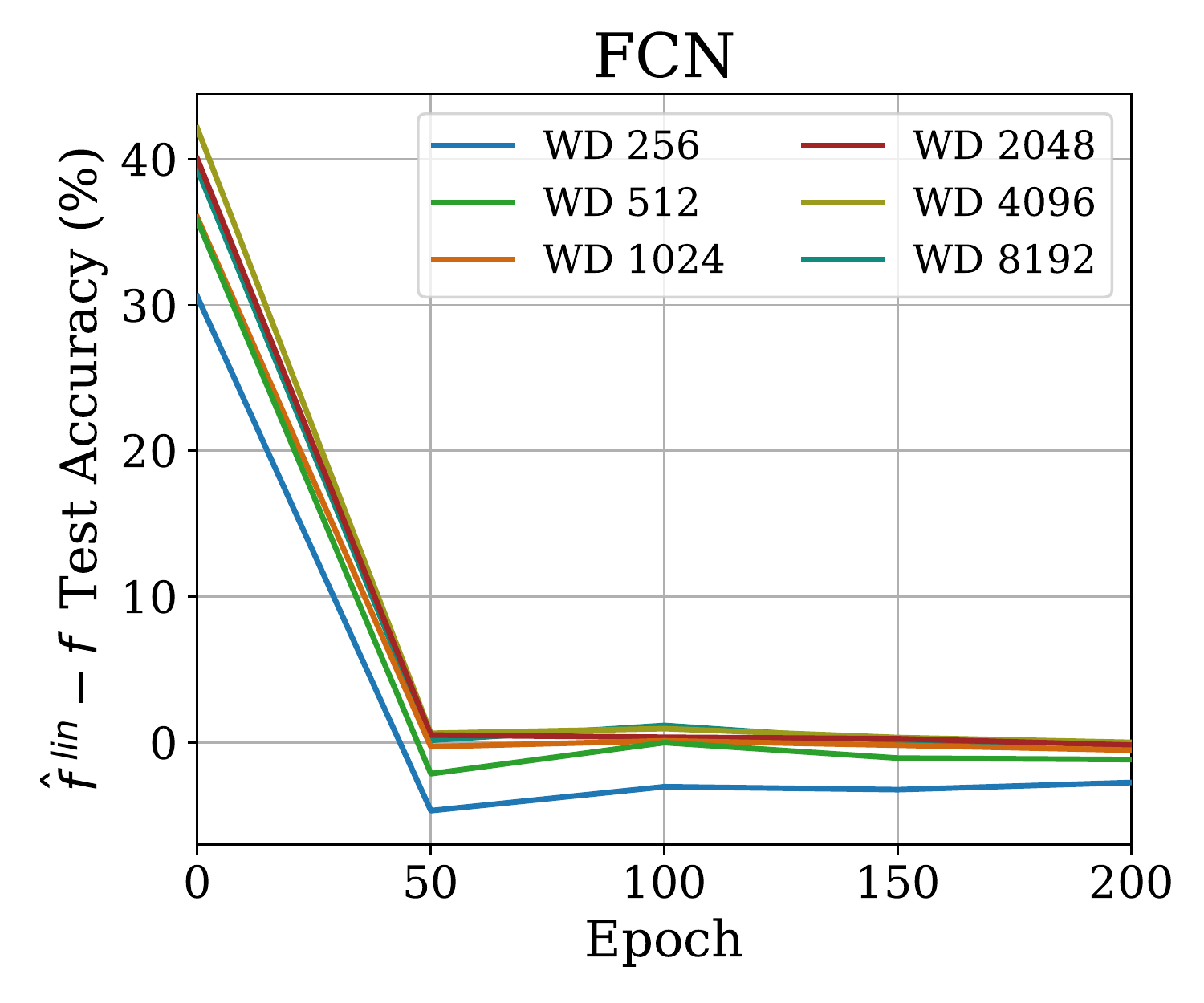}
    \end{subfigure}
    \hfill
    \begin{subfigure}[b]{0.24\textwidth}
        \includegraphics[width=\textwidth]{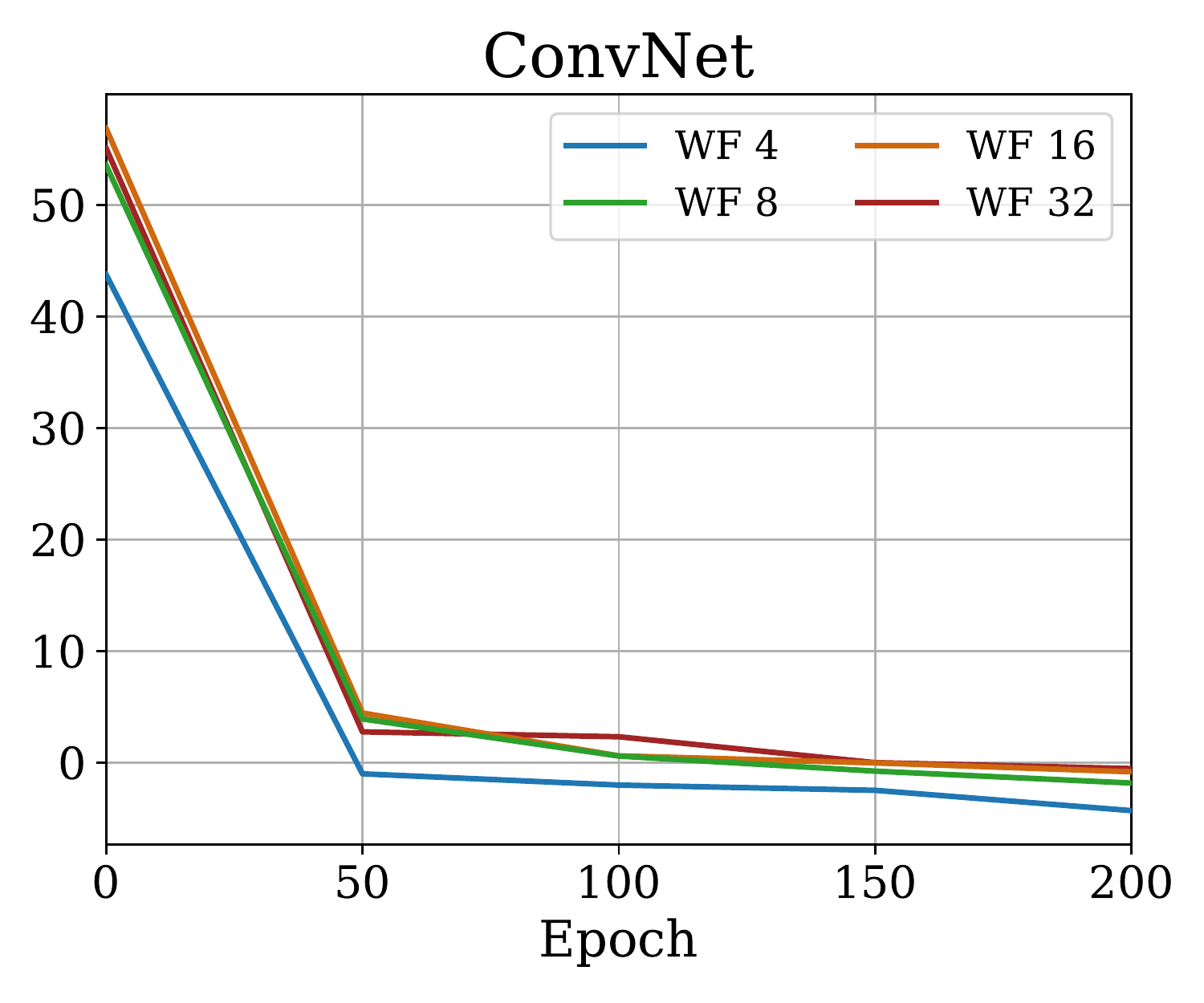}
    \end{subfigure}
    \hfill
    \begin{subfigure}[b]{0.24\textwidth}
        \includegraphics[width=\textwidth]{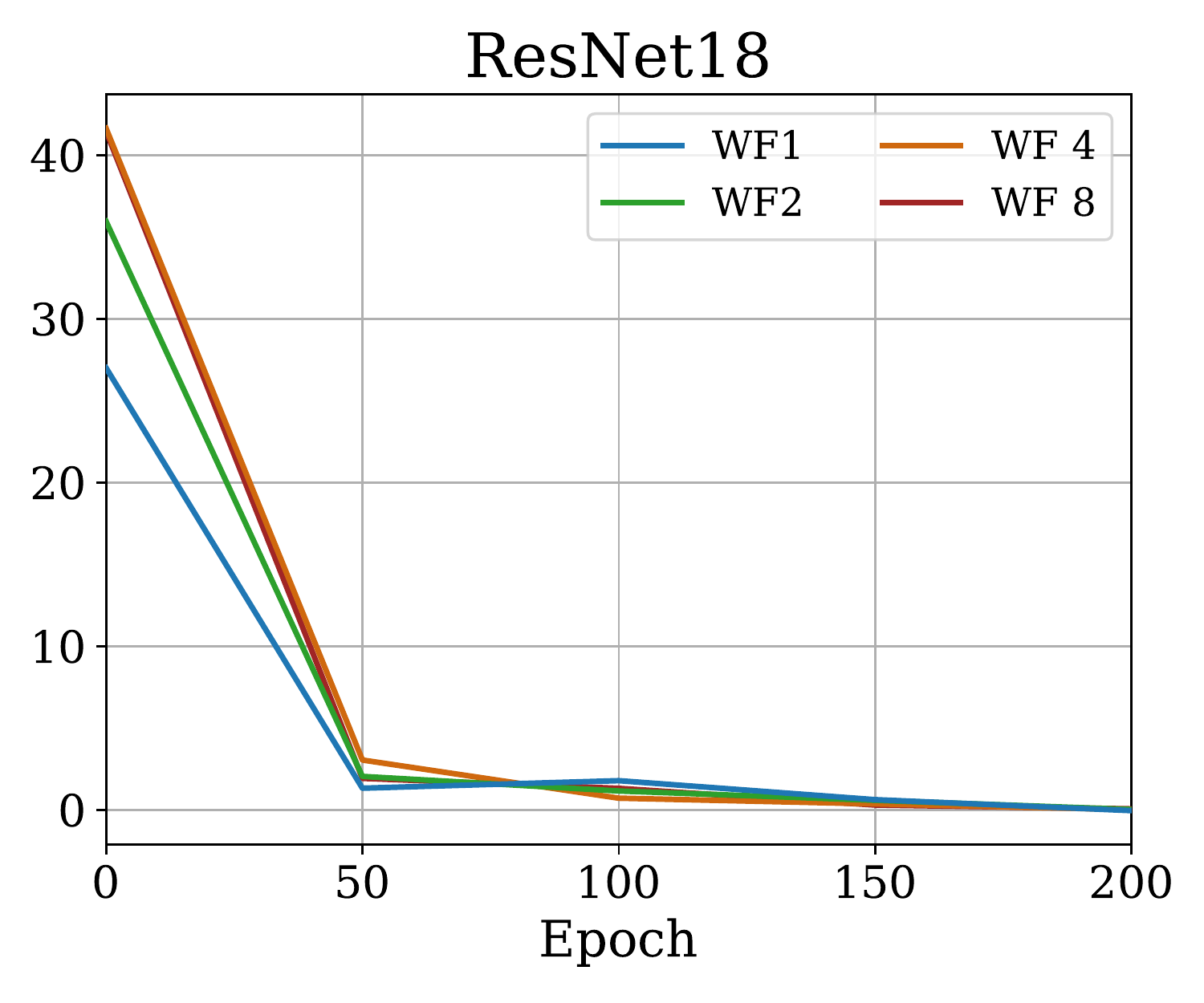}
    \end{subfigure}
    \hfill
    \begin{subfigure}[b]{0.24\textwidth}
        \includegraphics[width=\textwidth]{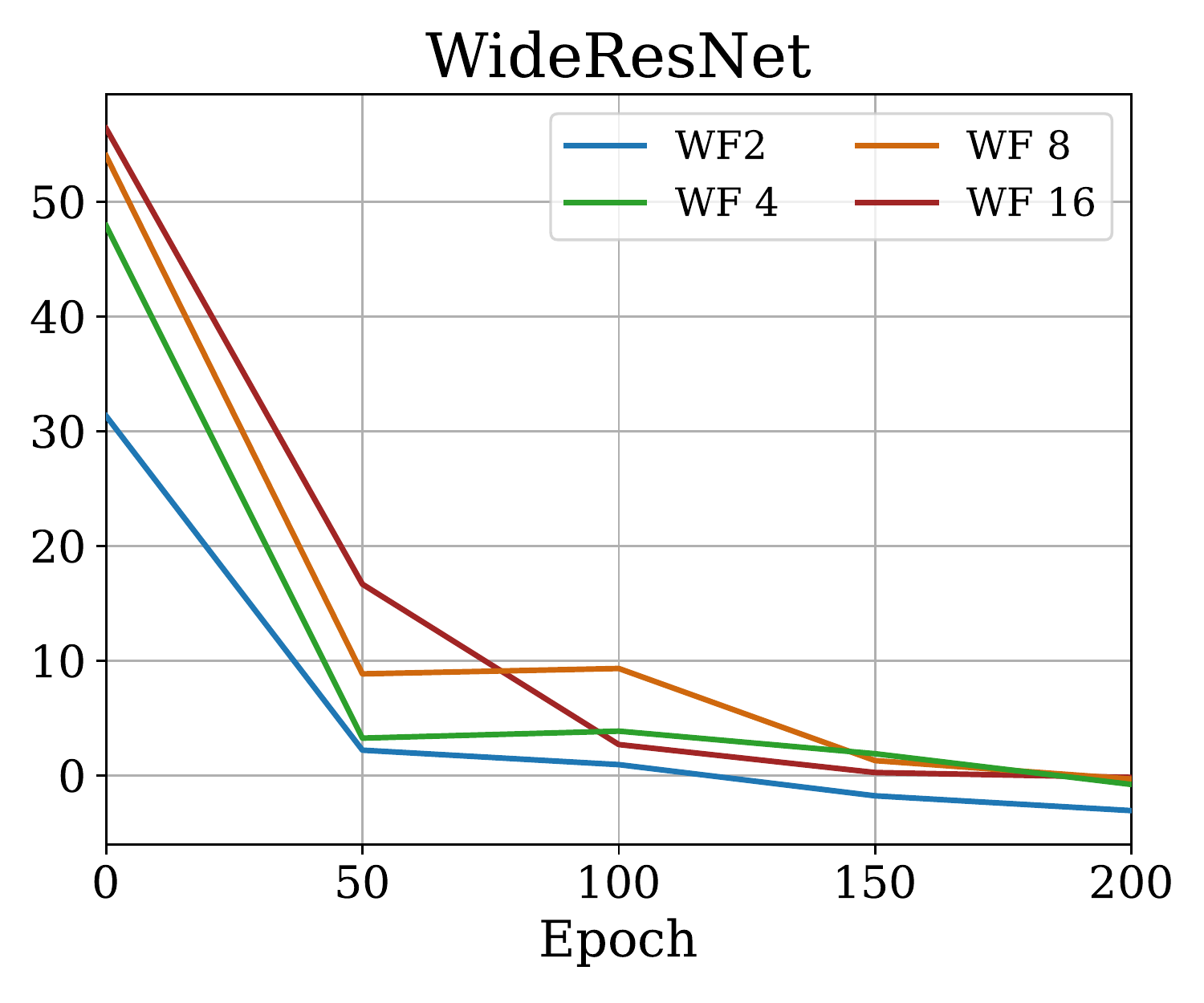}
    \end{subfigure}
    \vspace{-1mm}
    \caption{Evaluating the \textbf{difference in test accuracy of kernel regression using \pNTK~as in \eqref{eq:pseudo_lin} vs the current model $f$} throughout SGD training on the full CIFAR-10 dataset: how much does a linearized predictor with the current representation improve prediction accuracy over the current model obtained by SGD?}
    \label{fig:f_vs_fhat_full_CIFAR-10_acc}
    \vspace{-2mm}
\end{figure*}

A formal statement is given and proved in \cref{app:kernel-regression}.

\cref{fig:kr_norm_diff_pseudo_full} also supports that as width grows, the predictions of kernel regression using $\mpNTK$ converge to the prediction of those obtained from using $\meNTK$, \textit{while requiring orders of magnitude of less memory and time to compute}.
\cref{fig:kr_acc_diff_pseudo_full} shows similar results for the difference in prediction accuracies achieved using kernel regression through $\mpNTK$ and $\meNTK$ kernels.
\cref{app:kernel-regression} also shows further analysis of how well the \textit{linearized network} predicts the final accuracy of the trained model for each architecture and width pair.
Although $\norm{\hat{f}^\mathit{lin}(x) - f^\mathit{lin}(x)}_F$ decreases with width of the network in \cref{fig:kr_norm_diff_pseudo_full} at initialization, this does not \emph{necessarily} translate to a monotonic behaviour in prediction accuracies,
a non-smooth function of the vector of predictions;
we do see that the expected pattern more or less holds, however.

A surprising outcome depicted in \cref{fig:kr_norm_diff_pseudo_full,fig:kr_acc_diff_pseudo_full} is that while training the model's parameters, predictions of $\hat{f}^\mathit{lin}$ and $f^\mathit{lin}$ converge very quickly.
This is particularly intriguing, as it contrasts the results depicted in \cref{fig:diagonality,fig:fro_norm_diff,fig:max_eigval_diff,fig:cond_number_diff}. In other words, although the kernels $\meNTK$ and $\mpNTK \otimes I_O$ seem to be diverging in Frobenius norm, eigenspectrum, and so on, kernel regression using those two kernels converges quickly, so that after 50 epochs the difference in predictions almost totally vanishes. We believe further investigation of why this phenomenon is observed could lead to new interesting insights about the training dynamics of NNs.

\section{Application: Full Regression on CIFAR-10}
\label{sec:full-cifar}

Thanks to the reduction in time and memory complexity of \pNTK{} over \eNTK{}, motivated by \cref{theorem:kr_norm_diff} and experimental findings in \cref{fig:kr_acc_diff_pseudo_full}, we finally evaluate the corresponding \pNTK{}s of the four architectures that we have used in our experimental evaluations in different widths using full CIFAR-10 data, at initialization and throughout training the models under SGD.
As mentioned previously, running kernel regression with \eNTK{} on all of CIFAR-10 would require evaluating $25\times10^{10}$ Jacobian-vector products and more than $\approx 1.8$ terabytes of memory; using \pNTK, this can be done with a far more reasonable $25 \times 10^8$ Jacobian-vector products and $\approx 18$ gigabytes of memory.
This is still a heavy load compared to, say, direct SGD training,
but is within the reach of standard compute nodes.

\cref{fig:fhat_full_CIFAR-10_acc} shows the test accuracy of $\hat{f}^\mathit{lin}$ on the full train and test sets of CIFAR-10. 
In the infinite-width limit, the test accuracy of $\hat{f}^\mathit{lin}$ at initialization (and later, because the kernel stays constant in this regime) should match the final test accuracy of $f$: that is, the epoch 0 points in \cref{fig:fhat_full_CIFAR-10_acc} would agree with roughly the epoch 200 points in \cref{fig:f_full_CIFAR-10_acc}.
This comparison is plotted directly in \cref{app:full-cifar}.
Furthermore, the test accuracies of predictions of kernel regression using the \pNTK{} are lower than those achieved by the NTKs of infinite-width counterparts for fully-connected and fully-convolution networks.
This is consistent with results on eNTK by \citet{cntk2019arora} and \citet{lee2020finite}, 
although \citet{cntk2019arora} studied only a ``CIFAR-2'' version.

It is worth noting from \cref{fig:fhat_full_CIFAR-10_acc,fig:f_vs_fhat_full_CIFAR-10_acc} that, in contrast to the findings of \citet{fort2020deep}, we observe that the corresponding \pNTK~of the NN continues to change even after epoch 50 of SGD training.
Although for fully-connected networks and some versions of ResNet18, this change is not significant, in fully-convolutional networks and WideResNets, the \pNTK~continues to exhibit changes until epoch 150, where the training error has vanished.
We remark that \citet{fort2020deep} analyzed \eNTK{}s based on only 500 random samples from CIFAR-10, while the \pNTK{} approximation has enabled us to run our analysis on the $100$-times larger full dataset.

Lastly, to help the community better analyze the properties of NNs and their training dynamics, and avoid wasting computation by redoing this work, we plan to share computed \pNTK s for all the mentioned architectures and widths, as well as ResNets with 34, 50, 101 and 152 layers on CIFAR-10 and CIFAR-100 \citep{xiao2017fashion} datasets, both at initialization and using pretrained ImageNet \citep{deng2009imagenet} weights.
We hope that our contribution will enable further analyses and breakthroughs towards a stronger theoretical understanding of the training dynamics of deep NNs.

\begin{figure}[t]
    \centering
    \includegraphics[width=0.45\textwidth]{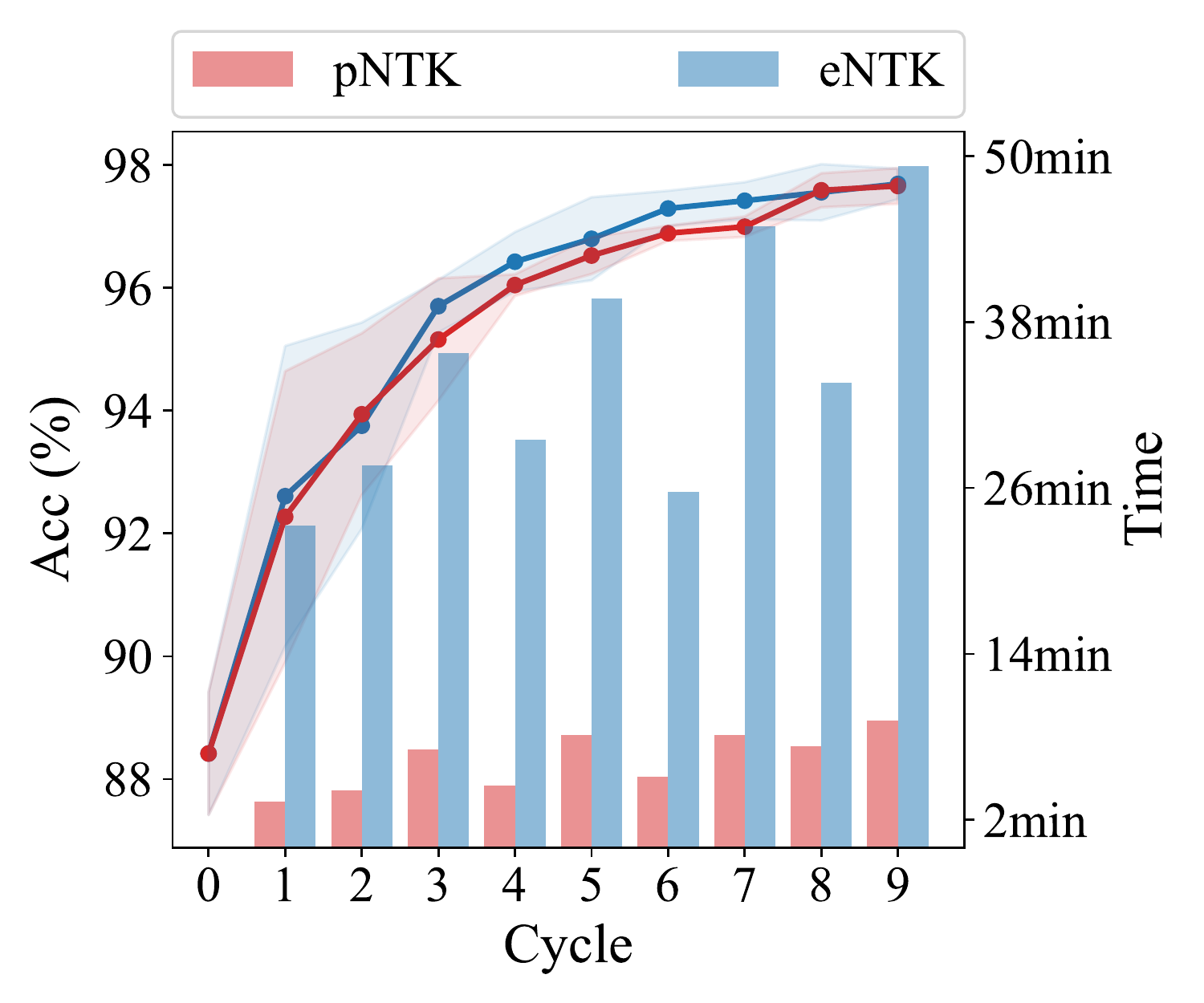}
    \vspace{-2mm}
    \caption{\textbf{Comparison of \pNTK{} with \eNTK{} on a look-ahead active learning task.} \pNTK{} is much faster than \eNTK{} without losing performance.}
    \label{fig:application_entk_pntk}
    \vspace{-5mm}
\end{figure}

\section{Application: Active Learning}
\label{sec:application}
One use case for \eNTK{} is in active learning, where the model requests annotations for specific data points in an ``active'' manner.
Recently, \citet{mohamadi:active-ntk} used kernel regression with the \pNTK{} to approximate a re-trained model for the computation of ``look-ahead'' data acquisition functions in active learning.
That is, a model requests annotations for data points where training on them is likely to make the model's output on test data change the most.
In \cref{fig:application_entk_pntk},
we compare the performance of their scheme using \pNTK{} (as they did in their work) to using the full \eNTK{},
in terms of both active learning performance on the MNIST dataset (left axis),
and wall-clock time needed to compute the acquisition function (right axis).
Similarly to that work,
we measure accuracy on the test set
for a model which begins with $100$ randomly trained points,
then acquires $20$ additional labelled points in each cycle.
The acquisition performance with the \pNTK{} matches that with the \eNTK{},
but computation is much faster,
taking about 15\% as long in total
on this problem with $O = 10$.
Active learning performance is measured in accuracy on the test set using a model trained on a labelled set.

\section{Discussion}
Our \pNTK{} approach to approximating the \eNTK{} has provable bounds, good empirical performance, and multiple orders of magnitude improvement in runtime speed and memory requirements over the direct \eNTK{}.
We evaluate our claims and the quality of the approximation under diverse settings,
giving new insights into the behaviour of the \eNTK{} with trained representations.
We help justify the correctness of recent approximation schemes, and hope that our rigorous results and thorough experimentation will help researchers develop a deeper understanding of the training dynamics of finite networks,
and develop new practical applications of the NTK theory.

One major remaining question is to 
theoretically analyze what happens to the \pNTK{} or \eNTK{} during SGD training of the network.
In particular, the fast convergence of $\hat f^\mathit{lin}$ and $f^\mathit{lin}$ when training the network, as seen in \cref{fig:kr_acc_diff_pseudo_full,fig:kr_norm_diff_pseudo_full},
runs counter to our expectations based on the approximation worsening in
Frobenius norm (\cref{fig:kr_norm_diff_pseudo_full}),
maximum eigenvalue (\cref{fig:max_eigval_diff}),
and condition number (\cref{fig:cond_number_diff}).
This seems likely to be important to practical use of the \pNTK.

Perhaps relatedly,
it is also unclear why 
\pNTK~consistently results in higher prediction accuracies than when kernel regression is done using \eNTK,
given that our motivation for \pNTK{} is entirely in terms of approximating the \eNTK{} (\cref{fig:kr_acc_diff_pseudo_full}).
Intuitively, this may be related to a regularization-type effect:
the \pNTK{} corresponds to a particularly limited choice of a ``separable'' operator-valued kernel \citep[see, \eg,][]{mauricio:op-val-review}.
Separable kernels are a common choice in that literature for both computational and regularization reasons;
by enforcing this particularly simple form,
we remove many degrees of freedom relating to the interaction between ``tasks'' (different classes) that may be unnecessary or hard to estimate accurately with the \eNTK.
This might, in some sense, correspond to a one-vs-one rather than one-vs-rest framework in the intuitive sense discussed in \cref{sec:def-pseudo}.
Understanding this question in more detail might require a more detailed understanding of the structure of the \eNTK{} at finite width,
and/or a much more detailed understanding of the interaction between classes in the dataset with learning in the NTK regime.
Finally, even the \pNTK{} is still rather expensive compared to running SGD on neural networks.
It might make for a better starting point than the full \eNTK{} for other speedup methods, however,
like kernel approximation or advanced linear algebra solvers \citep[\eg][]{falkon}.

\bibliography{icml2023_conference}
\bibliographystyle{icml2023}

\clearpage
\appendix
\onecolumn

\section{Details of Experimental Setup} \label{app:expt-details}

In this section, we present the details on the experimental setup used for the plots depicted in the main body of the paper. As mentioned, the exact width for FCNs have been reported. For WideResNet-16-k we use two block layers, and the initial convolution in the network has a width of $16 \WF$ where $\WF$ is the reported $\WF$. For instance, $\WF=16$ means that the first block layer has a width of 256 and the second block layer has a width of 512. For ResNet18, we also used the same approach, multiplying $\WF$ by 16. Thus, when $\WF=4$, the constructed network will have the exact architecture as the classical ResNet18 architecture reported. A $\WF$ of 16 means a ResNet18 with each layer being 4 times wider than the original width.

When training the neural networks using SGD, a constant batch size of 128 was used across all different networks and different dataset sizes used for training. The learning rate for all networks was also fixed to 0.1. However, not all networks were trainable with this fixed learning rate, as the gradients would sometimes blow up and give \texttt{NaN} training loss, typically for the largest width of each mentioned architecture. In those cases, we decreased the learning rate to 0.01 to train the networks.

Note that to be consistent with the literature on NTKs, techniques like data augmentation have been turned off, but a weight decay of 0.0001 along with a momentum of 0.9 for SGD is used. Data augmentation here plays an important role in the attained test accuracies of the fully trained networks.

\section{Relative Convergence of Kernel Matrices} \label{supp:fro_norm_proof}
This section will prove \cref{theorem:pntk_fro_norm},
in two parts.
First, we give a generic analysis
in \cref{sec:linear-readouts} bounding the difference between the $\pNTK$ and $\eNTK$
for any network whose last layer is linear and random,
in terms of the norm of the \eNTK{} of the previous parts of the network.
\Cref{sec:relu-ntk-growth} then bounds the growth of the \eNTK{} for ``fan-in'' ReLU networks;
their combination gives a $\bigO_p(1 / \sqrt n)$ bound on the Frobenius difference between the \eNTK{} and \pNTK{} for width-$n$ fan-in ReLU nets.
Later subsections apply these results to various applications.

\subsection{Linear Read-Out Layers} \label{sec:linear-readouts}

Towards this, we first define some notation and show a simple recursive formula for computing the tangent kernel that we take advantage of to prove the theorems. Consider a NN $f: \R^D \to \R^O$. We assume the final read-out layer of the NN $f$ is a dense layer with width $n$.
Assuming the NN $f$ has $L$ layers, we define $\theta_l$ to be the corresponding parameters of layer $l \in \{1, 2, \dots, L\}$.
Furthermore, let's define $g: \textcolor{black}{\R^D} \to \R^n$ as the output of the penultimate layer of the NN $f$, such that $f(x) = \theta_L g(x)$ for some $\theta_L \in \R^{O \times w}$. 

As noted by \citet{linntk2019lee} and \citet{yang2020tensor}, the NTK can be reformulated as the layer-wise sum of gradients (when the parameters of each layer $\theta_l$ are assumed to be vectorized) of the output with respect to $\theta_l$. Accordingly, we denote the \eNTK~of a NN $f$ as
\begin{equation} \label{eq:ntk_summation}
    \Theta_f(x_1, x_2) = \sum_{l=1}^{L} \nabla_{\theta_l} f(x_1) {\nabla_{\theta_l} f(x_2)}^\top
.\end{equation}

Now, noting that as the final layer of $f$ is a dense layer, we can use the chain rule to write $\nabla_{\theta_l} f(x)$ as $\frac{\partial f}{\partial g(x)} \frac{\partial g(x)}{\partial \theta_l}$ where $\frac{\partial f(x)}{\partial g(x)} = \theta_L$. Thus, we can rewrite \eqref{eq:ntk_summation} as
\begin{align} \label{eq:explained_entk}
    \begin{split}
        \Theta_f(x_1, x_2) &= \sum_{l=1}^{L-1} \theta_L \nabla_{\theta_l} g(x_1) \, {\nabla_{\theta_l} g(x_2)}^\top {\theta_L}^\top + \nabla_{\theta_L} f(x_1) {\nabla_{\theta_L} f(x_2)}^\top \\
        &= \theta_L \left( \sum_{l=1}^L \nabla_{\theta_l} g(x_1) \, \nabla_{\theta_l} g(x_2)^\top \right) {\theta_L}^\top + {g(x_1)}^\top g(x_2) \, I_O \\
        &= \theta_L \, \Theta_g(x_1, x_2) \, {\theta_L}^\top + {g(x_1)}^\top g(x_2) \, I_O.
    \end{split}
\end{align}

Recall that we can view the \pNTK~of a network $f$
as simply adding a fixed, non-trainable dense layer to the network with weights $v$,
where $v$ is either a standard basis vector for the single-logit approximation,
or $\frac{1}{\sqrt{O}} \mathbf{1}_O$ for the sum-of-logits form.
Then \eqref{eq:explained_entk} shows us that the \pNTK{} is simply a weighted sum of the \eNTK's elements,
\begin{equation} \label{eq:supp:pntk_is_sum_of_entk_lastlayer}
    \hat{\Theta}_f(x_1, x_2) = v\tp \Theta_f(x_1, x_2) v
;\end{equation}
note that the second term does not appear
since $v$ is not a trainable parameter.

The key result of this subsection, which holds fairly generally, is based on showing that,
when $\theta_L$ is at initialization,
off-diagonal entries of $\theta_L \Theta_g \theta_L\tp$ are near the \pNTK{}'s corresponding value of zero,
while diagonal entries are close to $\Theta_f$.
Specifically, we have the following result
for the single-logit approximation,
i.e.\ when $v$ in \eqref{eq:supp:pntk_is_sum_of_entk_lastlayer}
is one-hot.

\begin{Lemma} \label{Lemma:supp:hw-main}
Let $f$ be of the form $f(x) = W g(x)$,
where $W \in \R^{O \times n}$
and $g$ is an arbitrary deep network.
Suppose that $W$ has independent, zero-mean, sub-gaussian entries with variance parameter $\nu$.
Let  $\delta > 0$ be smaller than a constant depending only on $O$.
Let $x_1$, $x_2$ be arbitrary input points.
Then, with high probability over the random value of $W$,
\begin{equation}
    \norm{\Theta_f(x_1, x_2) - \hat\Theta_f(x_1, x_2) I_O}_F \le
    \bigO_p\left( \nu \norm{\Theta_g(x_1, x_2)}_F \right)
.\end{equation}
\end{Lemma}
For standard \citet{he2016kaiming} initialization,
$\nu = \frac1n$.
We will later show that for standard ReLU networks of width $n$ at initialization,
$\norm{\Theta_g(x_1, x_2)}_F = \bigO_p(\sqrt n)$,
implying that the difference between the \pNTK{} and the \eNTK{} is $\bigO_p(1 / \sqrt n)$.

For Gaussian weights $W$,
the same guarantees follow for the sum-of-logits approximation
($v = \frac{1}{\sqrt{O}} \mathbf{1}_O$)
by noting that for fixed $g$ and random $W$,
the joint distribution of $f(x)$ and $\hat f(x) = v\tp f(x)$
are identical for any unit-norm $v$.

Our proof is based on the following key tool.
\begin{proposition}[Hanson-Wright Inequality; \citealt{vershynin2018high}, Theorem 6.2.1] 
\label{lemma:supp:hanson_wright_ineq}
Let $x$ be a random vector with independent centered sub-gaussian entries with variance proxy $\nu$, and let $A$ be a square matrix.
Let $c > 0$ be a universal constant.
Then
\[
\Pr \Big( \big|x\tp A x - \E[x\tp A x] \big| \ge t \Big)
\le 2 \exp\left( -c \min\left( \frac{t^2}{\nu^2 \norm{A}_F^2}, \frac{t}{\nu \norm{A}} \right) \right)
.\]
\end{proposition}
(A version with explicit constants, but in a slightly less convenient form, is given by \citealt{hanson-wright}.)

The following form converts to an error probability of $\delta$,
and gives a slightly simpler but weaker bound using $\norm A \le \norm A_F$.
\begin{corollary} \label{lemma:supp:hanson_wright_bound}
    In the setup of \cref{lemma:supp:hanson_wright_ineq},
    it holds with probability at least $1 - \delta$ that
    \[
        \abs{x\tp A x - \E[x\tp A x]}
        \le \nu \norm{A}_F \max\left( \tfrac1c \log\tfrac2\delta, \sqrt{\tfrac1c \log\tfrac2\delta} \right)
    .\]
\end{corollary}

\begin{corollary} \label{lemma:supp:hanson_wright_diff}
    Let $x$ and $y$ be independent random vectors, whose entries are independent, centered, and sub-gaussian with variance proxy $\nu$.
    Let $A$ be any square matrix,
    and $c > 0$ the universal constant of \cref{lemma:supp:hanson_wright_ineq}.
    Then
    \[
      \Pr\left( \abs{x\tp A y} \ge t \right)
      \le 2 \exp\left( -2c \min\left( \frac{t^2}{\nu^2 \norm{A}_F^2}, \frac{t}{\nu \norm{A}} \right) \right)
    .\]
    This implies that
    with probability at least $1 - \delta$,
    we have that
    \[
        \abs{x\tp A x - \E[x\tp A x]}
        \le \nu \norm{A}_F \max\left( \tfrac{1}{2c} \log\tfrac2\delta, \sqrt{\tfrac{1}{2c} \log\tfrac2\delta} \right)
    .\]
\end{corollary}
\begin{proof}
    Notice that
    \begin{equation} \label{eq:hw-diff-setup}
        \begin{bmatrix}x \\ y\end{bmatrix}\tp
        \underbrace{\begin{bmatrix}0 & \tfrac12 A \\ \tfrac12 A\tp & 0 \end{bmatrix}}_{\tilde A}
        \begin{bmatrix}x \\ y\end{bmatrix}
        = \tfrac12 x\tp A y + \tfrac12 y\tp A\tp x
        = x\tp A y
    .\end{equation}
    The vector $\begin{bmatrix}x \\ y\end{bmatrix}$
    satisfies the conditions for \cref{lemma:supp:hanson_wright_ineq}.
    We have
    $\norm{\tilde A}_F = \sqrt{\tfrac14 + \tfrac14} \, \norm{A}_F = \frac{1}{\sqrt 2} \norm{A}_F$,
    while
    \[
    \norm{\tilde A}
    = \sup_{\norm{x}^2 + \norm{y}^2 = 1}
      \sqrt{\norm{\tfrac12 A y}^2 + \norm{\tfrac12 A\tp x}^2}
    = \tfrac12 \norm{A} \sup_{\norm{x}^2 + \norm{y}^2 = 1} \sqrt{\norm{x}^2 + \norm{y}^2}
    = \tfrac12 \norm{A}
    .\]
    Noting that $\E x\tp A y = 0$,
    the proof follows by applying \cref{lemma:supp:hanson_wright_ineq,lemma:supp:hanson_wright_bound} to \eqref{eq:hw-diff-setup}.
\end{proof}

We are now ready to prove the previous result.
\begin{proof}[Proof of \cref{Lemma:supp:hw-main}]
Assume without loss of generality that $\hat\Theta_f$ uses the first-logit approximation,
i.e.\ $v = (1, 0, \dots, 0)$.
Also assume $O > 1$,
since for $O = 1$ the \eNTK{} and \pNTK{} are trivially identical.

Define the difference matrix between the two kernels as, recalling \eqref{eq:explained_entk} and \eqref{eq:supp:pntk_is_sum_of_entk_lastlayer},
\begin{align*}
D(x_1, x_2)
&= \Theta_f(x_1, x_2) - \hat\Theta_f(x_1, x_2) I_O
\\&= \left[W \Theta_g(x_1, x_2) W\tp + g(x_1)\tp g(x_2) I_O \right]
   - v\tp \left[ W \Theta_g(x_1, x_2) W - g(x_1)\tp g(x_2) \right] v I_O
\\&= W \Theta_g(x_1, x_2) W\tp - W_1\tp \Theta_g(x_1, x_2) W_1 I_O
.\end{align*} 
We will use the Hanson-Wright inequality (\cref{lemma:supp:hanson_wright_ineq,lemma:supp:hanson_wright_bound,lemma:supp:hanson_wright_diff}) to bound each entry of this matrix with high probability,
then use a union bound to bound the overall Frobenius norm of $D(x_1, x_2)$.

We begin by bounding the off-diagonal elements of $D$.
By \eqref{eq:explained_entk} and 
\eqref{eq:supp:pntk_is_sum_of_entk_lastlayer},
we can see that for any $i \ne j$,
$D_{ij}(x_1, x_2) = W_i\tp \Theta^{(L-1)} W_j$.
Because the matrix $D$ is symmetric,
there are $\frac{O (O-1)}{2}$ distinct off-diagonal entries to bound.
Using a union bound with
\cref{lemma:supp:hanson_wright_diff},
we obtain that for any $\delta < 2 O (O-1) e^{-2c}$,
it holds with probability at least $1 - \delta/2$ that
\begin{align} \label{eq:hw-offdiag}
\forall i \ne j, \quad |D_{ij}(x_1, x_2)| \le \nu \norm{\Theta_g(x_1, x_2)}_F \frac{1}{2c} \log \dfrac{2 O (O-1)}{\delta}.
\end{align}

$D_{11}(x_1, x_2)$ is identically zero.
For the other diagonal elements, we have that
\begin{align}
\begin{split}
D_{ii}(x_1, x_2) &= W_i^\top \Theta^{(L-1)} W_i -  W_1^\top \Theta_g W_1 \\
&= (W_i - W_1)^\top \Theta^{(L-1)} (W_i + W_1) \underbrace{{}- W_i^\top \Theta_g W_1 + W_1^\top \Theta_g W_i}_{\; =0 \text{, as } \Theta_g \text{ is symmetric}} \\
&= \begin{bmatrix} W_i \\ W_1 \end{bmatrix}^\top \underbrace{\Bigg( \begin{bmatrix} I_n \\ -I_n \end{bmatrix} \Theta_g \begin{bmatrix} I_n & I_n \end{bmatrix} \Bigg)}_{\Theta_g^*} \begin{bmatrix} W_i \\ W_1 \end{bmatrix}.
\end{split}
\end{align}
Noting that 
$\Theta_g^* = \begin{bmatrix} \Theta_g & \Theta_g \\ -\Theta_g & -\Theta_g  \end{bmatrix}$,
we have that
$\norm{\Theta_g^*}_F = 2 \norm{\Theta_g}_F$.
Thus, applying \cref{lemma:supp:hanson_wright_bound} to each of the $O-1$ entries and taking a union bound,
we find that as long as
$\delta < 4 (O - 1) e^{-c}$,
it holds with probability at least $1 - \delta/2$ that
\begin{equation} \label{eq:hw-diag}
\forall i, \quad \big|D_{ii}(x_1, x_2) \big| \le 2 \nu \norm{\Theta_g(x_1, x_2)}_F \frac{1}{c} \log \dfrac{4(O-1)}{\delta}
.\end{equation}

Combining \eqref{eq:hw-offdiag} and \eqref{eq:hw-diag},
we obtain that as long as
$\delta < 2 (O-1) e^{-c} \min( O e^{-c}, 2 )$,
\[
    \norm{D(x_1, x_2)}_F
    \le \nu \norm{\Theta_g(x_1, x_2)}_F \frac1c
    \sqrt{
        O (O-1) \left( \tfrac12 \log \frac{2 O (O-1)}{\delta} \right)^2
        + (O-1) \left( 2 \log \frac{4 (O-1)}{\delta} \right)^2
    }
.\qedhere\]
\end{proof}

\subsection{Background on sub-exponential variables}
The following proofs rely heavily on concentration inequalities for sub-exponential random variables; we will first review some background on these quantities.

A real-valued random variable $X$ with mean $\mu$ is called \textit{sub-exponential} \citep[see e.g.][]{wainwright_2019} if there are non-negative parameters $(\nu, \alpha)$ such that
\[
\E[e^{\lambda (X - \mu)}] \le e^{\frac{\nu^2 \lambda^2}{2}} \quad \text{ for all } \abs{\lambda} < \frac{1}{\alpha}.
\]
We use $X \sim SE(\nu, \alpha)$ to denote that $X$ is a sub-exponential random variable with parameters $(\nu, \alpha)$,
but note that this is not a particular distribution.

One famous sub-exponential random variable is the product of the absolute value of two standard normal distributions, $z_i \sim \N(0, 1)$, such that the two factors are either independent ($X_1 = \abs{z_1} \abs{z_2} \sim SE(\nu_p, \alpha_p)$ with mean $2/\pi$) or the same ($X_2 = z^2 \sim SE(2, 4)$ with mean 1). We now present a few lemmas regarding sub-exponential random variables that will come in handy in the later subsections of the appendix.

\begin{Lemma} \label{Lemma_scaled_sub_exponential}
If a random variable $X$ is  sub-exponential with parameters $(\nu, \alpha)$, then the random variable $sX$ where $s \in \R^+$ is also sub-exponential with parameters $(s \nu, s\alpha)$.
\end{Lemma}

\begin{proof}
Consider $X \sim SE(\nu, \alpha)$ and $X' = sX$ with $\E[X'] = s \E[X]$, then according to the definition of a sub-exponential random variable

\begin{align}
\begin{split}
    &\E \left[ \exp \left(\lambda (X - \mu) \right) \right]  \le \exp(\frac{\nu^2 \lambda^2}{2}) \quad \text{ for all } \abs{\lambda} < \frac{1}{\alpha} \\
    &\Longrightarrow \E \left[ \exp \left(\frac{\lambda}{s} (sX - s\mu) \right) \right] \le \exp(\frac{\nu^2 s^2 \frac{\lambda^2}{s^2}}{2}) \quad \text{ for all } \abs{\frac{\lambda}{s}} < \frac{1}{s\alpha} \\
    & \xRightarrow{\lambda' = \frac{\lambda}{s}} \E \left[ \exp \left(\lambda' (X' - \mu') \right) \right] \le \exp(\frac{{\nu^2 s}^2 {\lambda'}^2}{2}) \quad \text{ for all } \abs{\lambda'} < \frac{1}{s\alpha}
\end{split}
\end{align}
Defining $\alpha' = s \alpha$ and $\nu' = s \nu$ we recover that $X' \sim SE(s\nu, s \alpha)$.
\end{proof}

\begin{proposition} \label{prop:supp:prob_sum_subexp}
If the random variables $X_i$ for $i \in [1 - N]$ for $N \in \mathbb{N}^+$ are all sub-exponential with parameters $(\nu_i, \alpha_i)$ and independent, then $\sum_{i=1}^N X_i \in SE(\sqrt{\sum_{i=1}^N \nu_i^2}, \max_i \alpha_i)$,
and 
$\frac{1}{N} \sum_{i=1}^N X_i \sim SE\left( \frac{1}{\sqrt N} \sqrt{\frac1N \sum_{i=1}^N \nu_i^2}, \frac1N \max_i \alpha_i\right)$.
\end{proposition}
\begin{proof}
This is a simplification of the discussion prior to equation 2.18 of \citet{wainwright_2019}.
\end{proof}

\begin{proposition} \label{bernsetin_ineq}
For a random variable $X \sim SE(\nu, \alpha)$, the following concentration inequality holds:
\[
\Pr \left( \abs{X - \mu} \ge t \right) \le 2 \exp \left( - \min \left( \frac{t^2}{2\nu^2}, \frac{t}{2\alpha} \right) \right)
.\]
\end{proposition}

\begin{proof}
Direct from multiplying the result derived in Equation 2.18 of \citet{wainwright_2019} by a scalar.
\end{proof}

\begin{corollary} \label{bernstein_in_delta}
For a random variable $X \sim SE(\nu, \alpha)$, the following inequality holds with probability at least $1-\delta$:
\[
\abs{X - \mu} < \max \left(\nu \sqrt{2 \log \frac{2}{\delta}}, 2\alpha \log \frac{2}{\delta} \right)
.\]
\end{corollary}

\subsection{Bound on Growth of ReLU Networks' NTKs} \label{sec:relu-ntk-growth}

We will now specialize to fully-connected ReLU networks,
and show the remainder of the required results in that setting.

Let's denote a neural network with $L$ dense hidden layers whose width is $n$ as:
\begin{align}
\begin{split}
    f^0(x) &= x \\
    f^{l+1}(x) &= \phi(W^{(l+1)}f^{l}(x)) \\
    f(x) &= f^L(x) = W^{(L)} f^{L-1}(x)
\end{split}
\end{align}
such that $\phi$ is a differentiable coordinate-wise activation function.

\begin{setting}[ReLU-MLP] \label{sett:supp:relu_mlp_setting}
We make the following assumptions about the network $f$:
\begin{itemize}
    \item We assume $W^{(l)} \in \R^{n_l \times n_{l-1}}$ for $l \in {1, \dots, L}$ is initialized according to the \citet{he2016kaiming} initialization (``fan-in''), meaning that each scalar parameter is distributed according to $\mathcal{N}(0, 1/n_{l-1})$. 
    \item We assume the width of all hidden layers are identical (and equal to $n$). The proof extends naturally to the case of non-equal widths as long as $n_{l+1}/n_l \to c_l \in (0, \infty)$ for each consecutive pair of layers.
    \item We assume $\phi$ is the ReLU activation. \textcolor{black}{This can be generalized to 1-Lipschitz, ReLU-like functions such as GeLU, PReLU, and so on, as discussed in \cref{app:other-archs}.}
    \item We assume the training data $\cX$ is finite and contained in a compact set and there are no overlapping datapoints.
\end{itemize}
\end{setting}

\textbf{A Note On Parameterization} Although we assume a Gaussian distribution for each scalar variable, the proofs in this section apply to any other distribution used for scalar parameters as long as:
\begin{itemize}
    \item The variance of the parameters is set according to \citet{he2016kaiming}, and the mean is zero.
    \item Each scalar parameter is initialized independently of all other ones.
    \item The distribution used is sub-Gaussian; specifically, $w_{ij}^{l+1} \in SG(1 / n_l)$.
\end{itemize}
This applies to all bounded initialization methods, like truncated normal or uniform on an interval.

In general, the product of two sub-Gaussian distributions has a sub-exponential distribution.
For the product of two independent weights, $w^{l+1}_{ij} w^{l+1}_{ab}$ with $i \ne a$ and/or $j \ne b$,
we denote the parameters as
$\frac{1}{n_l} SE(\nu_p, \alpha_p)$.
For $(w^{l+1}_{ij})^2$, we use the parameters $\frac{1}{n_l} SE(\nu_s, \alpha_s)$, whose mean is $\mu_s \ne 0$.

Note that we can recursively define the \eNTK~of $f^{l+1}$ using the \eNTK~of $f^l$ as
\begin{align}
\label{eq:supp:ntk_chain_expansion}
\begin{split}
    \Theta^{(l+1)}(x_1, x_2) &= \sum_{i=1}^{l} \frac{\partial f^{l+1}(x_1)}{\partial W^{(i)}} \, {\frac{\partial f^{l+1}(x_2)}{\partial W^{(i)}}}^\top + \overbrace{\frac{\partial f^{l+1}(x_1)}{\partial W^{l+1}} \, {\frac{\partial f^{l+1}(x_2)}{\partial W^{l+1}}}^\top}^{K_D^{l+1}(x_1, x_2)} \\
    &= \sum_{i=1}^{l} \frac{\partial \phi(W^{(l+1)} f^{l}(x_1))}{\partial W^{(i)}} \, {\frac{\partial \phi(W^{(l+1)} f^{l}(x_2))}{\partial W^{(i)}}}^\top + K_D^{l+1}(x_1, x_2) \\
    &= \sum_{i=1}^l \frac{\partial \phi(W^{(l+1)} f^{l}(x_1))}{\partial f^{l}(x_1)} \frac{\partial f^l(x_1)}{\partial W^{(i)}} \, {\frac{\partial f^l(x_2)}{\partial W^{(i)}}}^\top {\frac{\partial \phi(W^{(l+1)} f^{l}(x_2))}{\partial f^{l}(x_2)}}^\top + K_D^{l+1}(x_1, x_2) \\
    &= \frac{\partial \phi(W^{(l+1)} f^{l}(x_1))}{\partial f^{l}(x_1)} \left [ \sum_{i=1}^l \frac{\partial f^l(x_1)}{\partial W^{(i)}} \, {\frac{\partial f^l(x_2)}{\partial W^{(i)}}}^\top \right ] {\frac{\partial \phi(W^{(l+1)} f^{l}(x_2))}{\partial f^{l}(x_2)}}^\top + K_D^{l+1}(x_1, x_2) \\
    &= \frac{\partial \phi(W^{(l+1)} f^{l}(x_1))}{\partial f^{l}(x_1)} \Theta^{(l)}(x_1, x_2) {\frac{\partial \phi(W^{(l+1)} f^{l}(x_2))}{\partial f^{l}(x_2)}}^\top + K_D^{l+1}(x_1, x_2)
\end{split}
\end{align}
where
$K_D^{l+1}(x_1, x_2) = f^{l}(x_1)^\top f^{l}(x_2) I_n$ is a diagonal matrix,
and
\begin{align}
\begin{split}
    \frac{\partial \phi(W^{(l+1)} f^{l}(x))}{\partial f^{l}(x)}
    &= W^{(l+1)} \odot \left[\dot{\phi}(W^{(l+1)} f^{l}(x))\right]_{1 \times n}
    = \left[
    W^{(l+1)}_{ij}
    \dot{\phi}\left( W^{(l+1)}_{i,:} \cdot f^l(x) \right)
    \right]_{ij}
\end{split}
\end{align}
with $\odot$ the elementwise (Hadamard) product using ``broadcasting,''
and $\dot\phi$ the derivative of $\phi$.
We can think of the last layer as following the same equations with $\phi$ the identity function, so that $\dot\phi(x) = 1$.

Before moving on, it's useful to first show a simple inequality on the elements of a tangent kernel based on the Lipschitz-ness of the activation function; this will help us further in deriving the aforementioned bounds. Define $V^{(l)}(x) = W^{(l)} \odot \left[\dot{\phi}(W^{(l)} f^{l-1}(x))\right]_{1 \times n}$.
We can write each entry of $\Theta^{(l+1)}(x_1, x_2)$ as 
\begin{align}
{\Theta^{(l+1)}(x_1, x_2)}_{ij} &= \sum_{a=1}^n \sum_{b=1}^n {V^{(l+1)}(x_1)}_{ia} {V^{(l+1)}(x_2)}_{jb}{\Theta^{(l)}(x_1, x_2)}_{ab} + {{f^l(x_1)}^\top {f^l(x_2)}} \, \mathcal{I}(i=j)
\notag{}
\\
\abs{{\Theta^{(l+1)}(x_1, x_2)}_{ij}}
&\le \Abs{\sum_{a=1}^n \sum_{b=1}^n {V^{(l)}(x_1)}_{ia} {V^{(l)}(x_2)}_{jb} {\Theta^{(l)}(x_1, x_2)}_{ab}} + \abs{{f^l(x_1)}^\top {f^l(x_2)}} \, \mathcal{I}(i=j). \label{eq:supp:entk_expansion_v}
\end{align}

\begin{Lemma}[Diagonality of the first layer's tangent kernel] \label{Lemma:supp:elements_of_first_layer_2}
For a NN under \cref{sett:supp:relu_mlp_setting}, the \eNTK~of the first layer $\Theta^{(1)}(x_1, x_2)$ is diagonal. Moreover, there is a corresponding constant $C^{(1)} > 0$ such that for \textcolor{black}{all} diagonal elements ${\Theta^{(1)}(x_1, x_2)}_{ii}$, we have that
\begin{equation*}
    \abs{{\Theta^{(1)}(x_1, x_2)}_{ii}} \le C^{(1)}.
\end{equation*}
\end{Lemma}
\begin{proof}
Consider the one layer NN $f^{1}(x) = \phi(W^{(1)}x)$. For this case, we have:
\begin{align}
\begin{split}
    &{\Theta^{(1)}(x_1, x_2)}_{ij} = 
        \begin{dcases*}
            \sum_{a=1}^D x_{1a} \dot{\phi}(W_i x_1) x_{2a} \dot{\phi}(W_i x_2)  & if $i = j$ \\
            0 & if $i \ne j$
        \end{dcases*}
\end{split}
\end{align}
and thus, since the activation function $\phi$ is 1-Lipschitz we can conclude that \textcolor{black}{for all $i, j$}
\begin{align}
\begin{split}
    \abs{{\Theta^{(1)}(x_1, x_2)}_{ij}} \le
        \begin{dcases*} 
            \sum_{a=1}^D \abs{x_{1a}} \abs{x_{2a}} & if $i = j$ \\
            0 & if $i \ne j$
        \end{dcases*}
\end{split}.
\end{align}
Thus, the tangent kernel of the first layer is a diagonal matrix whose entries are independent of the width of the first layer ($n$), and can be bounded by a positive constant, \textcolor{black}{given by $C^{(1)} = \max_{(x_1, x_2) \in \cX \times \cX} \sum_{a=1}^D |x_{1a} x_{2a}|$}.
\end{proof}

\begin{Lemma} \label{Lemma:supp:entk_fro_norm_upper_bound}
Consider a NN $f$ under \cref{sett:supp:relu_mlp_setting}. For every small constant $\delta > 0$, $l \in [L-1]$, and arbitrary datapoints $x_1$ and $x_2$, it holds that
\begin{align}
\begin{split}
    \norm{\Theta^{(l)}(x_1, x_2)}_F \le \bigO(n\sqrt{n})
\end{split}
\end{align}
with probability at least $1 - \delta$ over random initialization for any $n > n_0$, where this lower bound $n_0$ is $\bigO\left(\polylog(L/\delta) \right)$. 
\end{Lemma}
\begin{proof}
Using the recursive definition of~\eNTK~in \eqref{eq:supp:entk_expansion_v} we can see that for all $l \ge 2$,
\begin{align}
\begin{split}
||\Theta^{(l)}(x_1, x_2)||_F &\le \; ||{V^{(l)}}^\top \Theta^{(l-1)}(x_1, x_2) V^{(l)} ||_F + || {f^{(l-1)}(x_1)}^\top {f^{(l-1)}(x_2)} \, I_n ||_F \\
& \le \bigO(||\Theta^{(l-1)}(x_1, x_2)||_F \log \frac 2\delta) + \Theta(n) \sqrt{n} \log \frac {2l}{\delta}
\end{split}
\end{align}
with probability at least $1 - 2\delta$. To derive the inequality, note that for all $2 \le l \le L - 1$, $\norm{V^{(l)}}_F = \bigO(1)$ with high probability \citep[Appendix G.3]{linntk2019lee} and that the inner product of post-activations grows linearly as shown in \cref{app:Lemma_relu_post_activation}. Since for the first layer we have that $\norm{\Theta^{(1)}(x_1, x_2)}_F = \bigO(\sqrt{n})$ (refer to \cref{Lemma:supp:elements_of_first_layer_2}), we can conclude the proof as long as the minimum width $(n_0)$ is chosen appropriately to satisfy the linear growth of post-activations with high probability as in \cref{app:Lemma_relu_post_activation}.
\end{proof}

The following is a version of \cref{theorem:pntk_fro_norm}, which we are now ready to prove.
\begin{theorem} \label{supp:lower_bound_proof}
Consider a NN $f$ under \cref{sett:supp:relu_mlp_setting}. For every arbitrary small $\delta > 0$ and the arbitrary datapoints $x_1$ and $x_2$, there exists $n_0$ such that

\begin{equation}
    \frac{\norm{\Theta^{(L)}(x_1, x_2) - \hat\Theta^{(L)}(x_1, x_2) \otimes I_O}_F}{\norm{\Theta^{(L)}(x_1, x_2)}_F} = \bigO \left(\frac{1}{\sqrt{n}} \right)
\end{equation}

with probability at least $1 - \delta$ for $n > n_0$.
\end{theorem}
\begin{proof}
We will start by showing that $\norm{\Theta^{(L)}(x_1, x_2)}_F$ where $L$ denotes the last layer is $\Theta(n)$ with high probability over random initialization. 

Using \cref{eq:supp:entk_expansion_v} we can show that for the off-diagonal elements we have
\begin{align}
\begin{split}
\abs{\Theta^{(L)}_{ij}(x_1, x_2)} \le \frac{\norm{\Theta^{(L-1)}(x_1, x_2)}_F}{n} \log \frac{2}{\delta}
\end{split}
\end{align}
with probability at least $1 - \delta$. Applying a union bound over the $O^2 - O$ off-diagonal elements we have:
\begin{align}
\begin{split}
\forall i \ne j \in [O]; \; \abs{\Theta^{(L)}_{ij}(x_1, x_2)} \le \frac{\norm{\Theta^{(L-1)}(x_1, x_2)}_F}{n} \log \frac{2O^2}{\delta}
\end{split}
\end{align}
Likewise, for the diagonal elements we have 
\begin{align}
\begin{split}
\abs{\Theta^{(L)}_{ii}(x_1, x_2) - \frac{1}{n} \tr(\Theta^{(L-1)}(x_1, x_2))} \le \frac{\norm{\Theta^{(L-1)}(x_1, x_2)}_F}{n} \log \frac{2}{\delta}
\end{split}
\end{align}
with probability at least $1 - \delta$ where $\E[{V^{(L)}_i}^\top \Theta^{(L-1)}(x_1, x_2) V^{(L)}_i] = \frac{1}{n} \tr(\Theta^{(L-1)}(x_1, x_2))$. Using \cref{app:Lemma_relu_post_activation}, we can see that $\tr(\Theta^{(L-1)}(x_1, x_2)) = \Theta(n^2) \log \frac{2n}{\delta}$ with probability at least $1 - \delta$ as long as minimum width $n_0$ satisfies the conditions for \cref{app:Lemma_relu_post_activation}. Putting it all together and applying a union bound over the diagonal elements of $\Theta^{(L)}(x_1, x_2)$ we can see that
\begin{align}
\begin{split}
\forall i \in [O]; \; \abs{\Theta^{(L)}_{ii}(x_1, x_2) - \Theta(n \log n)} \le \sqrt{n} \log \frac{2O}{\delta}
\end{split}
\end{align}
with probability at least $1 - \delta$. Combining the bounds on off-diagonal and diagonal elements of the kernel matrix, we can see that $\norm{\Theta^{(L)}}_F = \Theta(n)$ with high probability over random initialization.
The result follows from \cref{Lemma:supp:hw-main}.
\end{proof}

\begin{Lemma} \label{app:Lemma_relu_post_activation}
Consider a NN under \cref{sett:supp:relu_mlp_setting} with $L \ge 2$ and ReLU activation function. The dot product of two post-activations $\abs{{f^{(l)}(x_1)}^\top f^{(l)}(x_2)}$ grows linearly with the width of the network with high probability over random initialization.
\end{Lemma}
\begin{proof}
We begin by showing that the dot product of the post-activations of the first layer of the NN under setting \cref{sett:supp:relu_mlp_setting} grow linearly using a simple Hoeffding bound. Next, we apply Thorem 1 of \citet{arpit2019overparameterization} to show that the magnitude of this dot product is preserved in the next layers. 
First, note that as we assume the data lies in a compact set and as the post-activations are all positive, one can easily see that for each $x_1, x_2 \in \cX$ and for all $l \in [L]$ we have that:
\begin{equation} \label{eq:post_activation_bounded_by_norm_of_po}
    \min_{x \in \cX}{\norm{f^{(l)}(x)}}^2 \le {f^{(l)}(x_1)}^\top f^{(l)}(x_2) \le \max_{x \in \cX}{\norm{f^{(l)}(x)}}^2.
\end{equation}
To simplify the proofs in this Lemma, we use this fact and instead work with the norm of the post-activations and we note that the final result on the norms can be accordingly applied to dot products of post-activations of different inputs.
For the first layer, we have that $f^{(1)}(x) = \phi(W^{(1)}x)$ where $W^{(l+1)}_{ij} \sim \mathcal{N}(0, \frac{1}{n_l})$ and $x \in \R^{n_0}$. Hence, each ${f^{(1)}(x)}_i$ is i.i.d and distributed as $\mathcal{N}^R(0, \frac{\norm{x}^2}{n_0})$ where $\mathcal{N}^R$ is the Rectified Normal Distribution. Using the properties of the Rectified Normal distribution we get that:
\begin{equation}
    \E[\norm{f^{(1)}(x)}^2] = \frac{n \norm{x}^2}{n_0}
\end{equation}

Next, as the Rectified Normal is a sub-gaussian distribution, we can apply the Hoeffding bound to see that 
\begin{equation} \label{eq:first_layer_postac_norm_bound}
    \Pr \left[ \abs{\norm{f^{(1)}(x)}^2 - n \mu_1} \le \varepsilon_1 \right] \ge 1 - \delta_1
\end{equation}
where $\delta_1 = 2 \exp \left( - \frac{\varepsilon_1^2}{2\sigma^2} \right)$, $\mu_1 = \frac{\norm{x}^2}{n_0}$ and $\sigma$ is the standard deviation of $\norm{f^{(1)}(x)}^2$ over random initialization of the weights of the first layer. 
Next, we can adapt Theorem 1 from \cite{arpit2019overparameterization} to see that for post activations of layer $l \in [2-L]$
\begin{equation}
    \Pr \left[(1 - \varepsilon)^{l-1} \norm{f^{(1)}(x)}^2 \le \norm{f^{(l)}(x)}^2  \le (1 + \varepsilon)^{l-1} \norm{f^{(1)}(x)}^2 \right] \ge 1 - \delta_2
\end{equation}
where $\delta_2 = 2N(l-1) \exp \left (-n \left( \frac{\varepsilon}{4} + \log \frac{2}{1 + \sqrt{1 + \varepsilon}} \right) \right)$, $N$ is the size of our dataset and $\varepsilon$ is any positive small constant. Combining this with the result from the first layer's post-activations we can see that
\begin{equation}
    (1 - \varepsilon_2)^{l-1} (n \mu_1 - \varepsilon_1) \le \norm{f^{(l)}(x)}^2  \le (1 + \varepsilon_2)^{l-1} (n \mu_1 + \varepsilon_1)
\end{equation}
with probability at least $1 - \delta_1 - \delta_2$. Hence, for any $\delta > 0$, $n = \Omega \left( \log \frac{1}{\delta} \right)$ and $(x_1, x_2) \in \cX \times \cX$ one can come up with constants $G^{(l)}_1 = \Omega \left(\log(\frac{n}{\delta}) \right), G^{(l)}_2 = \bigO \left(\log(\frac{n}{\delta}) \right)$ for post activations of layer $l$ such that
\begin{equation}
    G^{(l)}_1 n \le {f^{(l)}(x_1)}^\top f^{(l)}(x_2) < G^{(l)}_2 n
\end{equation}
with probability at least $1 - \delta$ (note that exact values of $G_1^{(l)}$ and $G_1^{(l)}$ depend on $l$ and $N$ too).
\end{proof}

\section{pNTK's Maximum Eigenvalue Converges to eNTK's Maximum Eigenvalue as Width Grows}
\label{app:eigs}

\begin{proof}[Proof of \cref{theorem:pntk_max_eigval}]
Note that, as both \pNTK~and \eNTK~are symmetric PSD matrices, their maximum eigenvalues are equal to their spectral norm. Furthermore, the spectral norm of a matrix is upper-bounded by its Frobenius norm. Now, note that according to the triangle inequality, we have

\begin{align}
\begin{split}
\norm{\Theta(x_1, x_2)} &= \norm{\hat{\Theta}(x_1, x_2) \otimes I_O + \left( \Theta(x_1, x_2) - \hat{\Theta}(x_1, x_2) \otimes I_O \right)} \\ &\le \norm{ \hat{\Theta}(x_1, x_2) \otimes I_O} + \norm{ \Theta(x_1, x_2) - \hat{\Theta}(x_1, x_2) \otimes I_O}
\end{split}
\end{align}

Thus
\begin{equation}
    \norm{\Theta(x_1, x_2)} - \norm{\hat{\Theta}(x_1, x_2) \otimes I_O} \le  \norm{ \Theta(x_1, x_2) - \hat{\Theta}(x_1, x_2) \otimes I_O}.
\end{equation}
which according to \eqref{supp:fro_norm_proof} together with the fact that for any matrix $A$, $\lambda_{\max} (A \otimes I) = \lambda_{\max} (A)$ implies that
with probability at least $1 - \delta$,
\begin{align}
\begin{split}
    &\left \lvert \lambda_{\max} \left( \Theta(x_1, x_2) \right) - \lambda_{\max} \left( \hat{\Theta}(x_1, x_2) \right) \right \rvert \le \\
    &\qquad \qquad \qquad \qquad 4 O \left( C^{(L-1)}_1 + C^{(L-1)}_2 \right) \sqrt{n} \max \left( \sqrt{ \log \frac{4 O^2}{\delta}}, \sqrt{2} \log \frac{4 O^2}{\delta} \right).
\end{split}
\end{align}
Moreover, as mentioned in the proof of \cref{supp:lower_bound_proof}, combining the previous inequality with the fact that $\lambda_{\max} \left({\Theta(x_1, x_2)} \right) \ge \Omega(n)$ with high probability shows that there exists $\delta'$ and $n_0$ such that

\begin{equation} \label{eq:max_eigval_proof}
    \left \lvert \frac{ \lambda_{\max} \left( \Theta(x_1, x_2) \right) - \lambda_{\max} \left( \hat{\Theta}(x_1, x_2) \right)}{\lambda_{\max} \left({\Theta(x_1, x_2)} \right)}\right \rvert
    \le \bigO(1/\sqrt{n})
\end{equation}

with probability $1 - \delta'$ over random initialization for $n > n_0$ as desired.

\end{proof}

\section{Kernel Regression Using pNTK vs Kernel Regression Using eNTK} 

\begin{proof}[Proof of \cref{theorem:kr_norm_diff}]
We start by proving a simpler version of a theorem, and then show a correspondence that expands the result of the simpler proof to the original Theorem. Assuming $\abs{\cX} = \abs{\cY} = N$ (training data), we define
\begin{equation}
    h(x) = \Theta(x_1, \cX) {\Theta(\cX, \cX)}^{-1} \cY \text{  and  } \hat{h}(x) = \left(\hat{\Theta}(x_1, \cX) \otimes I_O\right) \left({\hat{\Theta}(\cX, \cX) \otimes I_O}\right)^{-1} \cY.
\end{equation}

Note that as the result of kernel regression (without any regularization) does not change with scaling the kernel with a fixed scalar, we can use a weighted version of the kernels mentioned in the previous equation without loss of generality. Accordingly, we define 
\begin{equation}
    \alpha = {\left( \frac{1}{n}\Theta(\cX, \cX) \right)}^{-1} \cY \text{ and } \hat{\alpha} = \left({\frac{1}{n} \hat{\Theta}(\cX, \cX) \otimes I_O}\right)^{-1} \cY.
\end{equation}

Using the fact that $\hat{M}^{-1} - M^{-1} = -\hat{M}^{-1} (\hat{M} - M) M^{-1}$ and $(A \otimes I)^{-1} = A^{-1} \otimes I$ we can show that

\begin{align}
\begin{split}
\hat{\alpha} - \alpha = - {\hat{\Theta}(\cX, \cX)}^{-1} \otimes I_O \left ( \frac{1}{n} \hat{\Theta}(\cX, \cX) \otimes I_O - \frac{1}{n} \Theta(\cX, \cX) \right )^{-1} \Theta(x_1, x_2) \cY
\end{split}
\end{align}

Assume $\lambda = \min \left( \lambda_{\min}(\Theta(\cX, \cX)), \lambda_{\min}(\hat{\Theta}(\cX, \cX))  \right)$. Then

\begin{align}
\begin{split}
    \norm{\hat{\alpha} - \alpha} \le \frac{1}{\lambda^2} \norm{\frac{1}{n} \hat{\Theta}(\cX, \cX) \otimes I_O - \frac{1}{n} \Theta(\cX, \cX)} \norm{\cY}
\end{split}
\end{align}

Plugging into the formula for kernel regression, we get that
\begin{align}
\begin{split}
    \hat{h}(x) - h(x) &= \left( \frac{1}{n} \hat{\Theta}(x, \cX) \otimes I_O \right) \hat{\alpha} - \frac{1}{n} \Theta(x, \cX) \alpha \\
    &= \left ( \frac{1}{n} \hat{\Theta}(x, \cX) \otimes I_O - \frac{1}{n} \Theta(x, \cX) \right) \hat{\alpha} + \frac{1}{n} \Theta(x, \cX) (\hat{\alpha} - \alpha) 
\end{split}
\end{align}

Thus

\begin{align}
\begin{split}
      \norm{\hat{h}(x) - h(x)} &\le \norm{\frac{1}{n} \hat{\Theta}(x, \cX) \otimes I_O - \frac{1}{n} \Theta_f(x, \cX)} \norm{\hat{\alpha}} + \norm{\frac{1}{n} \Theta(x, \cX)} \norm{\hat{\alpha} - \alpha}  \\
    &\le \frac{1}{\lambda} \norm{\frac{1}{n}\hat{\Theta}(x, \cX) \otimes I_O - \frac{1}{n} \Theta(x, \cX)} \norm{\cY} \\
    & \quad + \frac{1}{\lambda^2} \norm{\frac{1}{n} \Theta(x, \cX)} \norm{\frac{1}{n} \hat{\Theta}(\cX, \cX) \otimes I_O - \frac{1}{n} \Theta(\cX, \cX)} \norm{\cY}.
\end{split}
\end{align}

Now, note that as for a block matrix $A$ of $A_{ij}$ blocks we have that $\norm{A} \le \sum_{i,j} \norm{A_{ij}}$ it follows that for any matrix valued kernel $K$
\begin{equation}
\norm{K(\cX, \cX)} \le \sum_{x_1, x_2 \in \cX} \norm{K(x_1, x_2)}.
\end{equation}
Using this fact, we can rewrite the bound as
\begin{align}
\begin{split}
      \norm{\hat{h}(x) - h(x)} &\le \frac{N}{\lambda} \norm{\frac{1}{n}\hat{\Theta}(x, x^*_1) \otimes I_O - \frac{1}{n} \Theta(x, x^*_1)} \norm{\cY} \\
    & \qquad + \frac{N^2}{\lambda^2} \norm{\frac{1}{n}\Theta(x, \cX)} \norm{\frac{1}{n}\hat{\Theta}(x_2^*, x_3^*) \otimes I_O - \frac{1}{n}\Theta(x_2^*, x_3^*)} \norm{\cY}
\end{split}
\end{align}
for some particular $x_1^*, x_2^*, x_3^* \in \cX$. Using \eqref{supp:fro_norm_proof},
we can see with probability at least  $1 - \delta$ that
\begin{equation} \label{eq:kr_fro_diff}
    \left \lVert \hat{h}(x) - h(x) \right \rVert
    \le
    \frac{4 N O \alpha}{\lambda \sqrt{n}} \max \left( \sqrt{ \log \frac{4 O^2}{\delta}}, \sqrt{2} \log \frac{4 O^2}{\delta} \right)
    \norm{\cY} \left(1 + \frac{N}{\lambda} \norm{\frac{1}{n}\Theta(x, \cX)} \right)
.\end{equation}

To show the correspondence between $\hat{h}(x)$ and  $\hat{f}^\mathit{lin}(x)$,
as in \eqref{eq:pseudo_lin}, note that
\begin{align}
\begin{split}
    \hat{h}(x) &= \left( \hat{\Theta}(x, \cX) \otimes I_O \right) \left( {\hat{\Theta}(\cX, \cX)}^{-1} \otimes I_O \right) \cY \\
    &= \left(\hat{\Theta}(x, \cX) {\hat{\Theta}(\cX, \cX)}^{-1} \otimes I_O \right) \cY \\
    &= \operatorname{vec} \left(I_O \cY_v \hat{\Theta}(x, \cX) {\hat{\Theta}(\cX, \cX)}^{-1} \right)
\end{split}
\end{align}

where $\cY_v = \operatorname{vec}^{-1} (\cY)$ is the result of inverse of the vectorization operation, converting the $NO \times 1$ vector to a $O \times N$ matrix. Thus, $\hat{h}(x) = \hat{\Theta}(x, \cX) {\hat{\Theta}(\cX, \cX)}^{-1} \cY'$, where $\cY'$ is the $N \times O$ matrix reshaped from the $NO \times 1$ vector $\cY$.
\end{proof}

\label{app:kernel-regression}

\begin{figure*}[t!]
    \centering
    \begin{subfigure}[b]{0.24\textwidth}
        \includegraphics[width=\textwidth]{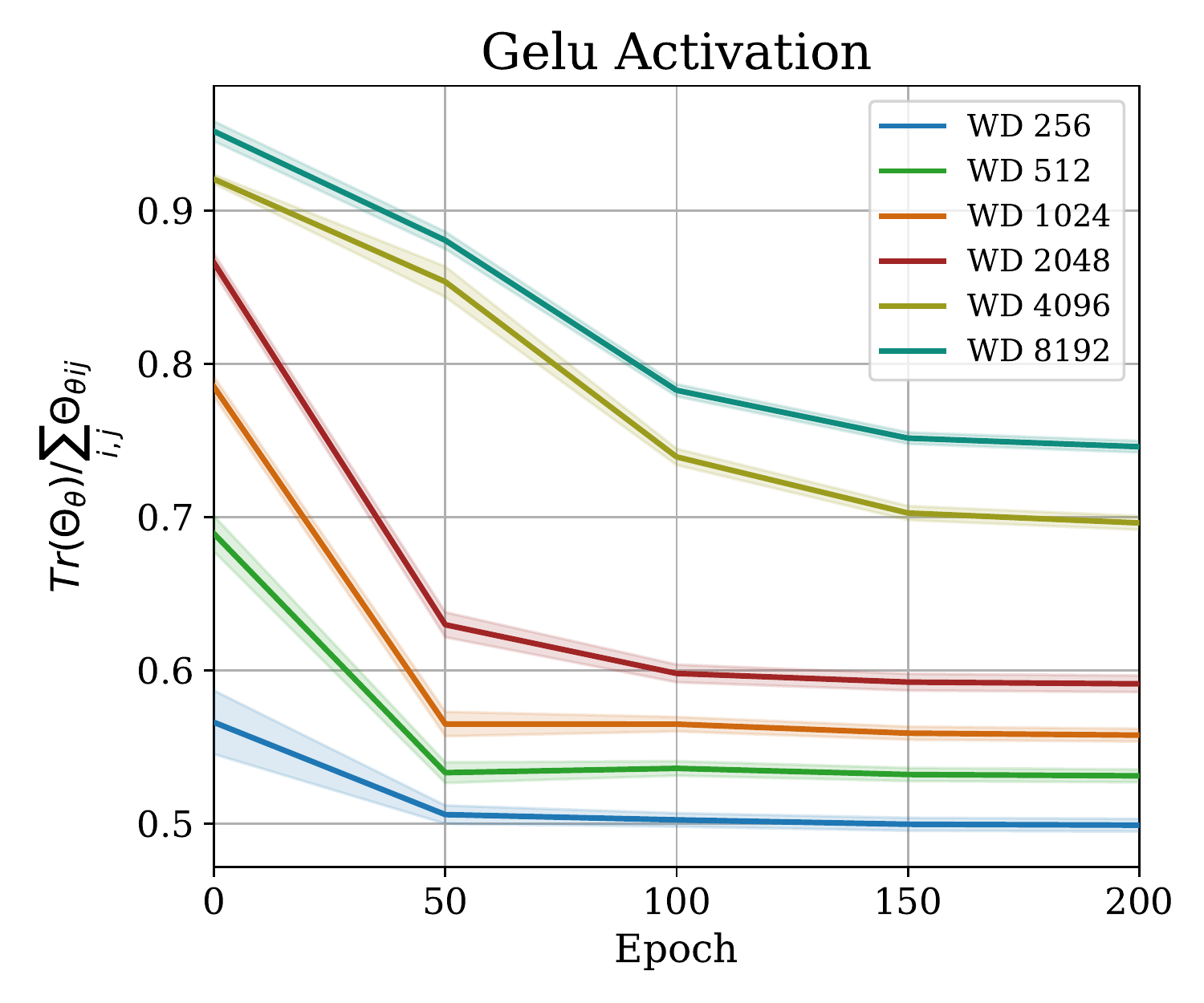}
    \end{subfigure}
    \hfill
    \begin{subfigure}[b]{0.24\textwidth}
        \includegraphics[width=\textwidth]{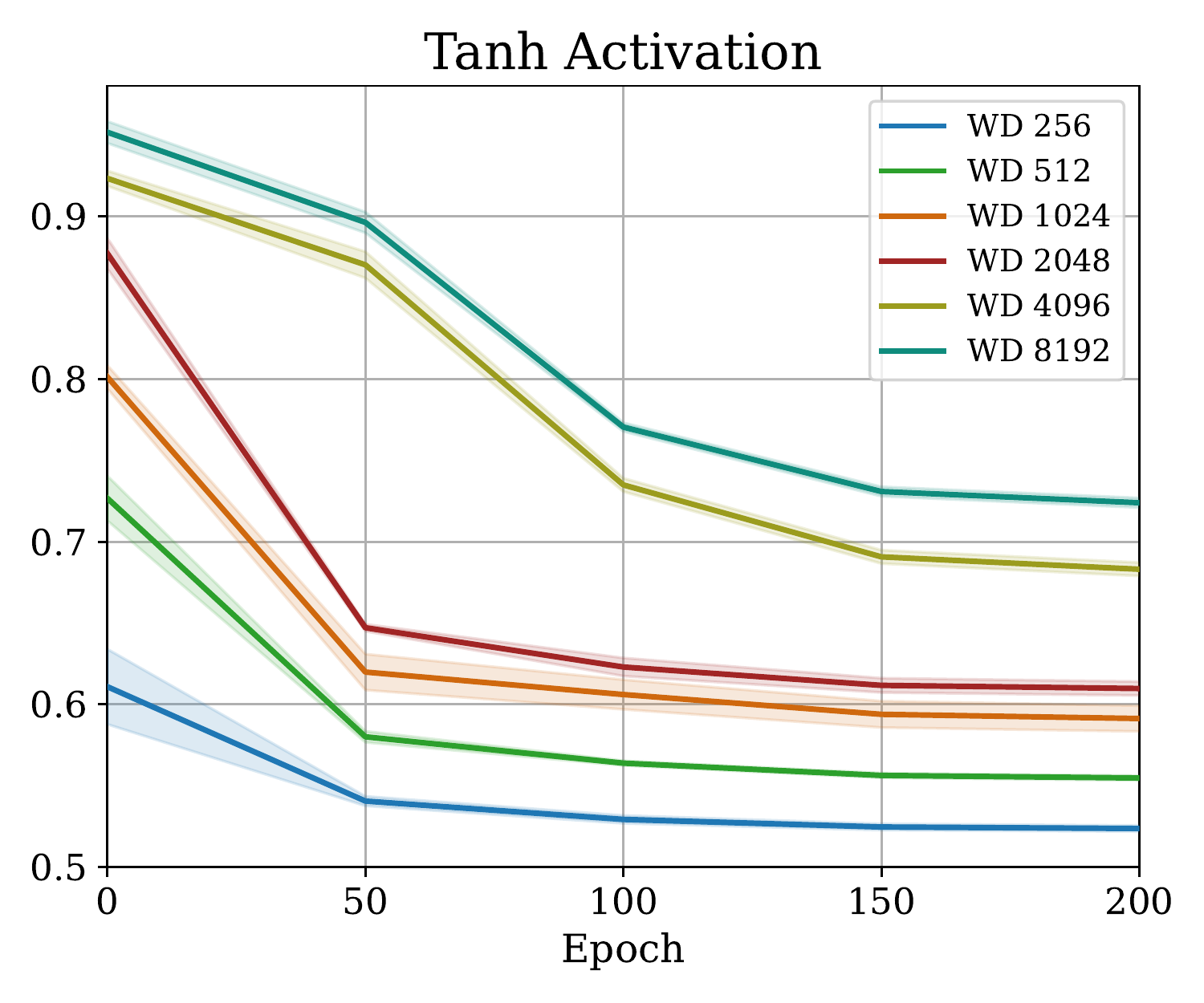}
    \end{subfigure}
    \hfill
    \begin{subfigure}[b]{0.24\textwidth}
        \includegraphics[width=\textwidth]{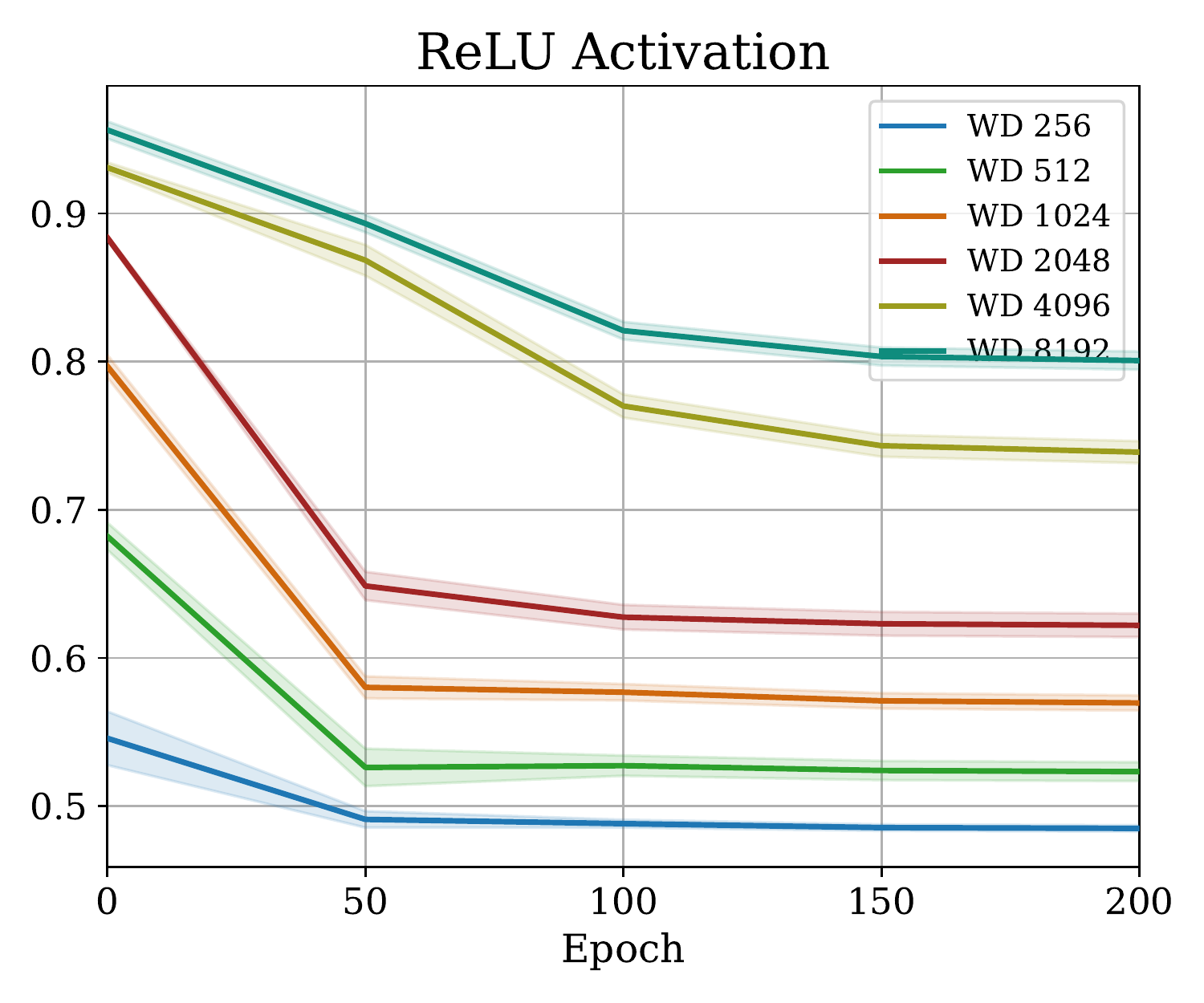}
    \end{subfigure}
    \hfill
    \begin{subfigure}[b]{0.24\textwidth}
        \includegraphics[width=\textwidth]{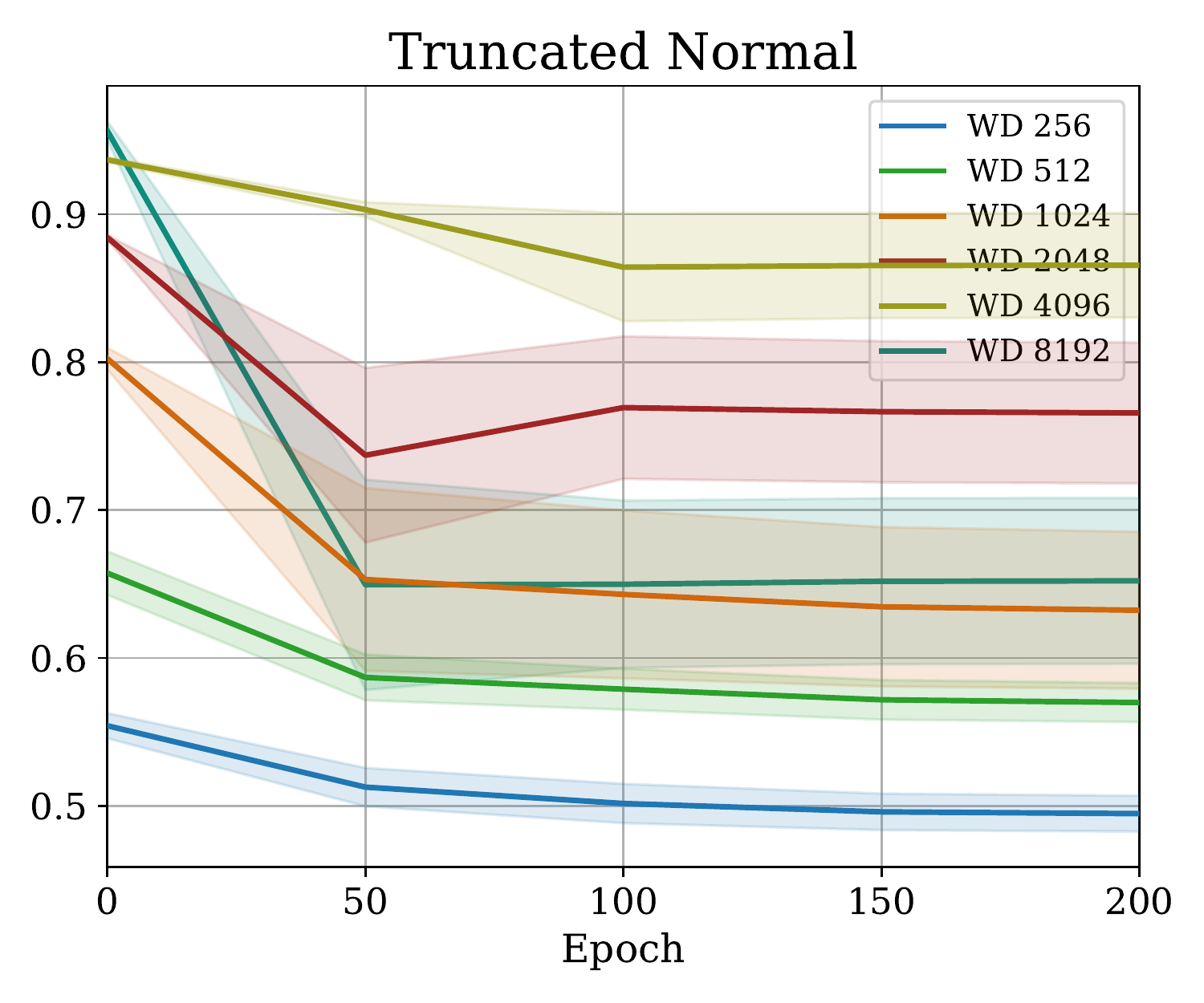}
    \end{subfigure}
    \caption{%
        \textcolor{black}{Comparing the \textbf{magnitude of sum of on-diagonal and off-diagonal elements of $\meNTK$} at initialization and throughout training,
        based on $1000$ points from CIFAR-10.
        The reported numbers are the average of $1000 \times 1000$ kernels each having a shape of $10 \times 10$.
        The same subset has then been used to train the NN using SGD.}
    }
    \label{fig:ablation_diagonality}
\end{figure*}

\begin{figure*}[t!]
    \centering
    \begin{subfigure}[b]{0.24\textwidth}
        \includegraphics[width=\textwidth]{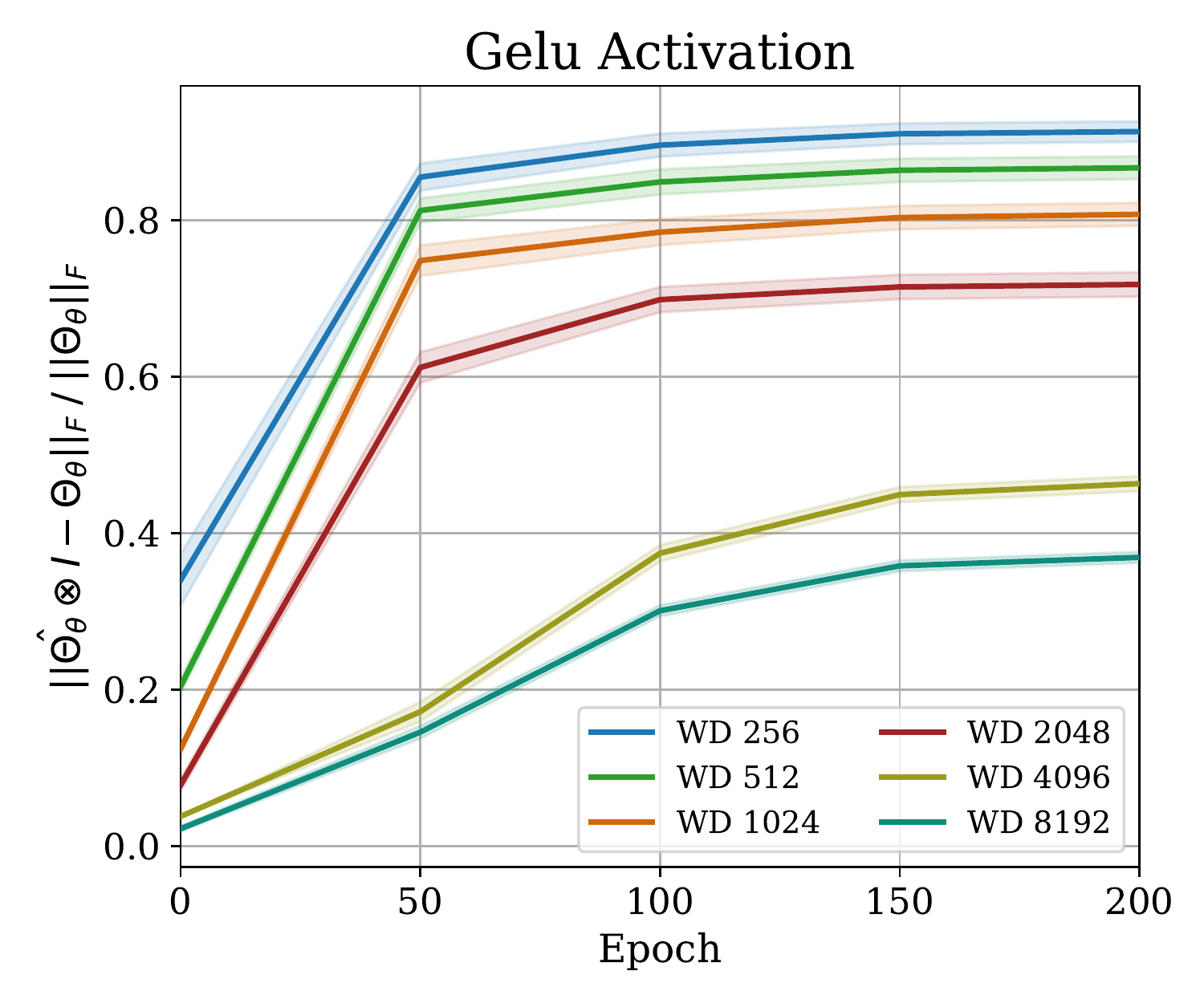}
    \end{subfigure}
    \hfill
    \begin{subfigure}[b]{0.24\textwidth}
        \includegraphics[width=\textwidth]{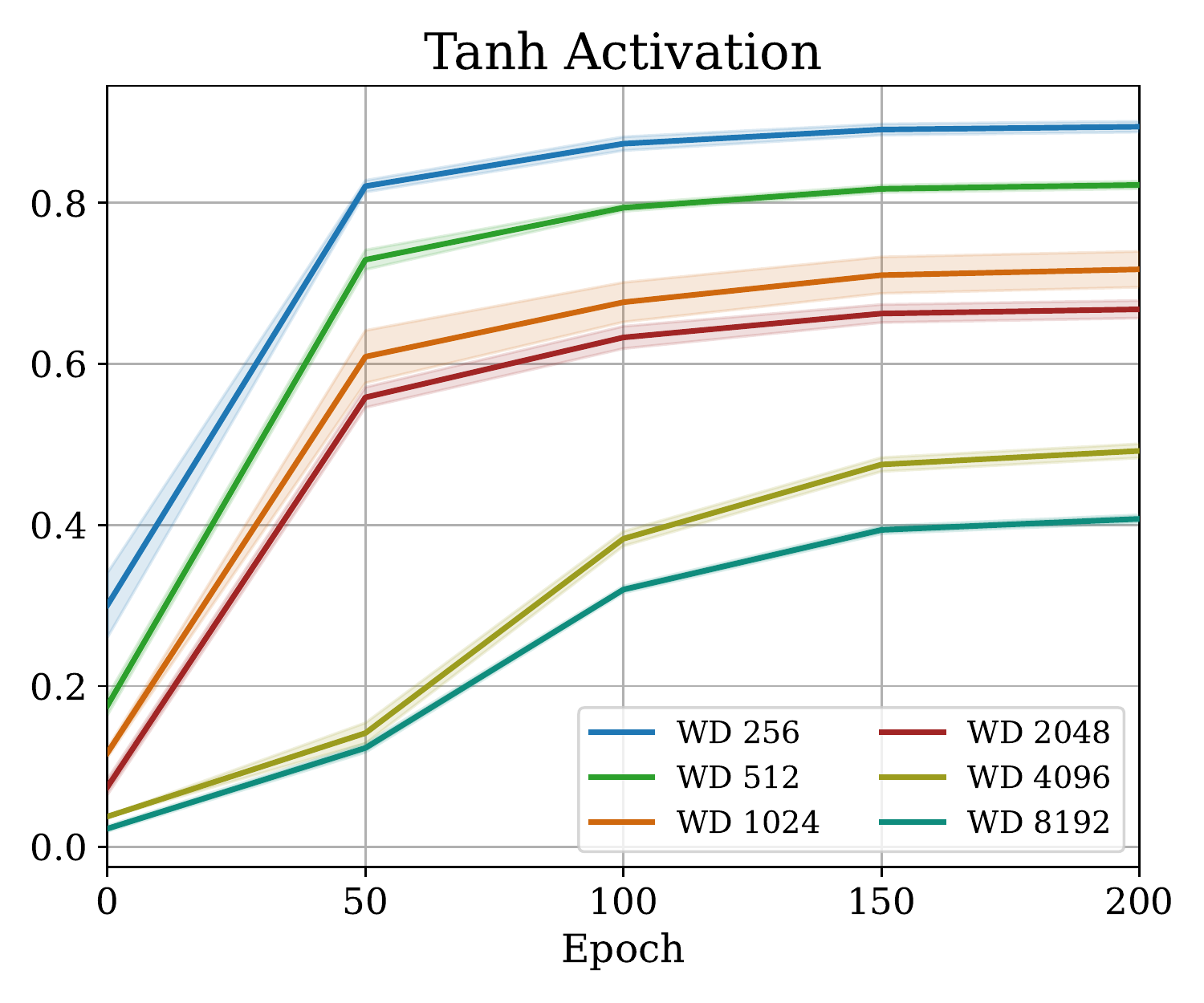}
    \end{subfigure}
    \hfill
    \begin{subfigure}[b]{0.24\textwidth}
        \includegraphics[width=\textwidth]{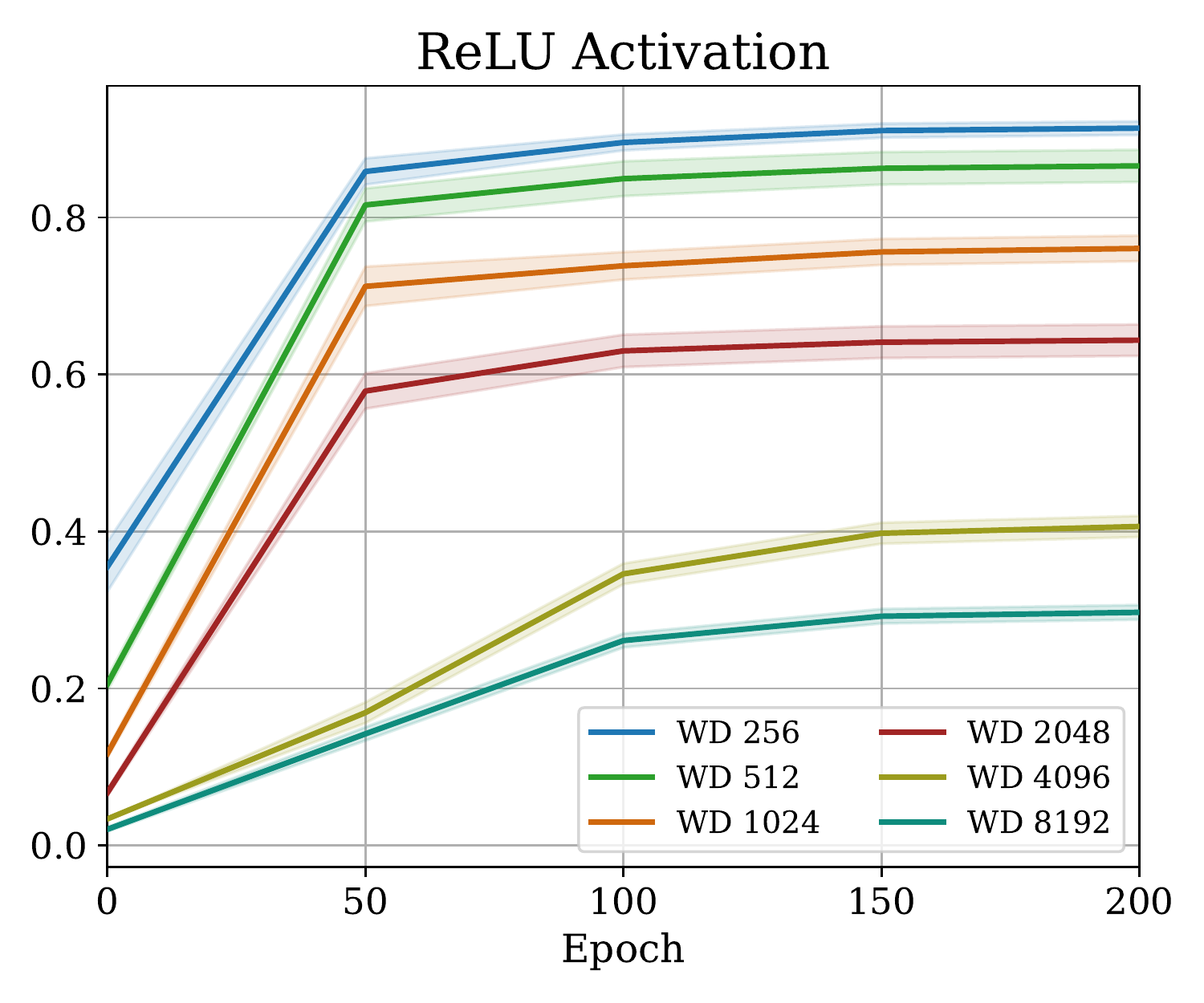}
    \end{subfigure}
    \hfill
    \begin{subfigure}[b]{0.24\textwidth}
        \includegraphics[width=\textwidth]{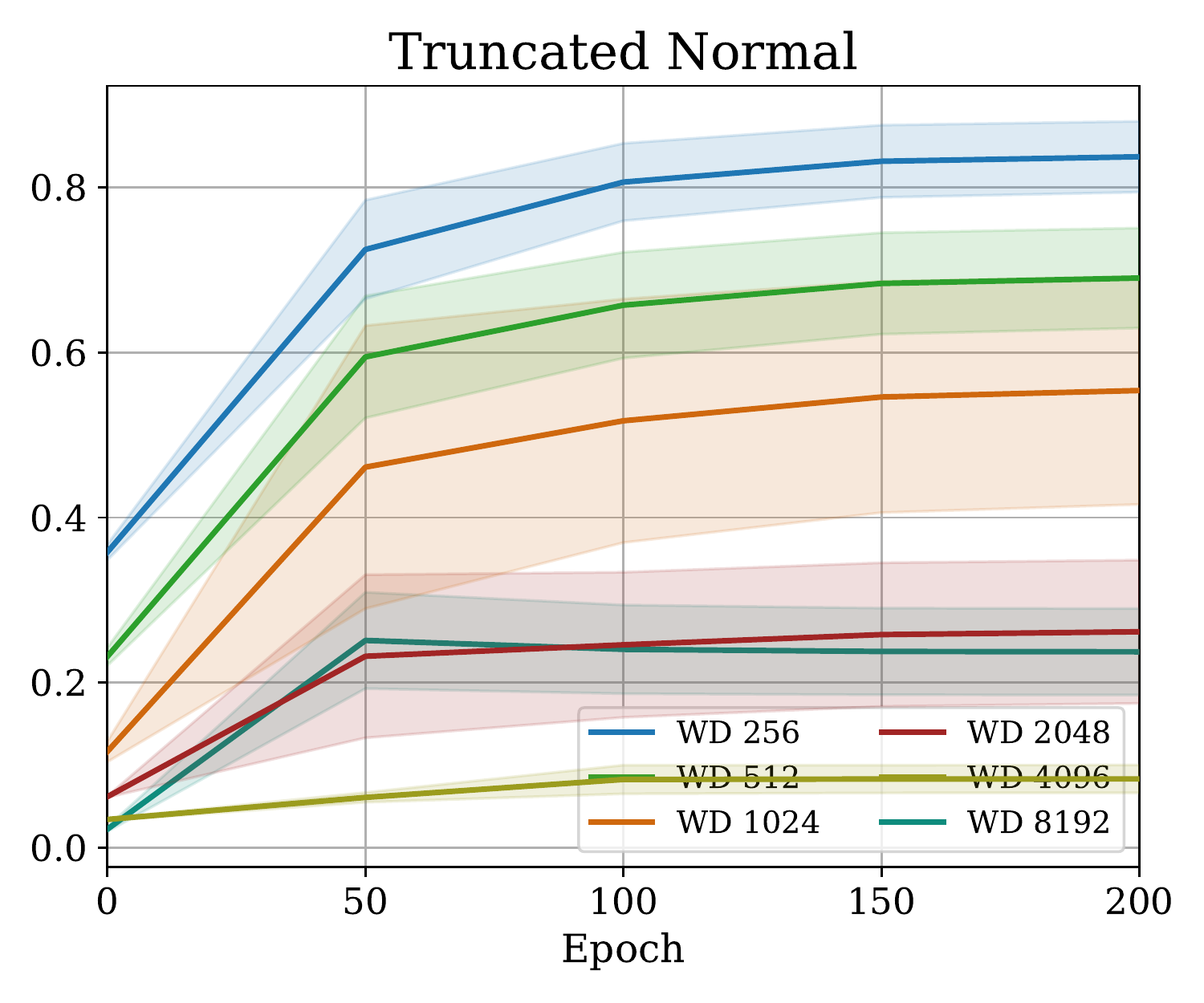}
    \end{subfigure}
    \caption{\textcolor{black}{Evaluating the \textbf{relative difference of Frobenius norm of $\meNTK(\dset, \dset)$ and $\mpNTK(\dset, \dset) \otimes I_O$} at initialization and throughout training, based on $1000$ points from CIFAR-10.}}
    \label{fig:ablation_normdiff}
\end{figure*}

\begin{figure*}[t!]
    \centering
    \begin{subfigure}[b]{0.24\textwidth}
        \includegraphics[width=\textwidth]{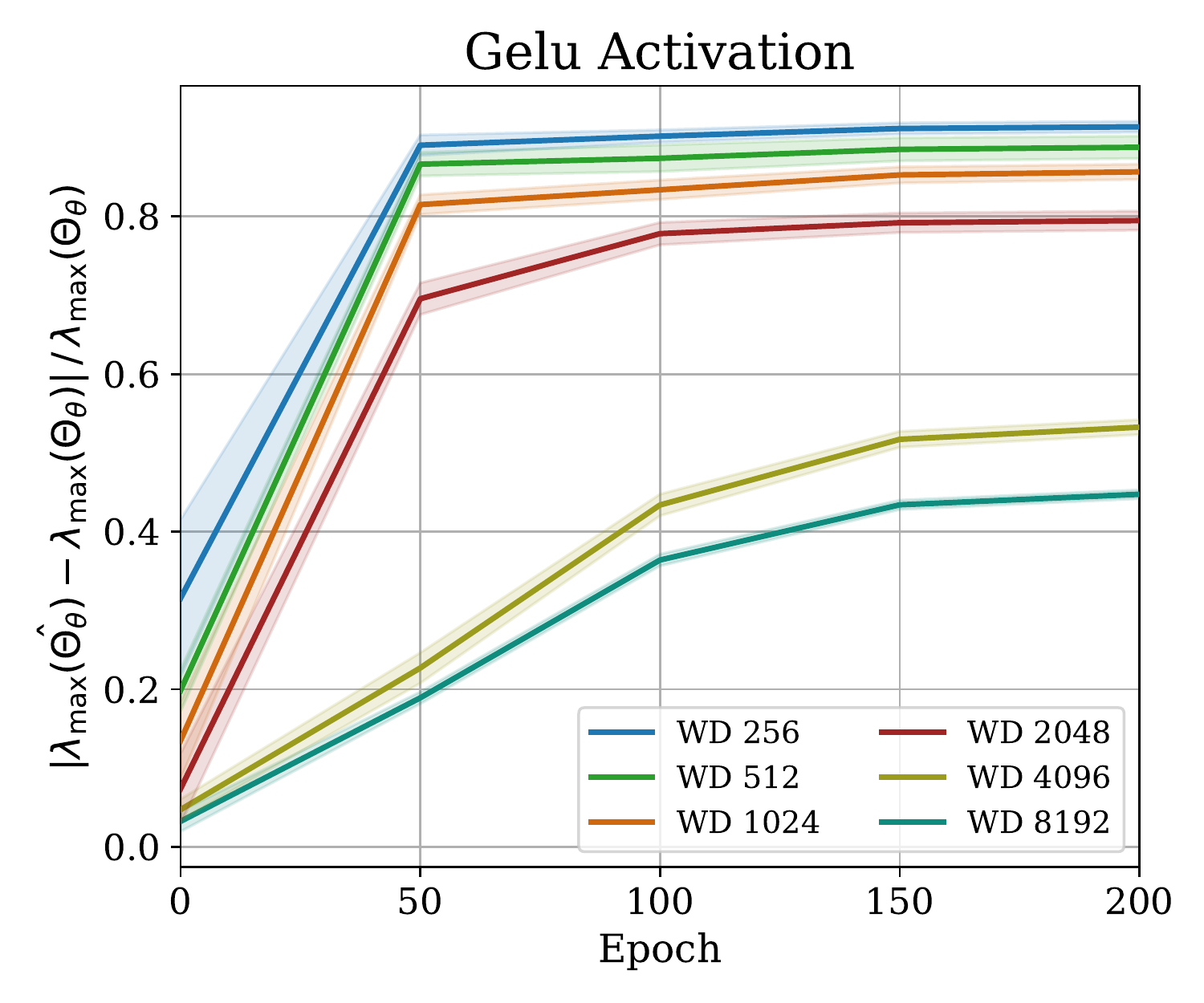}
    \end{subfigure}
    \hfill
    \begin{subfigure}[b]{0.24\textwidth}
        \includegraphics[width=\textwidth]{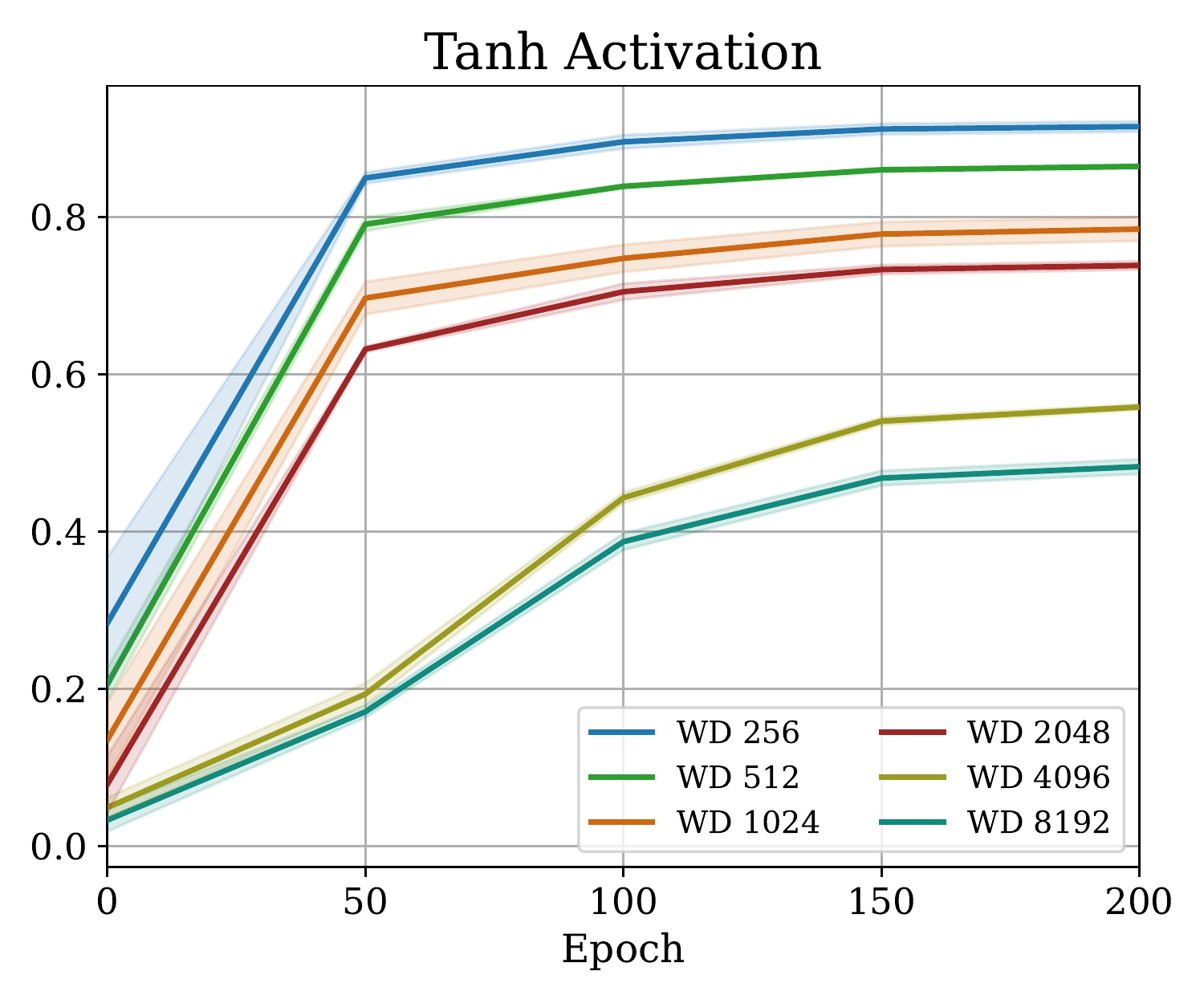}
    \end{subfigure}
    \hfill
    \begin{subfigure}[b]{0.24\textwidth}
        \includegraphics[width=\textwidth]{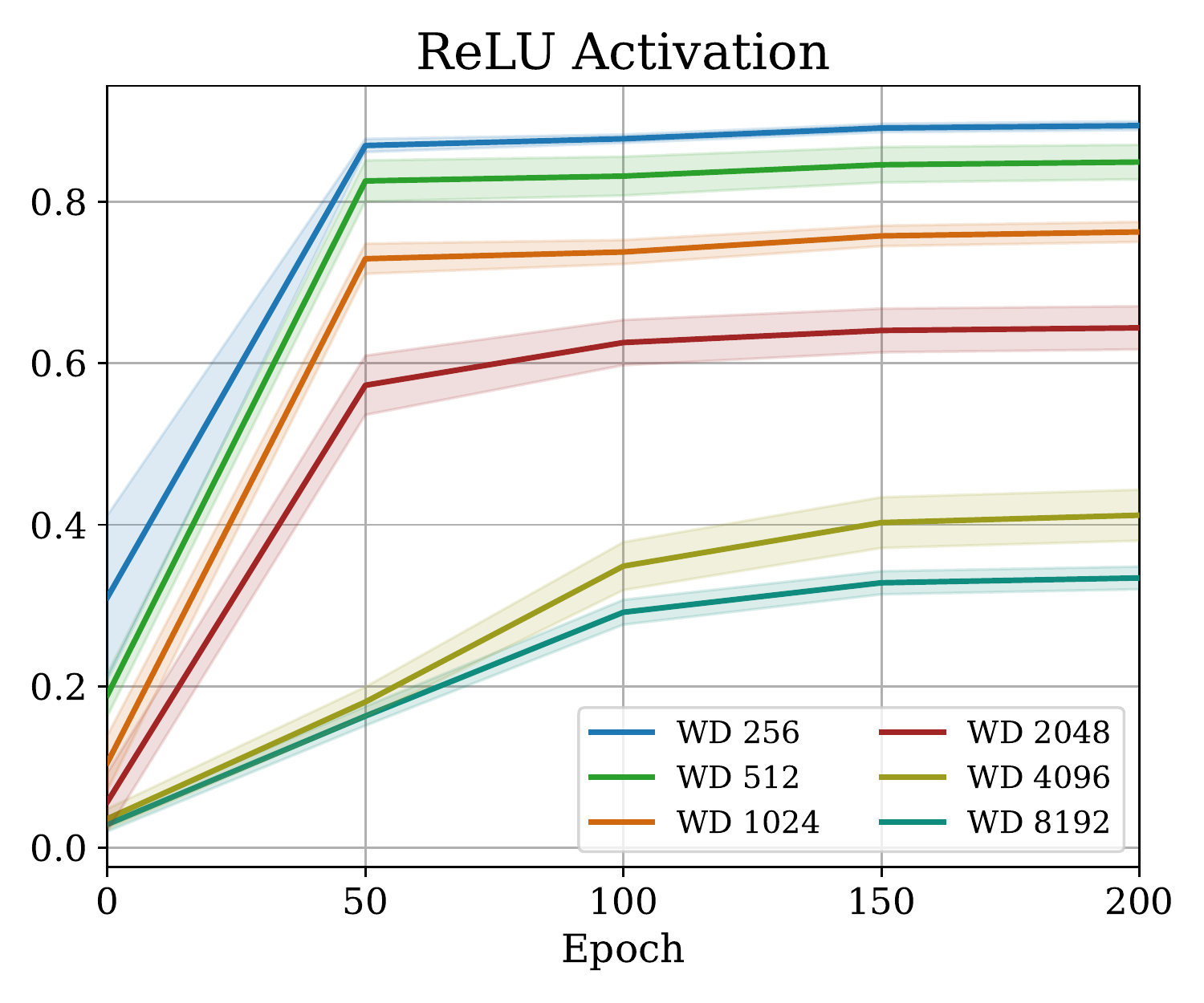}
    \end{subfigure}
    \hfill
    \begin{subfigure}[b]{0.24\textwidth}
        \includegraphics[width=\textwidth]{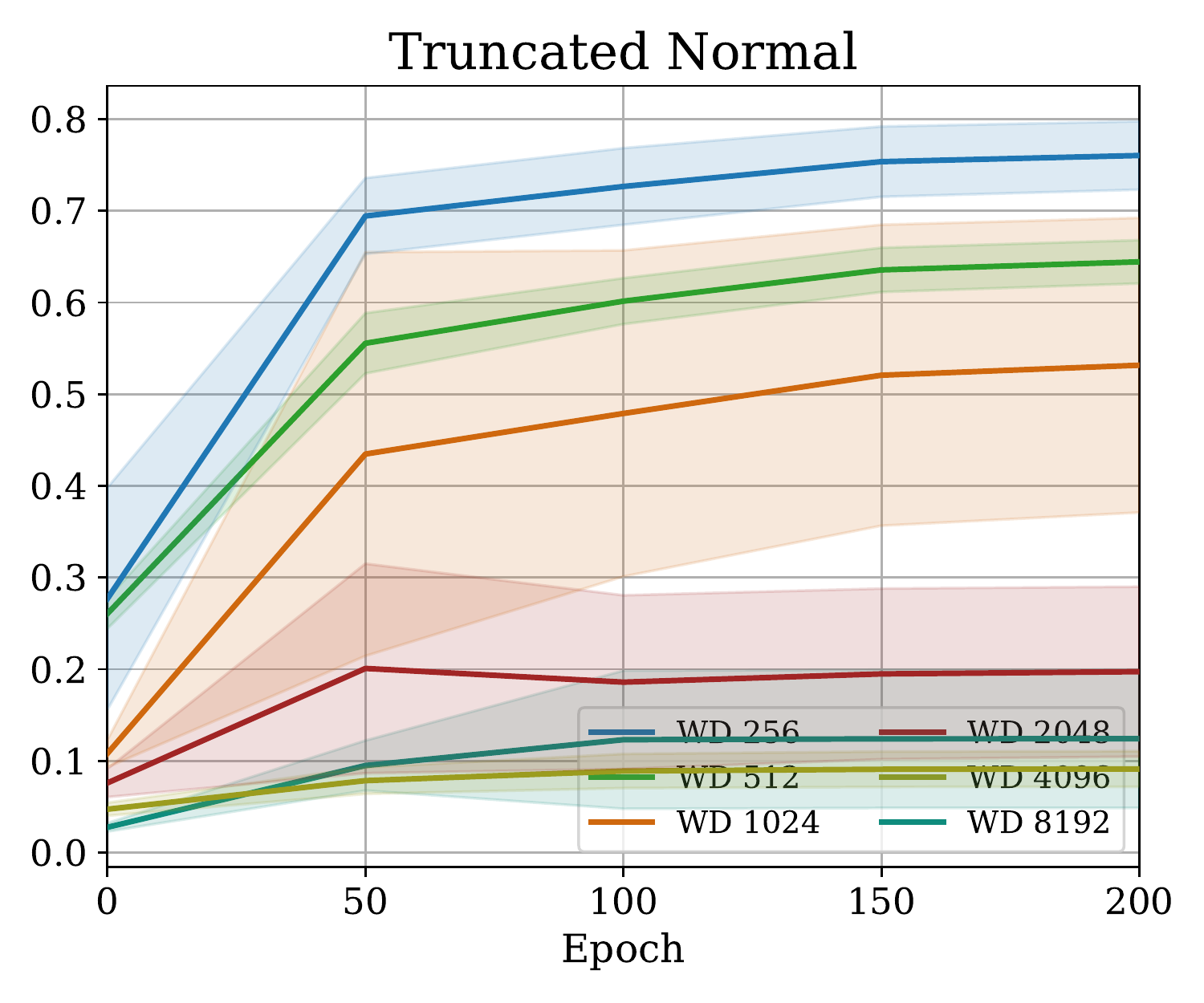}
    \end{subfigure}
    \caption{\textcolor{black}{Evaluating the \textbf{relative difference of $\lambda_{\max}$ of $\meNTK(\dset, \dset)$ and $\mpNTK(\dset, \dset)$} at initialization and throughout training, based on kernels on a subset ($\abs{\dset} = 1000$) of points from CIFAR-10.}}
    \label{fig:ablation_maxeigval_diff}
\end{figure*}

\begin{figure*}[t!]
    \centering
    \begin{subfigure}[b]{0.24\textwidth}
        \includegraphics[width=\textwidth]{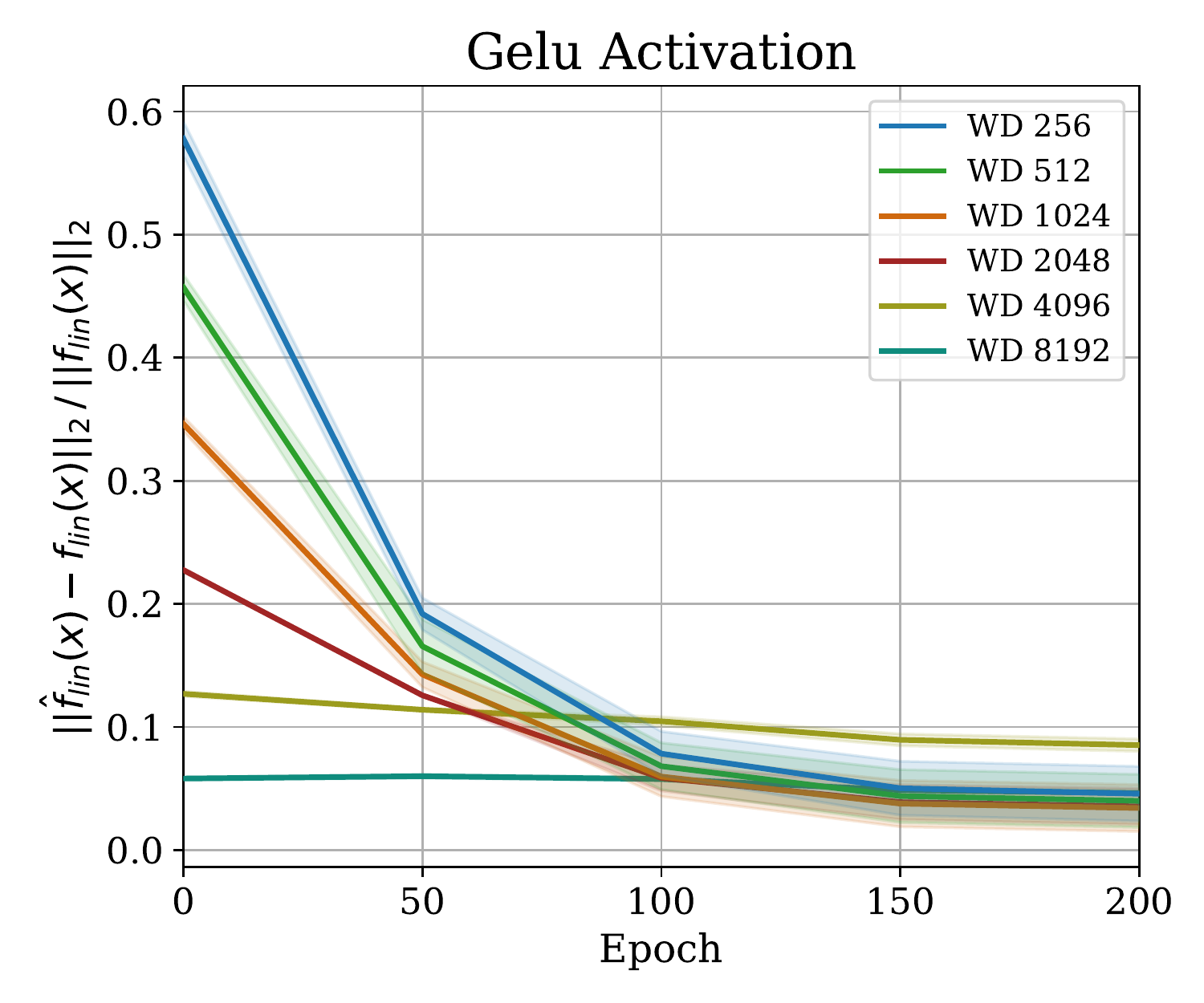}
    \end{subfigure}
    \hfill
    \begin{subfigure}[b]{0.24\textwidth}
        \includegraphics[width=\textwidth]{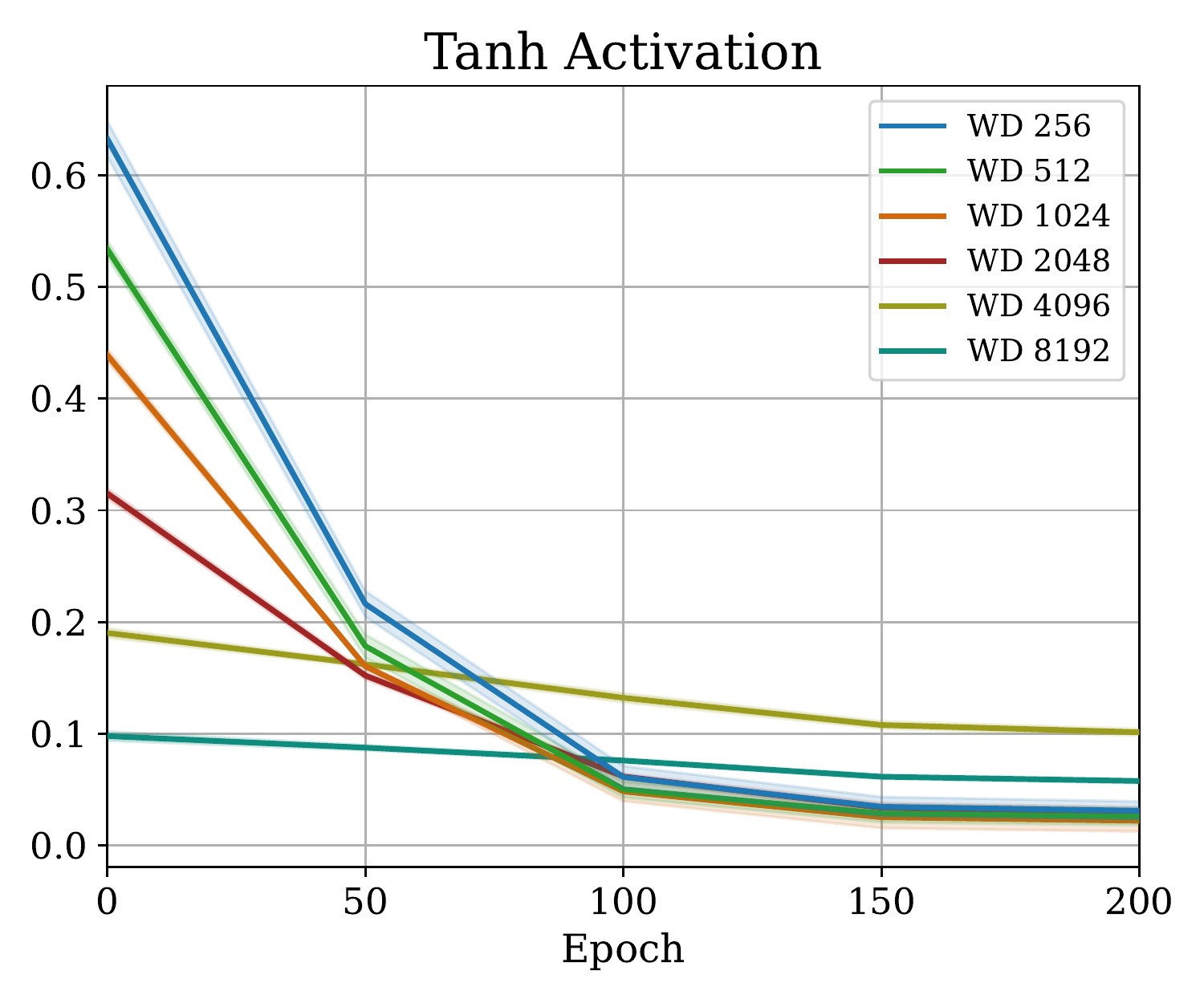}
    \end{subfigure}
    \hfill
    \begin{subfigure}[b]{0.24\textwidth}
        \includegraphics[width=\textwidth]{figs/appendix_figs/nonlin_and_init_experiments/gelu/gelu_preddiff.pdf}
    \end{subfigure}
    \hfill
    \begin{subfigure}[b]{0.24\textwidth}
        \includegraphics[width=\textwidth]{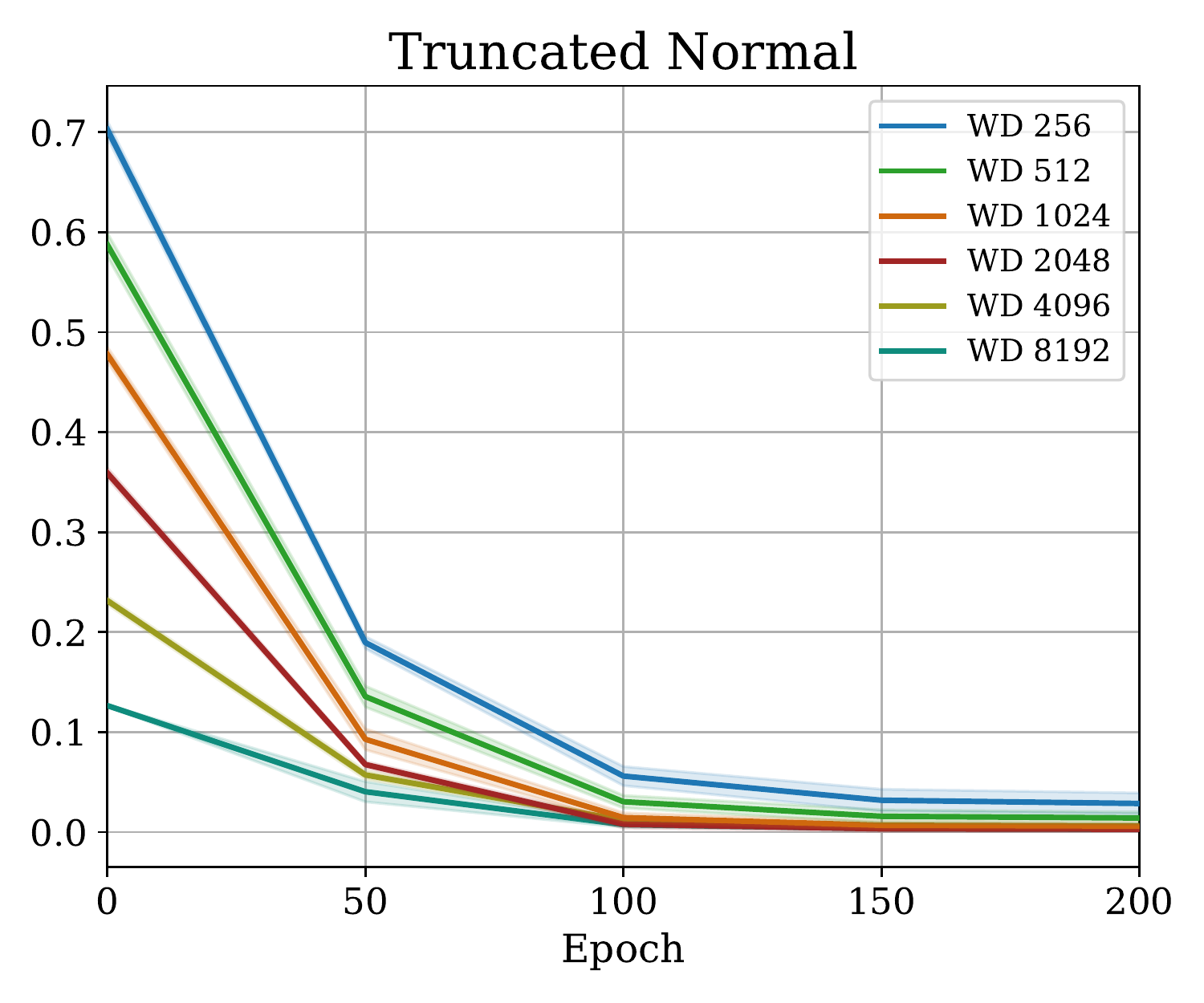}
    \end{subfigure}
    \caption{\textcolor{black}{Evaluating the \textbf{relative norm difference of kernel regression outputs using \eNTK~and \pNTK~as in \cref{eq:f_lin} and \cref{eq:pseudo_lin}} at initialization and throughout training. The kernel regression has been done on $\abs{\dset} = 1000$ training points and $\abs{\cX} = 500$ test points randomly selected from CIFAR-10's train and test sets.}}
    \label{fig:ablation_prediction_diff}
\end{figure*}

\section{Extending the Proofs to Other Architectures} \label{app:other-archs}
We start our description of how to extend the proofs to other architectures by providing a sketch on how the dense weight vectors can be replaced by other layers of choice like convolutions. First, note how the linear weights are used in \cref{eq:supp:ntk_chain_expansion}. As mentioned in Section 6 of \citet{yang2020tensor}, we can accordingly write the same expansion for other forward computational graphs and derive the corresponding canonincal decomposition for them. In subsection 6.2.1, \citet{yang2020tensor} provides a concrete example on how one can derive this expansion for a general RNN-like architecture. As the proofs provided in this section depend on the MLP structure only by means of the canonical decomposition, one can extend them to a general architecture by deriving the corresponding canonical decomposition of that architecture.

\textbf{Non-Gaussian Weights}: According to the strategy used in the proofs, we need the individual weights to be distributed such that the product of two independent scalar weights (as in \cref{eq:supp:ntk_chain_expansion}) remain sub-exponential. Hence, any sub-gaussian initialization method, such as any bounded initialization (\eg\ truncated normal or uniform on an interval) can be used, and the same proof structure would support the same convergence rate, albeit with different constants in convergence (independent of $n$).

\textbf{Non-ReLU activations}: In general, the proofs rely on the ReLU activation through \cref{app:Lemma_relu_post_activation}, which gives a concentration bound on the absolute value of the dot product of post-activations of each layer of the NN. To use other nonlinearities, we would only need an analagous result for that nonlinearity; the other proofs follow without requiring any other significant change.

\textbf{Experimental Evaluation}: To provide further experimental support for this argument, we have conducted an ablation study on the FCN architecture with different nonlinearities and with truncated Gaussian initialization (\cref{fig:ablation_diagonality,,fig:fro_norm_diff,,fig:ablation_maxeigval_diff,,fig:ablation_prediction_diff}). As seen in the provided figures, the impact of nonlinearity and initialization method as long as they follow the provided setting in \cref{sett:supp:relu_mlp_setting}, is marginal.

\section{More Details on Kernel Regression Using \pNTK~on Full CIFAR-10 Dataset}
\label{app:full-cifar}

\begin{figure*}[t]
    \centering
    \begin{subfigure}[b]{0.24\textwidth}
        \includegraphics[width=\textwidth]{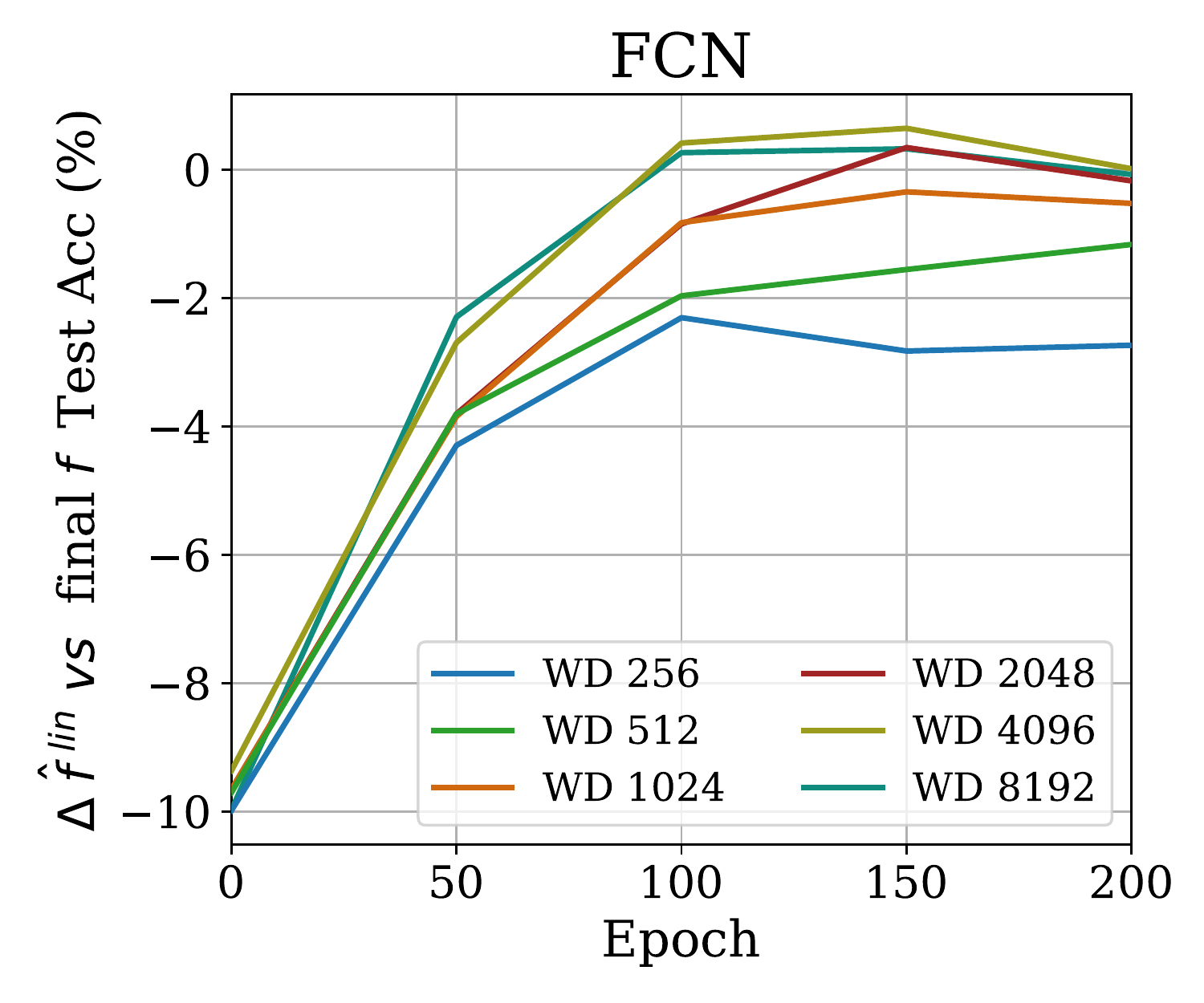}
    \end{subfigure}
    \hfill
    \begin{subfigure}[b]{0.24\textwidth}
        \includegraphics[width=\textwidth]{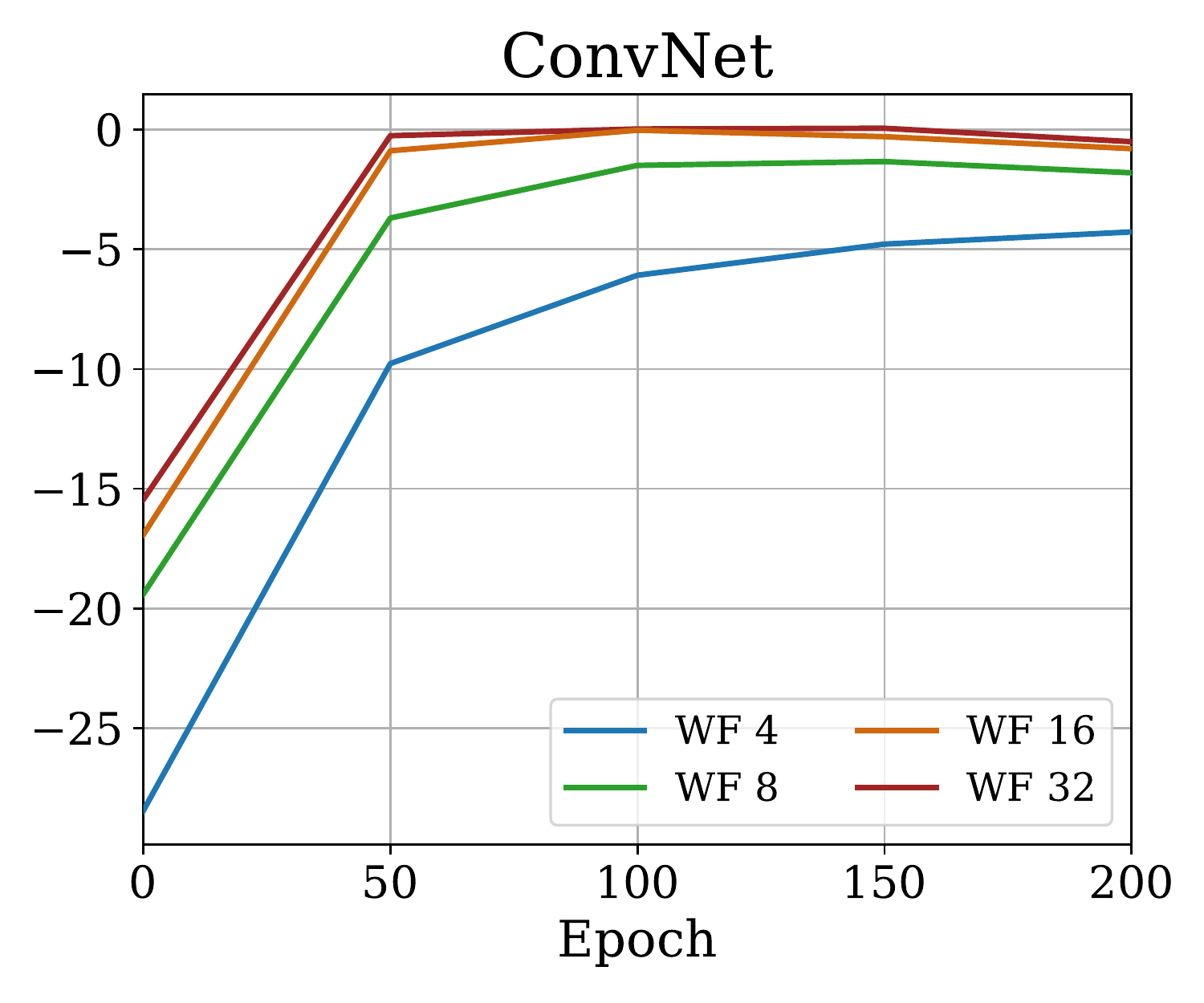}
    \end{subfigure}
    \hfill
    \begin{subfigure}[b]{0.24\textwidth}
        \includegraphics[width=\textwidth]{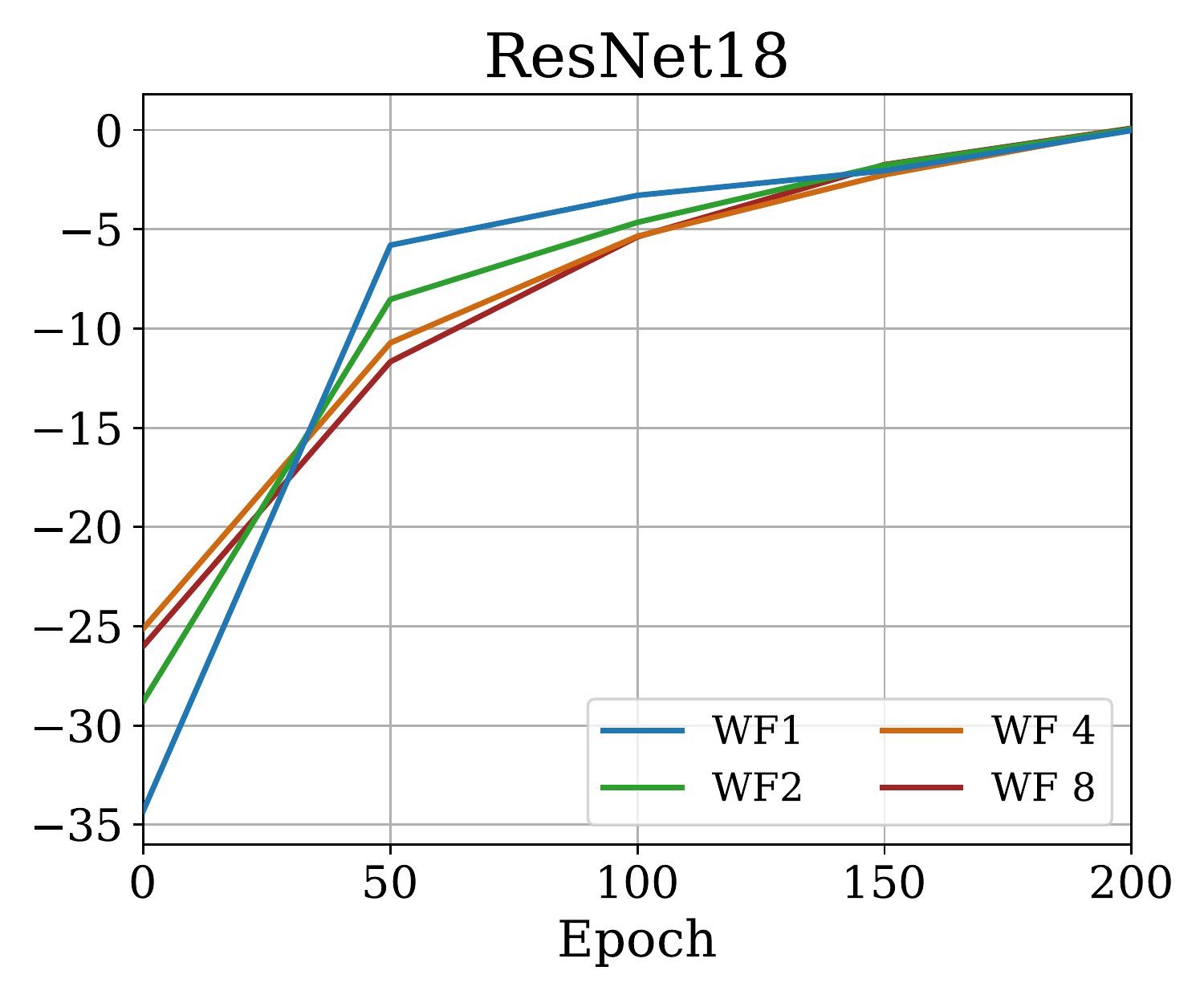}
    \end{subfigure}
    \hfill
    \begin{subfigure}[b]{0.24\textwidth}
        \includegraphics[width=\textwidth]{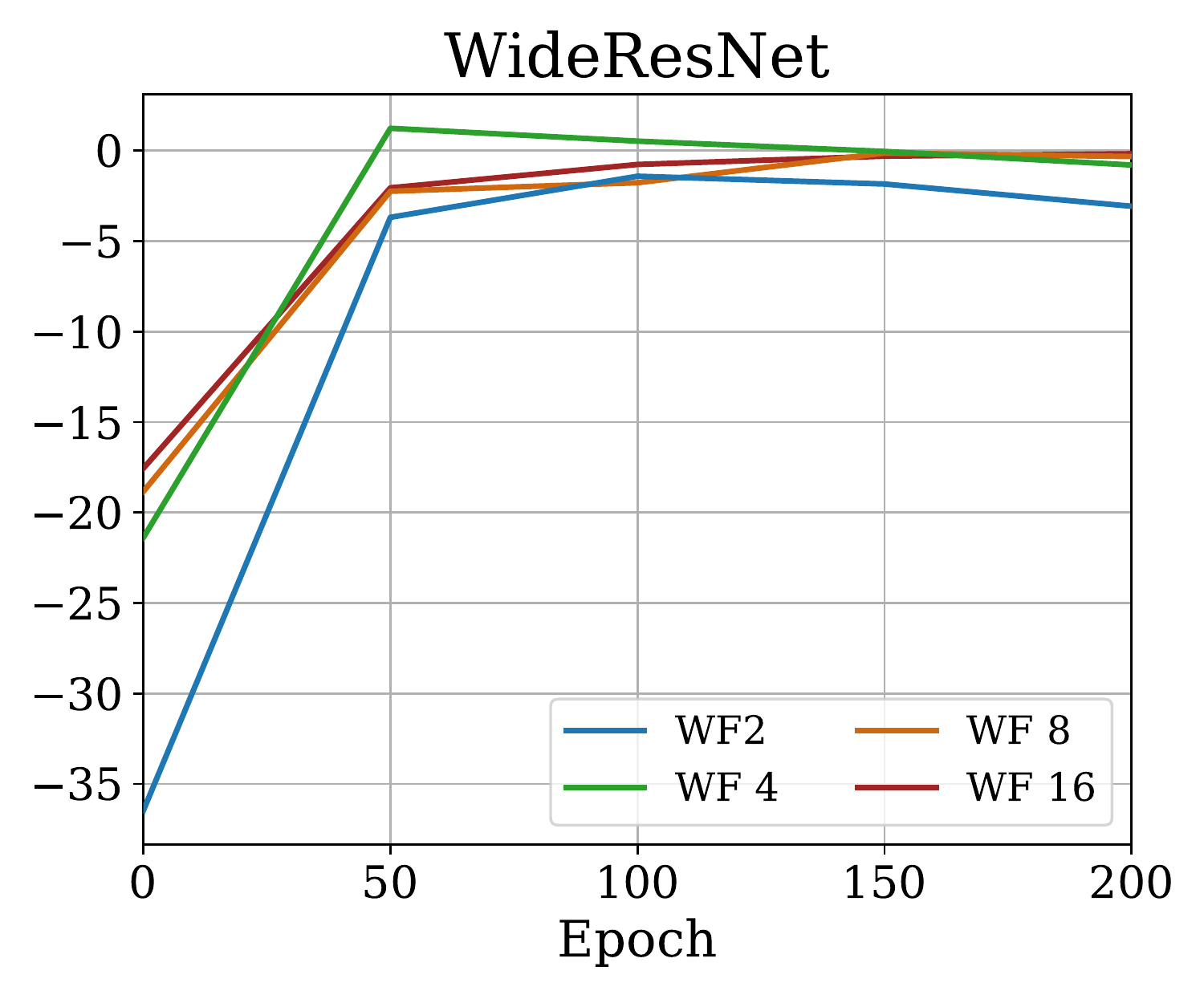}
    \end{subfigure}
    \caption{Evaluating the \textbf{difference in test accuracy of kernel regression using \pNTK~as in \eqref{eq:pseudo_lin} vs the final model $f$} throughout SGD training on the full CIFAR-10 dataset. How much worse would it be to ``give up'' on SGD at this point and train $\hat f^\mathit{lin}$ with the current representation?}
    \label{fig:final_f_vs_fhat_prediction_full_cifar10_acc}
\end{figure*}

\Cref{fig:final_f_vs_fhat_prediction_full_cifar10_acc} compares the accuracy of $\hat f^\mathit{lin}(x)$ with parameters derived at epoch $E \in \{0, 50, 100, 150, 200\}$ of training the NN with SGD. On the y-axis, the reported number is $f^\mathit{lin}(x) - f^*(x)$ where $f^*$ denotes the final model obtained after training $f$ for 200 epochs. As seen in \cref{fig:final_f_vs_fhat_prediction_full_cifar10_acc} the architecture of the model has a significant impact on how good the linearization predicts the final accuracy of the fully-trained model. However, as proven in \cref{theorem:pntk_fro_norm} in conjunction with the linearization approximations provided in \citet{linntk2019lee}, as width grows, this approximation becomes more accurate. One unexplored fact regarding this experiment is that fact that lineraization with trained parameters significantly outperforms linearization at initialization, which is intuitive but not rigorously investigated yet.

\section{Experimental Evaluation: Tightness of bounds}
\Cref{fig:entk_vs_pntk_tight} presents experimental evaluations that analyze the tightness of the approximation bound. The results are presented for the fully connected network used in the experiment.
\begin{figure*}[t]
    \centering
    \begin{subfigure}[b]{0.24\textwidth}
        \includegraphics[width=\textwidth]{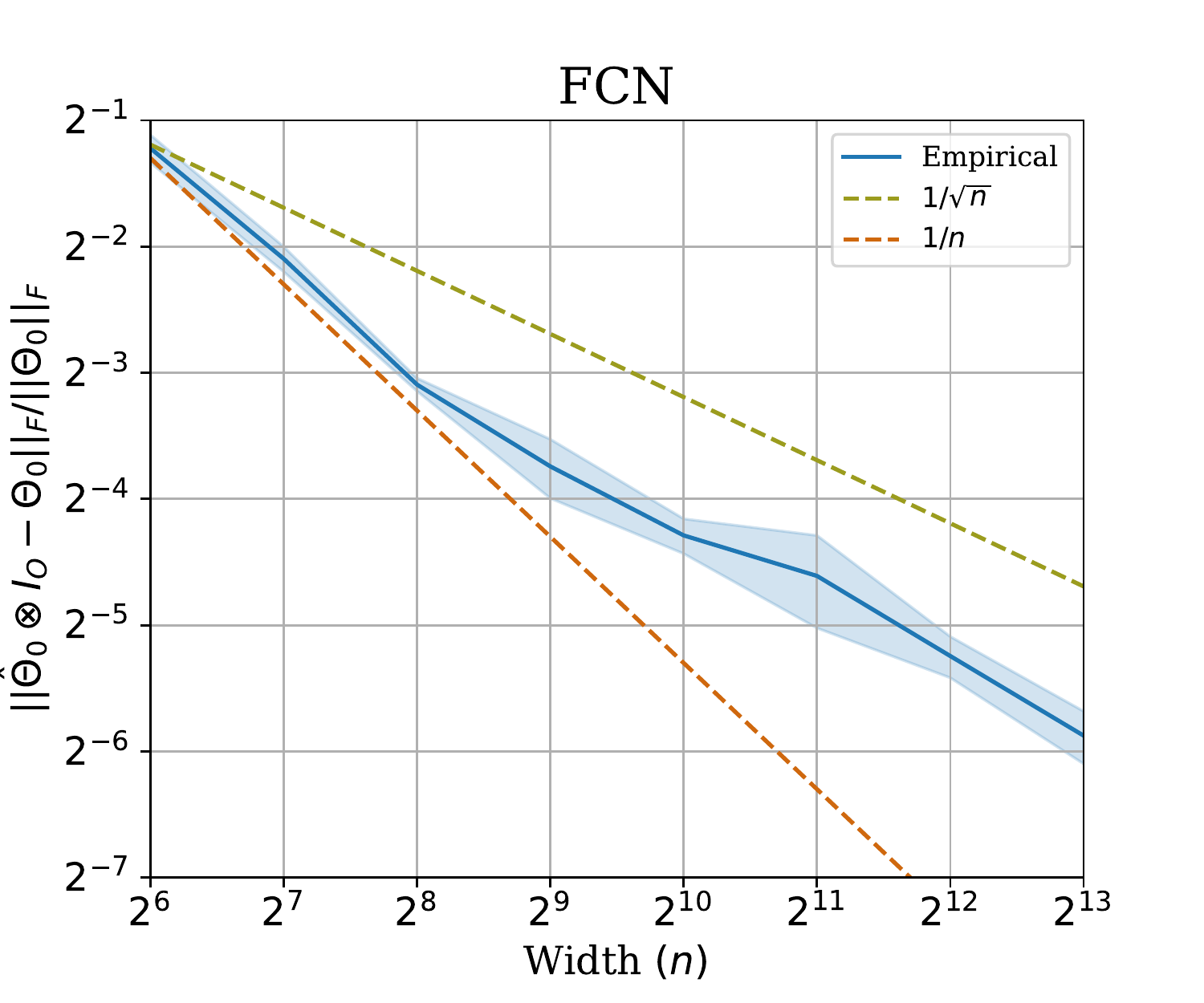}
    \end{subfigure}
    \hfill
    \begin{subfigure}[b]{0.24\textwidth}
        \includegraphics[width=\textwidth]{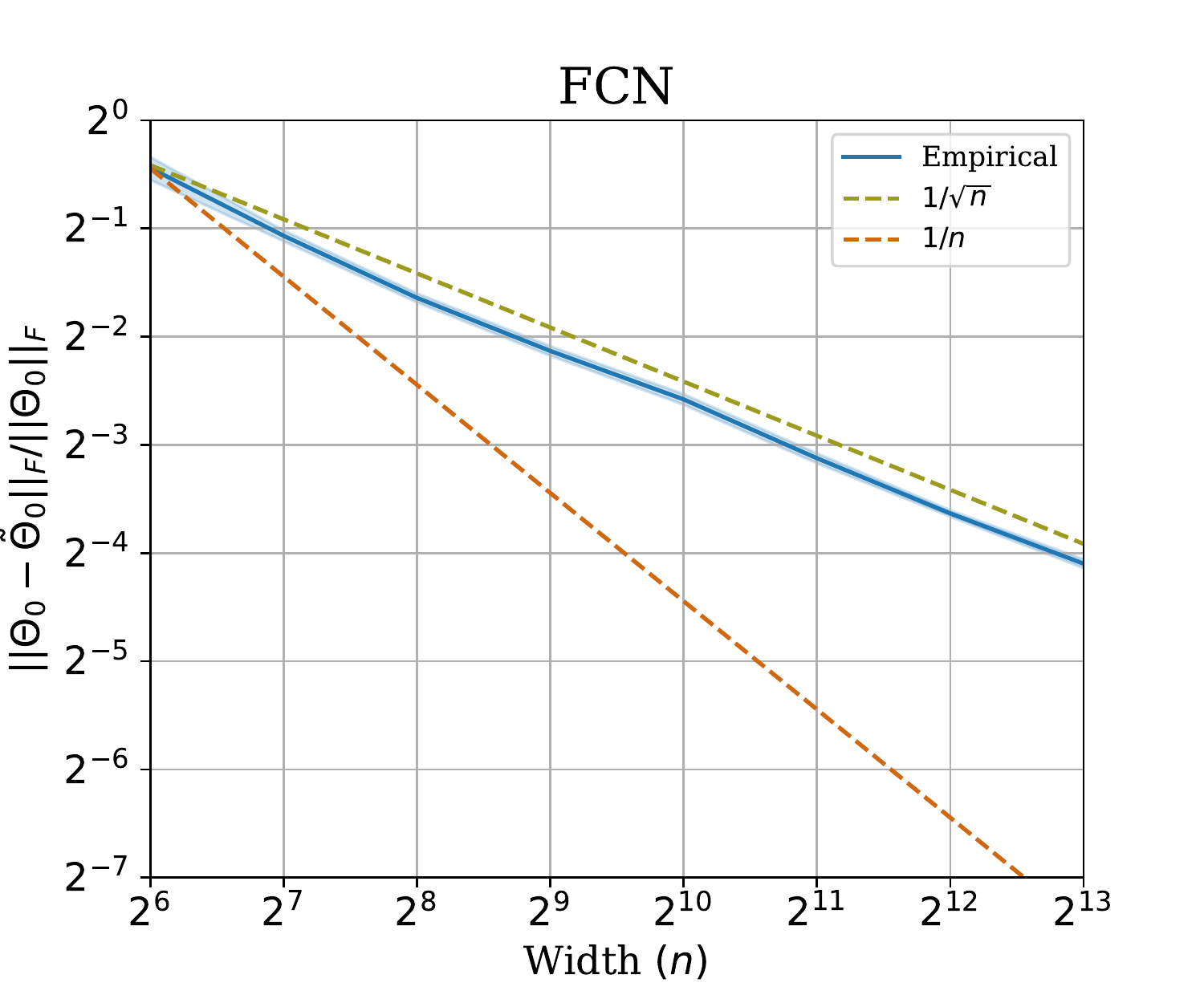}
    \end{subfigure}
    \hfill
    \begin{subfigure}[b]{0.24\textwidth}
        \includegraphics[width=\textwidth]{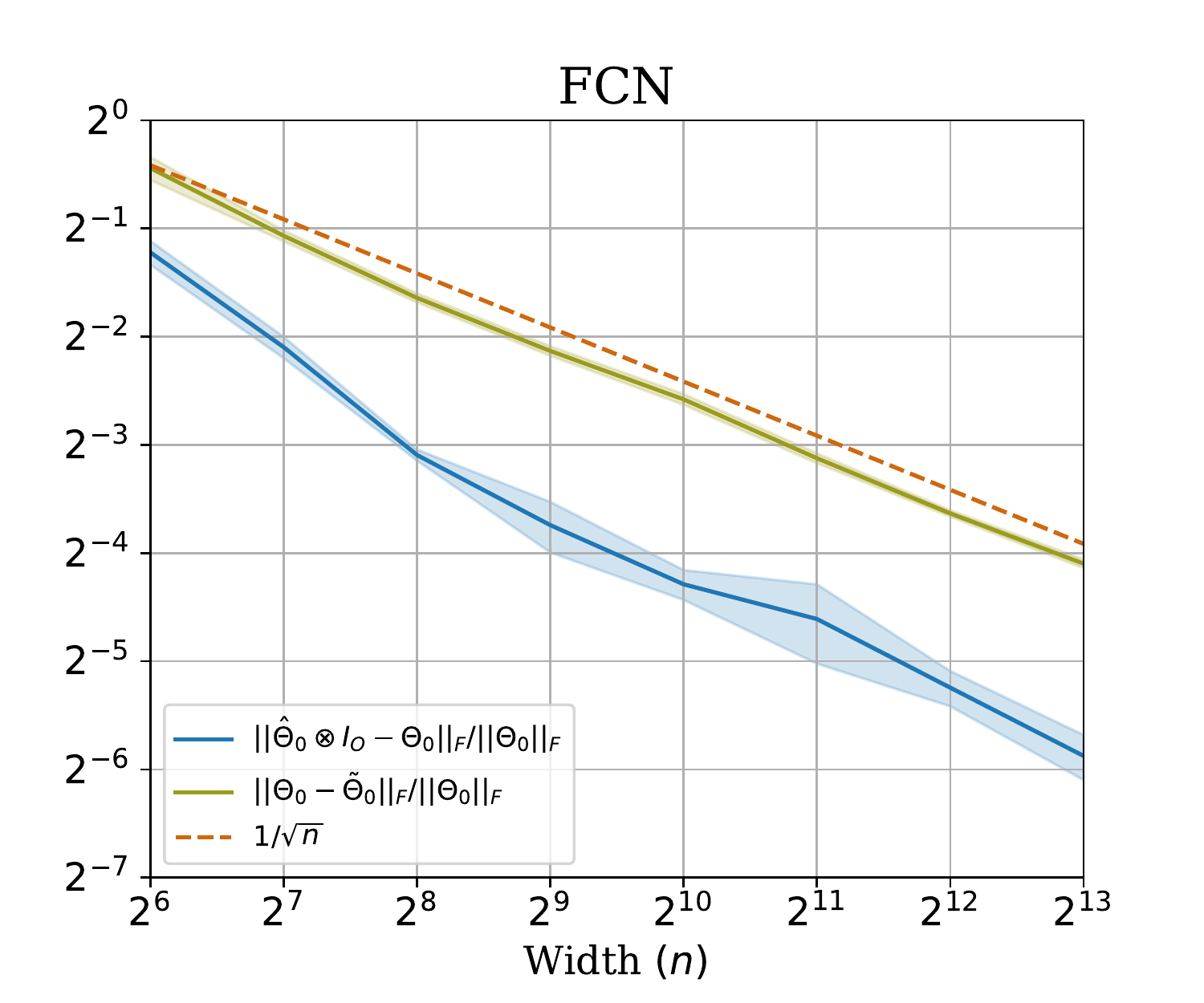}
    \end{subfigure}
    \hfill
    \begin{subfigure}[b]{0.24\textwidth}
        \includegraphics[width=\textwidth]{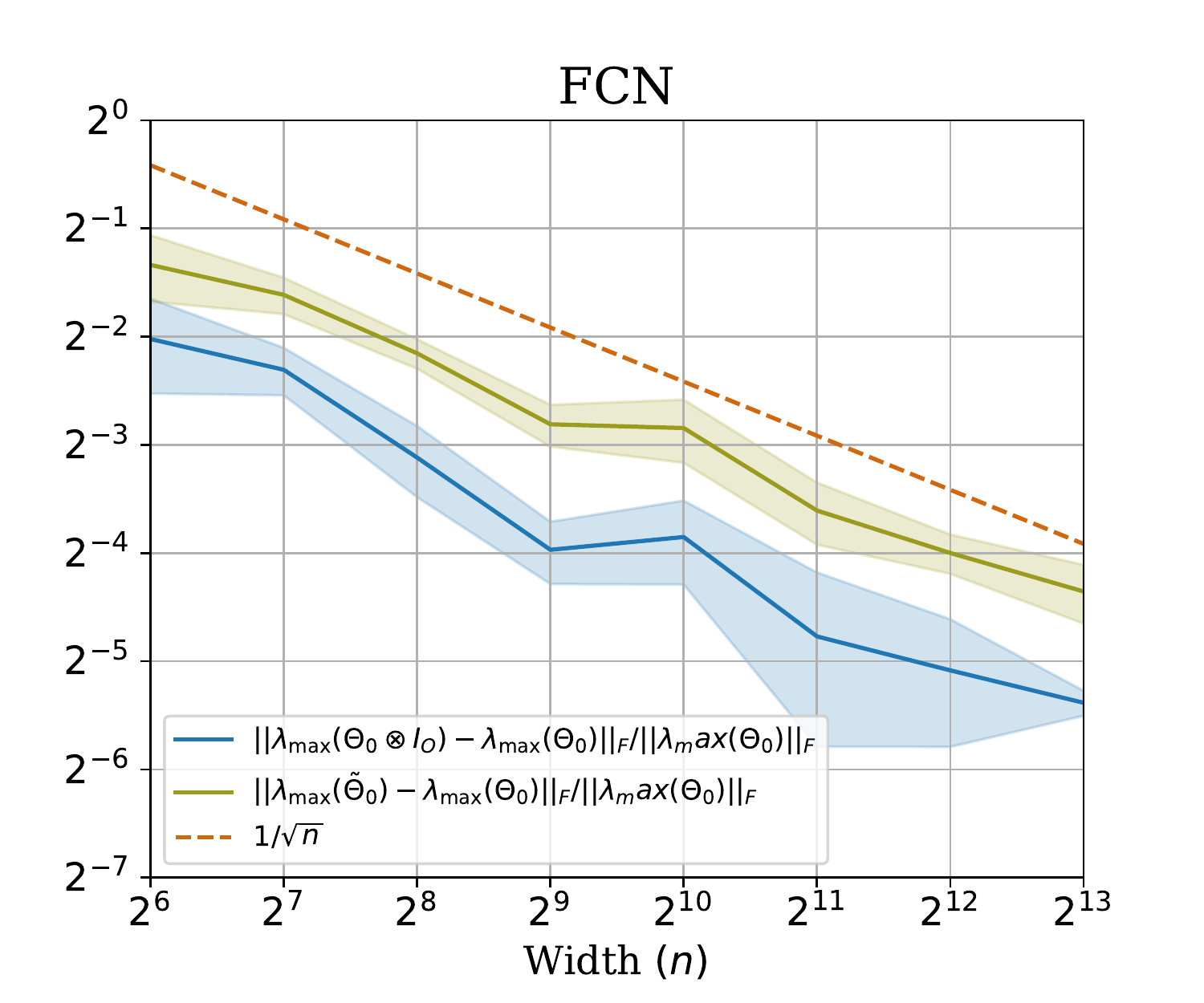}
    \end{subfigure}
    \caption{Experimental evaluation of tightness of approximation bounds}
    \label{fig:entk_vs_pntk_tight}
\end{figure*}

\end{document}